\documentclass[11pt]{article}

\usepackage[preprint]{acl}
\usepackage{times}
\usepackage{latexsym}

\usepackage[T1]{fontenc}

\usepackage{amsmath,amsfonts,bm}

\def\eqref#1{equation~\ref{#1}}

\def\1{\bm{1}}

\DeclareMathAlphabet{\mathsfit}{\encodingdefault}{\sfdefault}{m}{sl}
\SetMathAlphabet{\mathsfit}{bold}{\encodingdefault}{\sfdefault}{bx}{n}

\usepackage[utf8]{inputenc}
\usepackage{microtype}
\usepackage[T1]{fontenc}    
\usepackage{xcolor}
\usepackage{graphicx} 
\usepackage{amsmath}
\usepackage{amsthm}
\usepackage{amssymb}
\usepackage{amsfonts}
\usepackage{dsfont}
\usepackage{booktabs}
\usepackage{multirow}
\usepackage{adjustbox}
\usepackage{etoolbox}
\usepackage{enumitem}
\usepackage{algorithm}
\usepackage{algpseudocode}
\usepackage{pgfplots}
\usepackage{pgfplotstable}
\usepgfplotslibrary{groupplots} 
\pgfplotsset{compat=1.18}
\usepackage{caption}
\usepackage{subcaption}
\usepackage{tipa}
\usepackage{tabularray}
\usepackage{tcolorbox}
\usepackage{hyperref}
\usepackage{url}
\usepackage{nicefrac}       
\usepackage{microtype}      
\usepackage{listings}
\usepackage{pifont}
\usepackage{inconsolata}
\usepackage{fontawesome5}

\newtheoremstyle{tight}{0pt}{0pt}{}{}{\bfseries}{:}{0.5em}{}
\theoremstyle{tight}
\newtheorem{theorem}{Theorem}
\newtheorem{lemma}[theorem]{Lemma}

\definecolor{claude}{RGB}{31, 119, 180}
\definecolor{claudeThink}{RGB}{255, 127, 14}
\definecolor{gemini}{RGB}{44, 160, 44}
\definecolor{o4mini}{RGB}{214, 39, 40}
\definecolor{deepseek}{RGB}{148, 103, 189}
\UseTblrLibrary{booktabs}

\theoremstyle{definition}
\newtheorem{definition}{Definition}[section]

\definecolor{reslinkcolor}{RGB}{0,124,138}   
\newcommand{\emojiicon}[1]{%
  \raisebox{-0.22em}{\includegraphics[height=1.2em]{figures/#1}}}
\newcommand{\hficon}{\emojiicon{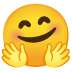}}
\newcommand{\globeicon}{\emojiicon{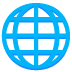}}
\newcommand{\trophyicon}{\emojiicon{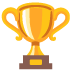}}
\newcommand{\reslink}[2]{\href{#1}{\textcolor{reslinkcolor}{\texttt{\small #2}}}}
\newcommand{\ressep}{\hspace{0.9em}}
\newcommand{\resrow}[3]{#1 & \textbf{#2:} & #3 \\}
\newenvironment{resources}%
  {\setlength{\tabcolsep}{0pt}%
   \begin{tabular}{@{}r@{\hspace{0.45em}}l@{\hspace{0.7em}}l@{}}}%
  {\end{tabular}}

\newcommand{\reasoningemoji}{\ding{72}}  
\newcommand{\unkemoji}{\ding{88}} 
\newcommand{\moeemoji}{\ding{110}}      

\lstdefinestyle{jsonstyle}{
    backgroundcolor=\color{white},
    basicstyle=\ttfamily\small,
    breaklines=true,
    frame=single,
    rulecolor=\color{black},
    showstringspaces=false,
    stringstyle=\color{orange},
    keywordstyle=\color{blue},
    commentstyle=\color{gray},
    morekeywords={true,false,null}
}

\title{PBEBench: A Multi-Step Programming by Examples Reasoning Benchmark inspired by Historical Linguistics}

\author{
 \textbf{Atharva Naik\textsuperscript{1}},
\textbf{Prakam\textsuperscript{3}},
 \textbf{Yash Mathur\textsuperscript{3}},
 \textbf{Darsh Agrawal\textsuperscript{1}},
\\
 \textbf{Manav Kapadnis\textsuperscript{1}},
 \textbf{Yuwei An\textsuperscript{1}},
 \textbf{Clayton Marr\textsuperscript{2}},
 \textbf{Carolyn Rose\textsuperscript{1}},
\\
 \textbf{David Mortensen\textsuperscript{1}}
\\
\\
{\normalfont\normalsize\begin{resources}
\resrow{\globeicon}{Website}{
  \reslink{https://cmu-llab.github.io/PBEBench/}{cmu-llab.github.io/PBEBench}}
\resrow{\trophyicon}{Leaderboard}{
  \reslink{https://cmu-llab.github.io/PBEBench/\#leaderboard}{live results}}
\resrow{\hficon}{Data}{
  \reslink{https://huggingface.co/datasets/changelinglab/PBEBench}{PBEBench}
  \ressep
  \reslink{https://huggingface.co/datasets/changelinglab/PBEBench-Lite}{PBEBench-Lite}}
\resrow{\faGithub}{Code \& logs}{
  \reslink{https://github.com/cmu-llab/pbebench}{cmu-llab/pbebench}}
\end{resources}}
\\
\\
 \textsuperscript{1}Carnegie Mellon University,
 \textsuperscript{2}Ohio State University,
 \textsuperscript{3}Independent Researcher
\\
 \texttt{\{arnaik,darsha,mkapadni,yuweia,cprose,dmortens\}@cs.cmu.edu}
}

\begin{document}
\maketitle

\begin{abstract}
While many benchmarks evaluate the reasoning abilities of Large Language Models (LLMs), few isolate reasoning as a capability independent of domain knowledge. We introduce a new benchmark for inductive reasoning inspired by Sound Law Induction (SLI) in historical linguistics and formulated in a simple multi-step Programming by Example (PBE) framework. 
The task requires inducing a cascade of string rewrite programs that transform inputs into target outputs. 
We present PBEBench, a fully automated evaluation approach that generates such problems with controllable difficulty and ordering constraints, enabling scalable and contamination-resistant evaluation of sequential inductive reasoning. 
Using this approach, we construct three datasets that show a large gap between models that leverage test-time compute or long chain-of-thought reasoning and those that do not. 
Although recent models such as GPT-5 and gpt-oss-120b show promise, solve rates remain below 5\% on hard \textit{PBEBench} instances with long program cascades, even under computationally expensive scaling strategies. Finally, we show that \textit{PBEBench} scores are more predictive of performance on real SLI than are other inductive reasoning benchmarks. 
We release code and data to support further research.\footnote{\url{https://github.com/cmu-llab/PBEBench}}
\end{abstract}

\section{Introduction}
In historical linguistics, explaining how ancestral word forms evolve into their modern reflexes requires positing an ordered sequence of sound changes. 
The order is critical: applying a change too early or too late can eliminate the conditions that enable other changes, producing incorrect outcomes. 
For example, if the change mapping Proto-Tangkhulic /\textipa{i}/ to /\textipa{W}/ word-finally in Tusom had preceded the deletion of word-final /\textipa{N}/, the Tusom word for `snail’ (from *\textipa{liN}) would be /\textipa{li}/ rather than the attested /\textipa{lW}/
Linguists say that the second rule \textsc{feeds} the first because it creates contexts where the first can apply.
There are three other ordering-sensitive relations between sound changes: \textsc{bleeding} (A removes contexts for B), \textsc{counter-feeding} (B would feed A if ordered first), and \textsc{counter-bleeding} (B would bleed A if ordered first).
Inferring such interactions has long been recognized as a core challenge in phonological theory \citep{Kiparsky1968-KIPLUA, kiparsky1971historical}.
At an abstract level, Sound Law Induction (SLI) can be viewed as a programming-by-example task: given input–output string pairs, the goal is to infer a sequence of local rewrite rules whose composition maps inputs to outputs \cite{naik2024can, naik2025programming}.
Similar ordering constraints arise in other domains, such as multi-file code refactoring \cite{gautamrefactorbench}, where transformations must be applied in a consistent order while reasoning over evolving intermediate states under partial observability.
Formally, given inputs $\vec{\imath}$ (ancestral words or input files) and corresponding outputs $\vec{o}$, the task is to infer any valid program \emph{cascade} $\vec{p}$, of incremental transformations that maps $\vec{\imath}$ to $\vec{o}$.

\cite{naik2025programming, naik2024can} show that LLMs struggle with SLI, but benefit more from training on logically consistent synthetic sound laws than on real sound laws from \emph{Index Diachronica}. 
However, these works do not explain why this is the case. 
We hypothesize that the difficulty arises not from individual rewrite operations, but from inducing sequences of interacting transformations whose correctness depends on precise ordering and unobserved intermediate states, explaining the advantage of structurally consistent synthetic data. 
Despite minimal formal machinery, the problem exhibits rich interaction structure, making it a natural testbed for multi-step inductive reasoning and sequential planning in a domain-light setting \cite{ma2024kor}.
Moreover, while LLMs may have encountered string rewriting tasks during pre-training, PBEBench instances are randomly generated with controlled structure, making it unlikely that pre-training memorization or distributional biases significantly skew performance.
We propose PBEBench, an automated approach towards generating such data automatically and scalably. 
PBEBench exhibits several desirable properties missing in existing reasoning benchmarks: (1) it does not rely on specialized background knowledge (see Section~\ref{sec:related-work}); (2) its difficulty emerges from a simple, formally specified, finite, problem space of sequential string rewrites; (3) it supports explicit control of fundamental interaction patterns---feeding, bleeding, counter-feeding, and counter-bleeding---studied in historical phonology \cite{Kiparsky1968-KIPLUA,kiparsky1971historical} (termed as BFCC relations here); (4) it is inherently resistant to data contamination and saturation, as new instances with controlled difficulty can be generated scalably; and (5) it evaluates models on practically and scientifically meaningful induction problems rather than human–model comparisons.

\phantomsection
\label{sec:introduction:research_questions}
Our study investigates the following hypotheses: \\
\textbf{H1: LLMs struggle with multi-step inductive reasoning.} We benchmark 13 reasoning and 8 non-reasoning models on three PBEBench datasets to study reasoning bottlenecks in LLMs. The domain-light setting suggests that reasoning models should outperform non-reasoning ones. \\
\textbf{H2: Difficulty in our dataset is driven by interaction structure.} We hypothesize that while inducing a single string rewrite is straightforward, LLMs struggle to compose long cascades, especially under complex BFCC ordering constraints. We test this by contrasting increased per-step complexity (via more examples) with longer cascades and richer BFCC interactions.
\\
\textbf{H3: Performance on PBEBench predicts performance on real SLI.} 
We evaluate all open-source LLMs on real SLI data and contemporary synthetic inductive reasoning benchmarks, and compare correlations to identify which benchmarks best predict real SLI performance. \\
\textbf{H4: Scaling the thinking budget is more effective than increasing the sampling budget.} We compare two scaling strategies: increasing the thinking budget (maximum sequence length) and the sampling budget, as described in Section~\ref{sec:program_induction:prompting}, to determine which yields more ``bang for the buck'' under a fixed compute budget.
 \section{Related Work}
 \label{sec:related-work}

\textbf{Programming By Example.}
Programming by Example (PBE) \citep{10.1145/1836089.1836091} is a program synthesis paradigm where programs are inferred from input–output pairs.
Early symbolic approaches rely on domain-specific languages and constraints: FlashFill uses a string-transformation DSL \citep{Gulwani2011}, while Syntax-Guided Synthesis (SyGuS) restricts program search via grammars \citep{Alur2013}.
DeepCoder \citep{Balog2017} learns function predictions to guide search, and RobustFill \citep{devlin2017robustfillneuralprogramlearning} trains sequence-to-sequence models to directly emit DSL programs. More recently, Large Language Models show few-shot capability in code generation \citep{Chen2021, guo2024deepseekcoderlargelanguagemodel} and test case generation \citep{li2024largelanguagemodelstest}, but on traditional PBE, they struggle with out-of-distribution data and typically improve after fine-tuning on the target distribution \citep{li2024is, naik2025programming}.

\textbf{Inducing Context-Sensitive Grammars in LLMs.} 
Attempts to induce string-rewrite rules from data have a long history \citep{gildea-jurafsky-1995-automatic}, with linguistic analyses of their formal properties dating back further \citet{Kiparsky1968-KIPLUA,kiparsky1971historical}.
\citet{naik2024can} show that LLMs can induce sound laws, extending this to full context-sensitive program synthesis \citep{naik2025programming}. 
However, these works do not provide provably correct algorithms for detecting feeding and bleeding relations, which is a unique contribution of this work (see Appendix~\ref{sec:theory}).

\textbf{Reasoning and Induction Benchmarks.} 
Several benchmarks have been proposed to evaluate reasoning and PBE capabilities. Code-centric suites such as HumanEval and MBPP assess function generation \citep{Chen2021,Austin2021}, while FlashFill-style datasets target example-driven string transformations \citep{Gulwani2011}. More recent benchmarks emphasize software-engineering workflows \citep{jain2025testgenevalrealworldunit,zhang2025dynamicbenchmarkconstructionevaluating, shao-etal-2025-case2code}. 
However, software engineering tasks often entangle general reasoning with domain-specific knowledge (of tools and libraries), leading to uneven prior exposure and complicating contamination-free evaluation.
Related work on program-aided reasoning uses code as an intermediate representation for solving natural-language problems. PAL \citep{gao2023pal} shows that generating executable Python programs can outperform text-only chain-of-thought on mathematical, symbolic, and algorithmic reasoning tasks, and later work finds that such program-aided reasoners are often better calibrated \citep{kabra2024program} than text-only counterparts .
Beyond code, multi-step, system-2, and mathematical reasoning are explored by HotpotQA, DROP, BIG-Bench Hard, and GSM8K \citep{yang2018hotpotqa,Dua2019DROPAR,suzgun2022challenging,Cobbe2021}.
Inductive rule learning and compositional generalization are probed via CLUTRR and SLR \citep{sinha2019clutrr,helff2025slr}, bAbI and KOR-Bench \citep{weston2015towards,ma2024kor}, and extrapolative splits such as SCAN and CFQ \citep{Lake2018,Keysers2020}; ARC provides a non-language analog \citep{chollet2019measureintelligence}. 
In contrast, PBEBench targets multi-step synthesis of \emph{string-rewriting cascades}, is easily scalable and leakage-resistant, and supports difficulty control through simple generation parameters.
\section{Methodology}
\label{sec:methodology}
Our approach comprises two components: (1) a problem proposer and (2) a problem solver (Fig.~\ref{fig:proposer_and_solver}). 
The proposer is a symbolic system that generates PBE instances with controllable difficulty and forms the core of our dynamic benchmarking framework. 
It scalably produces novel, contamination-free input, output, and program triples of controllable difficulty. 
The solver corresponds to any system under evaluation. 
We benchmark state-of-the-art LLMs on inductive reasoning using a program reordering, multi-step PBE task, and additionally evaluate them on human-curated, low-resource Sound Law Induction (SLI) data.

\begin{figure*}[!tbh]
\centering
\begin{subfigure}[t]{0.48\textwidth}
    \centering
    \includegraphics[width=\linewidth]{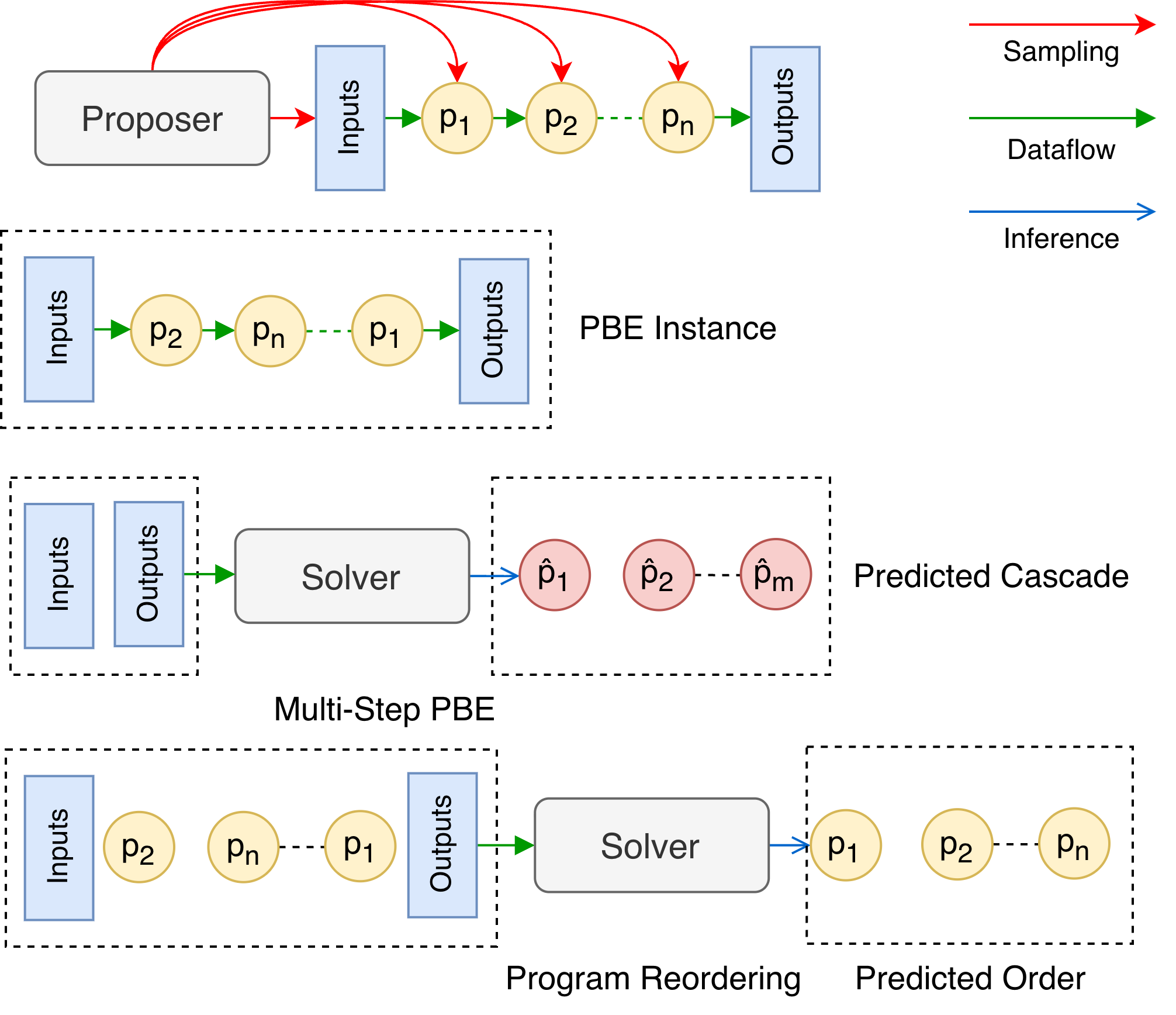}
    \caption{Proposer–solver architecture}
    \label{fig:proposer_and_solver}
\end{subfigure}
\hfill
\begin{subfigure}[t]{0.48\textwidth}
    \centering
    \includegraphics[width=\linewidth]{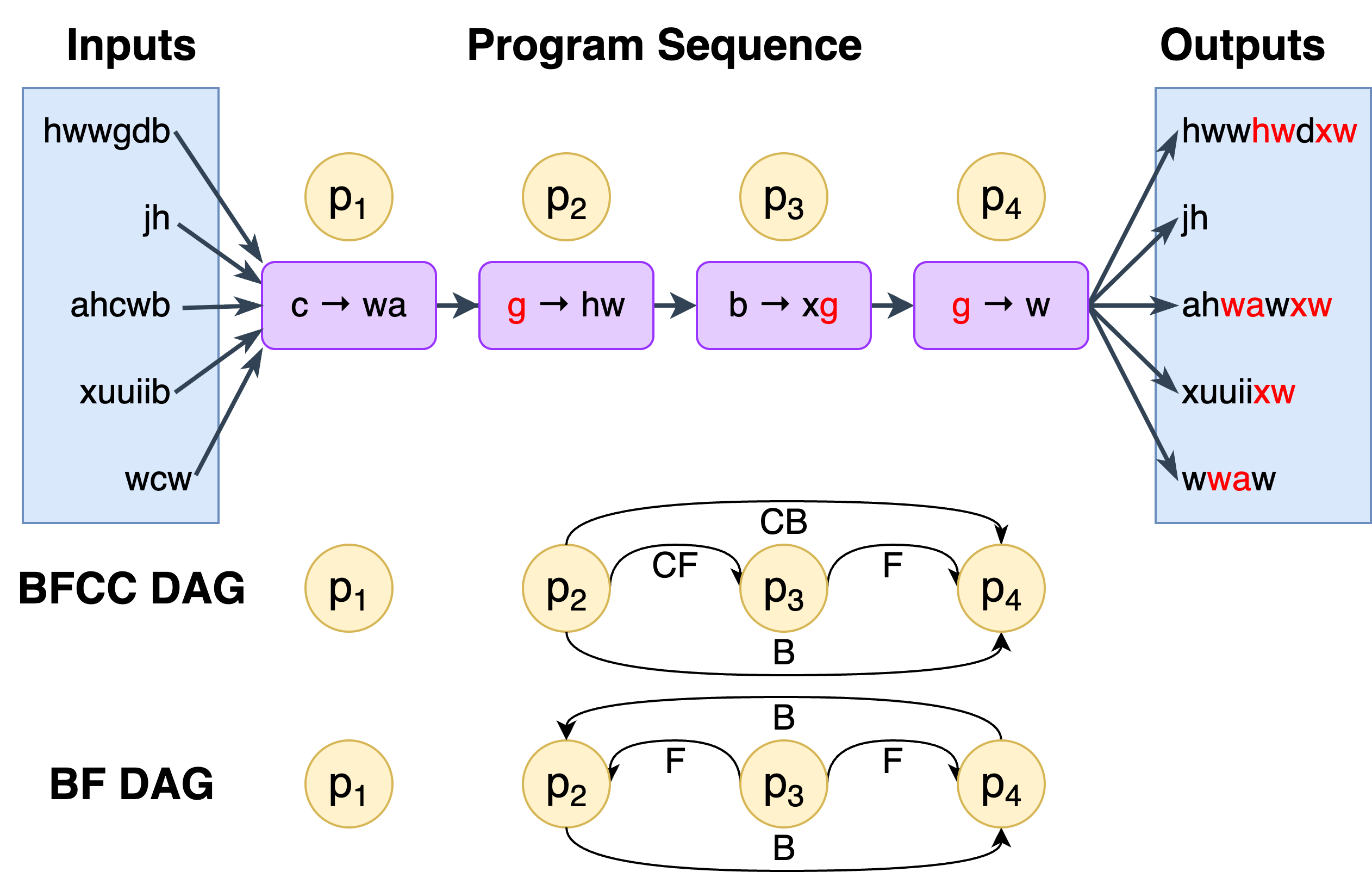}
    \caption{Structure of a PBE data instance}
    \label{fig:data_instance}
\end{subfigure}
\caption{Overview of PBEBench. (a) A symbolic proposer generates PBE instances of controllable difficulty, and a solver attempts to recover the underlying program sequence or ordering. (b) Each instance consists of <$\vec{\imath},\vec{p},\vec{o}$> triples, with program relations encoded by the BFCC DAG and a simplified BF variant that replaces counterfactual relations with a reversed link.}
\label{fig:pbe_overview}
\end{figure*}

\subsection{Problem Proposer}
\label{sec:problem_proposer}
The symbolic proposer (Fig~\ref{fig:proposer_and_solver}, top) constructs data for the multi-step PBE task by sampling inputs from a distribution of strings $\mathcal{I}$ and a sequence of string rewrite programs from a distribution $\mathcal{P}$.
Each program acts as a find-and-replace function, substituting all occurrences of a substring $\alpha$ with a substring $\beta$ (details in Section~\ref{sec:program_sampling_and_output_generation}).
The proposer applies the sampled \textit{cascade} of programs $\vec{p}=\langle p_1,\dots,p_m\rangle \in \mathcal{P}$ to the sampled inputs $\vec{\imath}=\langle i_1,\dots,i_n\rangle \in \mathcal{I}$, producing outputs
$\vec{o}=\langle p_m(\dots p_1(i_1)), \dots, p_m(\dots p_1(i_n))\rangle$.
For convenience, we write this transformation as $\vec{o}=\vec{p}(\vec{\imath})$.
Each dataset instance (Fig~\ref{fig:proposer_and_solver}, second from top) is defined by the triplet $\langle \vec{\imath}, \vec{p}, \vec{o} \rangle$, consisting of the inputs, the program cascade, and the corresponding outputs.
In addition, the proposer records metadata describing interactions between programs (e.g., \textsc{bleeding}, \textsc{feeding}), illustrated in Fig~\ref{fig:data_instance}.

The proposer is parameterized by the number of examples $n$ ($|\vec{\imath}| = |\vec{o}| = n$; $n=5$ in Fig~\ref{fig:data_instance}); the input alphabet $\Sigma$, used to generate both inputs and string rewrite programs; the minimum and maximum input string lengths $l_{min}$ and $l_{max}$ ($l_{min}=2$, $l_{max}=6$ in Fig~\ref{fig:data_instance}); the minimum and maximum cascade lengths $L_{min}$ and $L_{max}$ ($L_{min} \leq |\vec{p}| \leq L_{max}$; $|\vec{p}|=4$, $L_{min}=2$, $L_{max}=5$ in Fig~\ref{fig:data_instance}); the minimum and maximum substring lengths $s_{min}$ and $s_{max}$ for $\alpha$ and $\beta$ in the rewrite programs ($s_{min} \leq |\alpha|,|\beta| \leq s_{max}$); and the desired dataset size $D$.
We denote a benchmark snapshot generated with fixed parameter values as $\mathcal{D}(n, \Sigma, L_{min}, L_{max}, l_{min}, l_{max}, s_{min}, s_{max}, D)$.

For the program reordering task, we take PBE instances (Fig~\ref{fig:proposer_and_solver}, second from top) generated by the proposer and permute the program sequence by swapping program pairs that exhibit \textsc{feeding} or \textsc{bleeding} interactions.
This produces a permuted sequence $\sigma(\vec{p})$ and a different output vector $\vec{o}_{\sigma} = \sigma(\vec{p})(\vec{\imath}) \neq \vec{o}$.
The goal of the LLM is to recover the original program sequence $\vec{p}$ from the original inputs $\vec{\imath}$, original outputs $\vec{o}$, and the permuted program sequence $\sigma(\vec{p})$, such that $\vec{p}(\vec{\imath})=\vec{o}$.

\subsubsection{Input Sampling}
\label{sec:methodology:input_sampling}
Our input sampling procedure is parameterized by $(n, \Sigma, l_{min}, l_{max})$ as defined in Section~\ref{sec:problem_proposer}. 
We build the initial set of input strings $\vec{i_0} = \langle i^{0}_1, \dots, i^{0}_n \rangle$ ($hwwgdb, jh, \dots wcw$ in Fig~\ref{fig:data_instance}) by independently sampling each input $i^0_{j}$ for $1 \leq j \leq n$. 
For each $i^0_j$, we first sample its length from a uniform discrete distribution, $|i^0_j| \sim \mathrm{Unif}\{l_{min}, \dots, l_{max}\}$,
where $|i^0_j|$ denotes the length of the string $i^0_j$. 
We then generate the string itself by sampling $|i^0_j|$ characters independently with replacement from the alphabet $\Sigma$, i.e.,
$i^0_j \sim \mathrm{Unif}(\Sigma)^{|i^0_j|}$.
\subsubsection{Program Sampling and Output Generation}
\label{sec:program_sampling_and_output_generation}
Program sampling is parameterized by $(\Sigma, L_{min}, L_{max}, s_{min}, s_{max})$. 
We begin by programmatically selecting a cascade length $L$ such that $L_{min} \leq L \leq L_{max}$ ($L=4$ in Fig~\ref{fig:data_instance}). 
Next, we construct a sequence of $L$ programs $\vec{p} = \langle p_1, \dots, p_L \rangle$, where each program is of the form $p_k = \mathrm{replace}(\alpha_k, \beta_k)$. 
Here, $\mathrm{replace}$ has the same semantics as Python’s built-in \texttt{replace()} method for strings: the substring $\alpha_k$ is replaced by $\beta_k$, with the restriction that $\alpha_k$ is non-empty.
For example in Fig~\ref{fig:data_instance}, the first program $c \rightarrow wa$ is parameterized by $\alpha_1=c$ and $\beta_1=wa$.
To sample each program $p_k$, we generate $\alpha_k$ and $\beta_k$ independently, following a procedure analogous to input sampling:
\begin{enumerate}[itemsep=1pt, leftmargin=*, parsep=0pt, topsep=-2pt, partopsep=0pt]
    \item Sample the lengths $|\alpha_k|$ and $|\beta_k|$ from a discrete uniform distribution: $|\alpha_k|, |\beta_k| \sim \mathrm{Unif}\{s_{min}, \dots, s_{max}\}$

    \item Sample $\alpha_k$ uniformly from the set of substrings of length $|\alpha_k|$ in the intermediate input vector $\vec{\imath}_{k-1}=p_{k-1}(\dots p_1(\vec{\imath}))$ (where $\mathrm{Substr}_{s}(\vec{\imath}_{k-1})$ denotes the set of all substrings of length $s$ in $\vec{\imath}_{k-1}$): $ \alpha_k \sim \mathrm{Unif}(\mathrm{Substr}_{|\alpha_k|}(\vec{\imath}_{k-1}))$ (note that $\vec{\imath}_0 = \vec{\imath}$).
    For example in Fig~\ref{fig:data_instance}, for $\alpha_2$, we had $|\alpha_2|=1$, thus we would sample uniformly from substrings of length one (basically all characters) present in the intermediate inputs $\vec{i_1} = \langle hwwgdb, jh, ahwawb, xuuiib, wwaw \rangle$.

    \item Sample $\beta_k$ as a sequence of $|\beta_k|$ characters, each drawn independently from the uniform character distribution over the alphabet $\Sigma$ (same as in Section~\ref{sec:methodology:input_sampling}): $\beta_k \sim \mathrm{Unif}(\Sigma)^{|\beta_k|}$
\end{enumerate}
The outputs are then generated by executing the program cascade $\vec{p}$ over the inputs $\vec{\imath}$ as $\vec{o} = \vec{p}(\vec{\imath})$ which as stated above is just recursive application of the replace programs $p_k(p_{k-1}(\dots p_1(\vec{i_o})))$

\subsubsection{Enforcing Complexity Constraints}
\label{sec:rejection_sampling}
We control the complexity of cascades of programs $\vec{p}$ using rejection sampling guided by the classifiers developed in Section~\ref{sec:theory}. 
Complexity is balanced along two dimensions: 
(1) the cascade length $L=|\vec{p}|$, and 
(2) the relation types present between programs in $\vec{p}$. 
The second dimension is encoded as a binary category string per cascade $\vec{p}$, $c_{\vec{p}} \in \{0,1\}^4$.
Each digit is a binary indicator denoting the presence or absence of feeding, bleeding, counter-feeding, and counter-bleeding relations.
Instances are balanced across all 16 possible category strings, from $c_{\vec{p}}=0000$ (no relations present, arbitrary ordering possible) to $c_{\vec{p}}=1111$ (all relations present, ordering highly constrained).

\textbf{Computing Instance Complexity.} 
We hypothesize that the difficulty of a PBE problem $\langle \vec{\imath}, \vec{p}, \vec{o} \rangle$ is governed by (1) cascade length and (2) the types of relations between program pairs in $\vec{p}$ (relations formulated by phonologists and historical linguists; see \citet{Kiparsky1968-KIPLUA,kiparsky1971historical}). 
One of our key contributions is provably correct, automatic, symbolic classification of these relations for a given pair of programs (Appendix~\ref{sec:theory}). 
For a pair of programs ($p_i$,$p_j$) in $\vec{p}$, where $p_i$ is applied before $p_j$, the possible relations are: \\
\textbf{Feeding (F($p_i$, $p_j$)):} $p_i$ creates substrings that enable $p_j$ to apply. \\
\textbf{Bleeding (B($p_i$, $p_j$)):} $p_i$ removes substrings required by $p_j$. \\
\textbf{Counter-Feeding (CF($p_i$, $p_j$)):} $p_i$ could have fed $p_j$, but $p_j$ precedes $p_i$. \\
\textbf{Counter-Bleeding (CB($p_i$, $p_j$)):} $p_i$ could have bled $p_j$, but $p_j$ precedes $p_i$. \\
\textbf{No Relation:} $p_i$ and $p_j$ can be ordered arbitrarily. \\
Counter-relations need not be stored explicitly, since $CF(p_i, p_j)$ and $CB(p_i, p_j)$ are implied if $p_j$ precedes $p_i$ and $F(p_j, p_i)$ and $B(p_j, p_i)$, respectively. 
These relations can be visualized using a Directed Acyclic Graph (DAG), as illustrated in Figure~\ref{fig:data_instance}. 
We show both the full DAG with all relations and a simplified DAG that indirectly encodes counter-relations via this symmetry. \\
\textbf{Rejection Sampling to Control Complexity.} 
To enforce balanced complexity, we use rejection sampling.
Each PBE instance is first assigned a relation-type category string $c_{\vec{p}}$, and data is generated to approximate balance across the $2^4=16$ categories, as well as across cascade lengths between $L_{min}$ and $L_{max}$.
These constraints can conflict, since short cascades (e.g., $L=2$) rarely realize higher-order relation types, making some categories unattainable.
Following prior work on sound law induction \citep{naik2025programmingexamplesmeetshistorical}, we retain random sampling rather than deterministic generation or retrieval to promote diversity and avoid contamination, at the cost of high rejection rates.
We introduce a patience parameter $\tau$, denoting the number of sampling steps under both constraints; once $\tau$ ($=100{,}000$) is exceeded, we also accept instances satisfying just one constraint.
Early experiments prioritized relation-type balance, but later analysis (Appendix~\ref{sec:appendix:factorial_analysis}, \ref{sec:appendix:logistic_regression}) identified cascade length as a stronger predictor of difficulty, leading us to relax relation-type constraints after excceeding $\tau$ steps.
We show the effect of varying $\tau$ on relation-type balance and generation efficiency (Appendix~\ref{sec:appendix:efficiency}).

\subsection{Problem Solver}
\label{sec:program_induction:prompting}
\textbf{Prompting Strategy:} Once a benchmark snapshot $\mathcal{D}$ is generated, we evaluate each LLM $M$ by prompting it for multi-step PBE (second from bottom in Fig~\ref{fig:proposer_and_solver}; prompt in Appendix~\ref{sec:appendix:method_details:multi_step_pbe_prompt_template}) and program reordering (bottom in Fig~\ref{fig:proposer_and_solver}; prompt in Appendix~\ref{sec:appendix:method_details:program_reordering_prompt_template}). \\
For multi-step PBE, the prompt enforces the following constraints: each program must be a Python \texttt{replace} function; both arguments $\alpha_k$ and $\beta_k$ must be strings with $|\alpha_k|, |\beta_k| \leq s_{max}$; $\alpha_k$ must be non-empty; the cascade may contain at most $L_{max}$ programs (matching the longest ground-truth cascade); and outputs must follow a strict Markdown format for reliable extraction.
If a predicted cascade $\hat{\vec{p}}$ exceeds $L_{max}$ programs, only the first $L_{max}$ are evaluated; violations of any other constraint result in replacement with an identity program $p^I$, which leaves inputs unchanged.
Given $\hat{\vec{p}}$, predicted outputs are computed as $\hat{\vec{o}} = \hat{\vec{p}}(\vec{\imath})$.
Details of program extraction from LLM responses are provided in Appendix~\ref{sec:appendix:method_details:program_extraction}.
\\
For the program reordering task, the prompt defines \textsc{bleeding} and \textsc{feeding} relations, provides the ground-truth program cascade along with the inputs and outputs, and instructs the LLM to output only a fenced JSON block containing a permutation of indices (e.g., $[2,1,3]$) that recovers the original program order. 
This task evaluates the model’s ability to reason about ordering constraints induced by BFCC relations between programs.
\\
\textbf{Scaling Solution Search.} We investigate two solution-search scaling strategies, motivated by evidence of reasoning collapse and out-of-token failures in \citep{shojaee2025illusionthinkingunderstandingstrengths} and subsequent discussion \citep{lawsen2025commentillusionthinkingunderstanding}: (1) sampling-budget \citep{li2024is} and (2) test-time thinking-budget scaling \citep{muennighoff2025s1}. 
In the first, the LLM produces $K$ candidate solutions, and any candidate consistent with the input--output examples is accepted; if none succeed, we select the candidate with the highest edit-similarity reward (Section~\ref{sec:experiments:eval_metrics}). 
In the second, inspired by \cite{muennighoff2025s1}, we vary a thinking budget of $N$ tokens. 
However, because gpt-oss-120b does not reliably respect this constraint, even when forced to emit chain-of-thought termination tokens (Appendix~\ref{sec:cot_truncation_experiment}), we instead control the overall maximum sequence length.


\section{Experiments}
\label{sec:experiments}

\subsection{Dataset Creation}
\label{sec:benchmark_details}
Using the problem proposer described in Section~\ref{sec:problem_proposer}, we construct three synthetic benchmarks. \\ 
\textbf{PBEBench-Lite:} A simplified dataset with cascades of length 2 to 5 and 5 input--output pairs per instance, synthesized as 
$\mathcal{D}(n=5,\Sigma=\Sigma_{\text{lite}},L_{min}=2,L_{max}=5,l_{min}=2,l_{max}=6,s_{min}=1,s_{max}=3,D=1008)$,
with exactly 63 examples per relation category. The alphabet is restricted
$\Sigma_{\text{lite}}=\{a,\dots,k,u,v,w,x,y,z\}$ (lowercase letters excluding $l$–$t$). \\  
\textbf{PBEBench-Lite-Perm:} A program reordering dataset with 919 instances created from PBEBench-Lite using the swapping strategy (Section~\ref{sec:problem_proposer}).
\\
\textbf{PBEBench:} a larger dataset with 50 input-output pairs per instance, synthesized as 
$\mathcal{D}(n=50,\Sigma,L_{min}=2,L_{max}=20,l_{min}=2,l_{max}=6,s_{min}=1,s_{max}=3,D=1216)$, with 64 instances for each cascade length. 
Here $\Sigma=\{a,\dots,z,A,\dots,Z\}$ (full alphabet in both cases). \\
\textbf{Real-SLI:} We also evaluate all open-source LLMs on real SLI data drawn from 6 proto--attested language pairs, described in more detail in Appendix~\ref{sec:appendix:real_sli_benchmark}. 
To assess the predictive power of PBEBench for SLI performance, we additionally compare it against other inductive reasoning benchmarks, including CLUTRR \cite{sinha2019clutrr} and SLR-Bench \cite{helff2025slr}.
Some other datasets generated for some ablations are described in Appendix~\ref{sec:appendix:benchmark_stats}.


\subsection{Models Evaluated}
\label{sec:experiments:model_selection}
We evaluate a broad range of state-of-the-art LLMs spanning multiple model families and architectures. 
We cover models from leading developers including OpenAI (GPT, o-series, gpt-oss), Anthropic (Claude series), Google (Gemini), Qwen (Qwen2.5, Qwen3, QwQ), DeepSeek (R1-Distill-Qwen), Mistral (Codestral). 
The models differ in scale, reasoning specialization (thinking vs. non-thinking), architectural choices (dense vs. MoE), and source availability (open vs. closed). 
A full list of models with their attributes is provided in Table~\ref{tab:selected_models_details}, covering the breadth of model families and capabilities.
\subsection{Evaluation Metrics}
\label{sec:experiments:eval_metrics}
For the multi-step PBE task, since a given input vector $\vec{\imath}$ may be mapped to the output vector $\vec{o}$ by multiple program cascades $\vec{p}$, we evaluate model predictions using metrics based on functional equivalence, following prior work in programming by example \cite{li2024is} and historical linguistics \cite{hoenigswald1960language}. 
We execute the model-generated cascade $\hat{\vec{p}}$ on the inputs and compare the predicted outputs $\hat{\vec{o}}=\hat{\vec{p}}(\vec{\imath})$ with $\vec{o}$.

We use two main metrics: a coarse-grained solve rate (\texttt{Pass@1}) and a fine-grained normalized edit similarity (\texttt{Edit\_Sim}):
\[
\mathrm{pass@1} = \frac{1}{|\mathcal{D}|}\sum_{\vec{o},\vec{\imath} \in \mathcal{D}} \mathrm{1}_{\hat{\vec{p}}(\vec{\imath})=\vec{o}}
\]
\[
\mathrm{Edit\_Sim} = \frac{1}{|\mathcal{D}|} \sum_{\vec{o},\vec{\imath} \in \mathcal{D}} 1-\frac{\mathrm{dist}(\hat{\vec{p}}(\vec{\imath}),\vec{o})}{\mathrm{dist}(\vec{\imath},\vec{o})},
\]
where $\mathrm{dist}$ denotes the summed Levenshtein edit distance over corresponding strings.
We additionally report the \texttt{Valid\_Rate}, the proportion of generated programs that satisfy all prompt constraints and \texttt{Complexity} that sums up the length of the substrings $\alpha_k$ and $\beta_k$ in the model predicted program cascade and averages it across PBE instances.
For the program reordering task, we evaluate accuracy (Acc) by applying the predicted permutation $\hat{\sigma}$ to the permuted program sequence and checking whether $\hat{\sigma}(\sigma(\vec{p}))(\vec{\imath})=\vec{o}$ (Section~\ref{eq:perm_acc}).
We also report unique accuracy on instances with a single valid solution.

\begin{table}[!tbh]
\begin{tabular}{@{}lrr@{}}
\toprule
\textbf{Model} & \textbf{Acc} & \textbf{UAcc} \\ \midrule
Codestral-22B & 34.3 & 51.2 \\
Qwen2.5-32B-Instruct & 31 & 39.3 \\
Qwen2.5-Coder-32B-Instruct & 34.7 & 49.6 \\
Qwen3-32B & 36 & 43 \\
Qwen3-Coder-30B-A3B-Instruct \moeemoji & 44.1 & 63.2 \\
DeepSeek-R1-Distill-Qwen-32B \reasoningemoji & 75.5 & 92.1 \\
Qwen3-30B-A3B \reasoningemoji \moeemoji & 68 & 91.3 \\
QwQ-32B \reasoningemoji & 74.9 & 94.6 \\
Qwen3-32B (with CoT) \reasoningemoji & 77.9 & 93.8 \\
gpt-oss-20b \reasoningemoji \moeemoji & 86.3 & 92.6 \\
gpt-oss-120b \reasoningemoji \moeemoji & 97.5 & 98.8 \\ \midrule
Claude-4-Sonnet \unkemoji & 80.2 & 91.3 \\
Claude-4.5-Sonnet \unkemoji & 85.1 & 92.1 \\
Claude-4-Sonnet (Thinking) \reasoningemoji \unkemoji & 91.6 & 97.5 \\
Claude-4.5-Sonnet (Thinking) \reasoningemoji \unkemoji & 97.5 & 99.2 \\
o3-mini \reasoningemoji \unkemoji & 63.4 & 83.7 \\
o4-mini \reasoningemoji \unkemoji & 80.2 & 97.5 \\
Gemini 2.5 Flash \reasoningemoji \unkemoji & 92.8 & 89.7 \\
GPT-5 \reasoningemoji \unkemoji & 99.7  & 99.6 \\ \bottomrule
\end{tabular}
\caption{\textbf{PBEBench-Lite Program Reordering:} We compute accuracy (\texttt{Acc}) based on functional correctness of the unscrambled program sequence on the full data and the unique solution subset (\texttt{UAcc}). \moeemoji\ indicates a mixture-of-experts (or MoE) model. \reasoningemoji\ indicates a reasoning model. \unkemoji\ indicates unknown architecture.}
\label{tab:program_reordering_task}
\end{table}

\begin{table*}[t]
\centering
\begin{tabular}{@{}lrrrr@{}}
\toprule
\textbf{Model} & \multicolumn{1}{l}{\textbf{Pass@1}} & \multicolumn{1}{l}{\textbf{Edit Sim}} & \multicolumn{1}{l}{\textbf{Complexity}} & \textbf{Valid Rate} \\ \midrule
Codestral-22B & 1.1 & -1.1 & 15.4 & 82.5 \\
Qwen2.5-32B-Instruct & 3 & 12.5 & 12.5 & 82.9 \\
Qwen2.5-Coder-32B-Instruct & 4.1 & 18.8 & 12.67 & 68.9 \\
Qwen3-32B & 1.8 & 9.6 & 14.89 & 76.5 \\
Qwen3-Coder-30B-A3B-Instruct \moeemoji & 3.8 & 8.6 & 3.25 & 81.3 \\
DeepSeek-R1-Distill-Qwen-32B \reasoningemoji & 22.4 & 34.9 & 7.43 & 87.1 \\
Qwen3-30B-A3B \reasoningemoji \moeemoji & 28.9 & 33.6 & \textbf{4} & \textbf{99.1} \\
QwQ-32B \reasoningemoji & 36 & 40.9 & 4.72 & 94.9 \\
Qwen3-32B (with CoT) \reasoningemoji & 41.9 & 50.3 & 6.57 & 96.8 \\
gpt-oss-20b \reasoningemoji \moeemoji & 40.6 & 46.2 & 6.96 & 99 \\
gpt-oss-120b \reasoningemoji \moeemoji & \textbf{62.5} & \textbf{69.9} & 10.93 & 92.5 \\
\midrule
Claude-3.5-Sonnet \unkemoji & 18.5 & 44.3 & 11.87 & 82.1 \\
Claude-3.7-Sonnet \unkemoji & 23.2 & 50 & 12.05 & 84.1 \\
Claude-4-Sonnet \unkemoji & 29.7 & 58.7 & 13.35 & 77.2 \\
Claude-3.7-Sonnet (Thinking) \reasoningemoji \unkemoji & 36.6 & 61.5 & 12.94 & 0.819 \\
Claude-4-Sonnet (Thinking) \reasoningemoji \unkemoji & 35.8 & 60.8 & 13.71 & 78.2 \\
Claude-4-Opus (Thinking) \reasoningemoji \unkemoji & 53.9 & 75.2 & 13.44 & 85.6 \\ 
o3-mini \reasoningemoji \unkemoji & 19.6 & 19.8 & \textbf{1.65} & \textbf{99.5} \\
o4-mini \reasoningemoji \unkemoji & 36.3 & 37.4 & 4.08 & 97.3 \\
Gemini 2.5 Flash \reasoningemoji \unkemoji & 58.6 & 65.6 & 8.36 & 79 \\
\textbf{GPT-5 \reasoningemoji \unkemoji} & \textbf{72.4} & \textbf{76.5} & 10.58 & 92.9 \\
\bottomrule
\end{tabular}
\caption{\textbf{PBEBench-Lite Performance:} We compute the \texttt{Pass@1} and \texttt{Edit\_Sim} as the coarse and fine-grained evaluation, respectively, for each model. \moeemoji\ indicates a mixture-of-experts (or MoE) model. \reasoningemoji\ indicates a reasoning model. \unkemoji\ indicates unknown architecture. * indicates evaluated on 20\% of the dataset due to cost.}
\label{tab:model_performance}
\end{table*}

\section{Results}

\subsection{PBEBench-Lite Performance}
\label{sec:results:pbebench_lite_performance}
Table~\ref{tab:model_performance} reports \texttt{Pass@1}, \texttt{Edit\_Sim}, \texttt{Valid\_Rate} (expressed as percentages), and \texttt{Complexity} for all models (Section~\ref{sec:experiments:model_selection}). 
Hyperparameters used in model evaluations are listed in Table~\ref{tab:sampling_parameters}. 
Claude-4 Opus (Thinking) was evaluated on only 20\% of the data due to high API costs (run costs are reported in Appendix~\ref{sec:appendix:model_costs}). 
Reasoning models, both open and closed-source, consistently outperform non-reasoning models, with a larger gap among open-source models. 
The top-performing models are gpt-oss-120b (open) and GPT-5 (closed). 
Surprisingly, Qwen3-30B-A3B and o3-mini produce the simplest programs yet perform poorly, suggesting underthinking, while Claude-4-Sonnet and Qwen3-32B without CoT produce the most complex programs, indicative of overthinking.
The ground-truth programs have a mean complexity of 11.63, indicating that even the best-performing models, GPT-5 and gpt-oss-120b, generate simpler programs than the ground truth.
Factorial analysis of QwQ-32B and GPT-5 identifies cascade length as the strongest negative predictor of \texttt{Pass@1}. 
We further find that feeding and bleeding decrease \texttt{Pass@1}, whereas counter-feeding and counter-bleeding have no significant effect (Appendix~\ref{sec:appendix:factorial_analysis}).
Finally, confusion-matrix analyses of predicted versus ground-truth cascade lengths and relation types (Appendix~\ref{sec:appendix:conf_mat_clen}, \ref{sec:appendix:conf_mat_reln}) show that models favor simpler cascades with fewer relations and programs, typically succeeding only when a functionally equivalent solution exists.

\subsection{PBEBench-Lite-Perm Performance}
\label{sec:pbebench_lite_perm_performance}
Table~\ref{tab:program_reordering_task} reports \texttt{Acc} and \texttt{UAcc} for the PBEBench-Lite-Perm program reordering dataset.
We could not run Claude-3.5-Sonnet and Claude-3.7-Sonnet due to deprecation, and Claude-4-Opus due to budget constraints, and therefore evaluate Claude-4.5-Sonnet with and without reasoning.
While LLMs may sometimes bypass BFCC relations, this task requires explicit reasoning to recover correct ordering, with the unique-solution subset admitting only one valid order.
Using PBEBench-Lite with short cascades (2--5) allows some models to enumerate all orderings, simplifying the task relative to full multi-step PBE.
GPT-5 nearly solves the task with 99.7\% accuracy, whereas Qwen2.5-32B-Instruct performs worst at 31\%.
Logistic regression analysis (Table~\ref{tab:pbebench_lite_perm_difficuly_analysis}) reveals that LLMs predictably struggle more with longer cascades and counter relations, like \textsc{counter-feeding}.
Finally, \texttt{UAcc} is generally higher than \texttt{Acc}, explained by the fact that 182 of 242 unique-solution instances have cascade length 2, where random guessing already yields 50\% accuracy.
We also present a qualitative analysis of the various solution strategies adopted by the evaluated models in Appendix~\ref{sec:appendix:qual_analysis_program_reordering}.

\subsection{PBEBench Performance}
\label{sec:results_and_discussion:factorial_analysis}
We evaluate the strongest open and closed-source models (gpt-oss-120b and GPT-5) on PBEBench across cascade lengths ranging from 2 to 30.
GPT-5 consistently outperforms gpt-oss-120b, maintaining non-trivial \texttt{Pass@1} at longer cascades where gpt-oss-120b performance collapses.
At length 20, GPT-5 achieves moderate
accuracy with a small sampling budget, while gpt-oss-120b drops to near-zero performance even with larger budgets.
Performance declines monotonically with cascade length for both models, revealing clear reasoning limits.
Trends across lengths are summarized in Fig.~\ref{fig:cascade_length_perf_gpt_oss_120b} and Fig.~\ref{fig:cascade_length_perf_gpt5}.
A regression analysis further confirms cascade length as a dominant negative predictor of success, with the presence of specific BFCC relations associated with failure or success.
Full results and analyses are deferred to Appendix~\ref{sec:appendix:pbebench_performance}.

\subsection{Ablations and Scaling}
\label{sec:experiment_results_ablations}
We perform ablations using PBEBench-Lite and its variant PBEBench-Lite-MoreEg (Appendix~\ref{sec:appendix:benchmark_stats}) to analyze the effects of just increasing input-output examples on the gpt-oss models. 
Increasing the number of examples per PBE instance results in a slight reduction in performance.
We also analyze the effect of the scaling strategies: sampling budget and thinking budget on GPT-5 and gpt-oss-120b on the harder PBEBench data.
We observe consistent scaling behavior for both models, with rapid initial gains from increased compute, followed by diminishing returns and ``hitting a wall'' for longer ground truth cascades.
Detailed quantitative results, scaling curves, and variance analyses are provided in Appendix~\ref{sec:appendix:ablations}.

\subsection{Real SLI Performance}
\label{sec:results:real_sli_perf}
We evaluate all open-source models on real SLI data (Appendix~\ref{sec:appendix:real_sli_benchmark}), with results in Table~\ref{tab:real_sli_results}.
We additionally report performance on CLUTRR \cite{sinha2019clutrr} and SLR-Bench \cite{helff2025slr} in Appendix~\ref{sec:appendix:other_inductive_reasoning_benchmarks}.
Motivated by our hypothesis that SLI is bottlenecked by multi-step inductive reasoning (Section~\ref{sec:introduction:research_questions}), we measure Pearson rank correlations between model rankings on real SLI and those induced by PBEBench-Lite, CLUTRR, and SLR-Bench (Table~\ref{tab:benchmark_ranking_correlations}).
PBEBench-Lite shows the strongest correlation with real SLI ($r = 0.8273$, $p = 0.001677$), followed closely by SLR-Bench ($r = 0.8091$, $p = 0.002559$).
However, SLR-Bench is sensitive to Prolog coding proficiency, with Codestral-22B achieving high syntax scores but weak overall performance relative to Qwen models, whose reasoning is stronger despite poorer Prolog syntax (Table~\ref{tab:slrbench_results}).
By relying on Python, PBEBench reduces syntax bottlenecks that would otherwise underestimate inductive capability. 

\begin{figure}[!tbh]
    \centering
    \begin{subfigure}{0.48\textwidth}
        \centering
        \includegraphics[width=\linewidth]{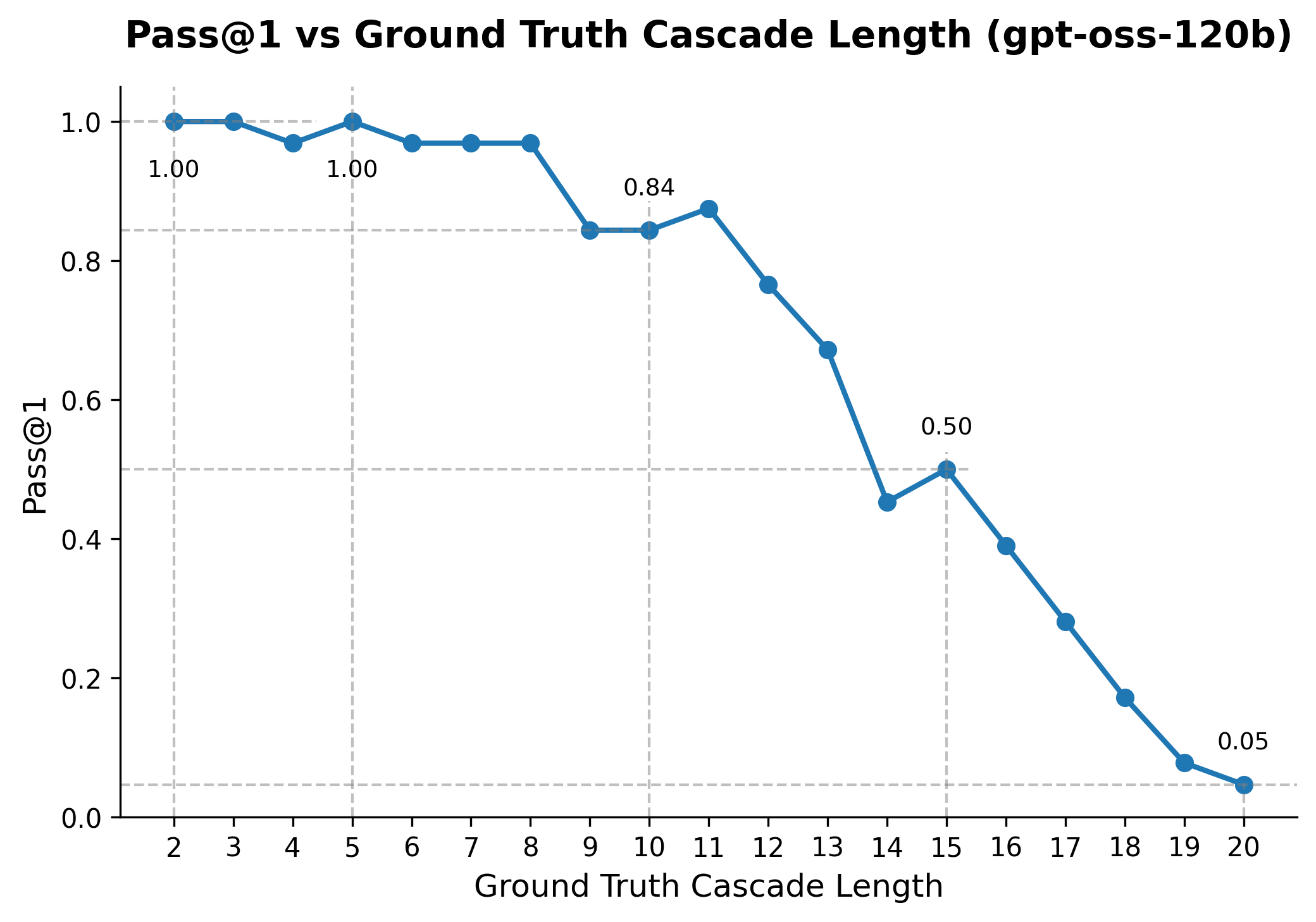}
        \caption{gpt-oss-120b (sampling budget: 32, max sequence length: 16384, cascade: 2-20)}
        \label{fig:cascade_length_perf_gpt_oss_120b}
    \end{subfigure}
    \hfill
    \begin{subfigure}{0.48\textwidth}
        \centering
        \includegraphics[width=\linewidth]{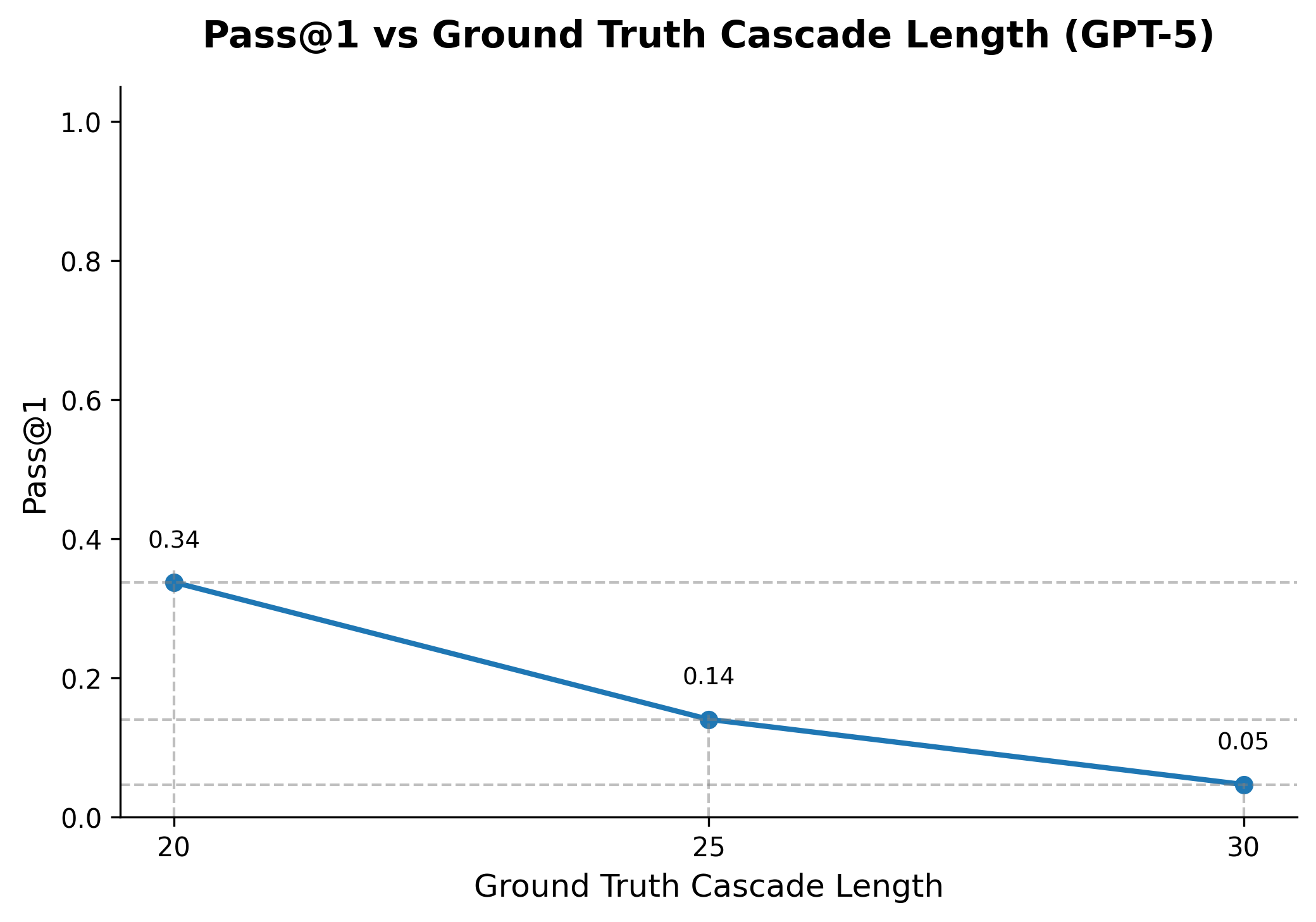}
        \caption{GPT-5 (sampling budget: 4, max completion tokens: 65536, reasoning (medium), cascade: 20, 25, 30)}
        \label{fig:cascade_length_perf_gpt5}
    \end{subfigure}
    \caption{\textbf{Performance across Cascade Lengths on PBEBench:} \texttt{Pass@1} for gpt-oss-120b and GPT-5 for various ground truth cascade lengths, a key difficulty measure, shows where inductive reasoning fails and problems become nearly unsolvable despite high compute budgets.}
\end{figure}

\section{Discussion}
After analyzing the results, we address all the hypotheses in Section~\ref{sec:introduction:research_questions}.
\textbf{H1: LLMs struggle with multi-step inductive reasoning.} As discussed in Section~\ref{sec:results:pbebench_lite_performance}, we see non-reasoning LLMs struggle to solve even simple multi-step PBE problems, and most open-source reasoning models outperform closed-source non-reasoning ones. 
Additionally, even the best performing reasoning models, GPT-5 and gpt-oss-120b, break when evaluated on the hard PBEBench data (Section~\ref{sec:results_and_discussion:factorial_analysis}) despite inference-time scaling strategies (Section~\ref{sec:experiment_results_ablations}).
\textbf{H2: Difficulty in our dataset is driven by interaction structure.}: We find that increasing per-step difficulty by adding more examples has little effect on performance (Section~\ref{sec:experiment_results_ablations}). In contrast, factorial and logistic regression analyses (Sections~\ref{sec:results:pbebench_lite_performance},\ref{sec:results_and_discussion:factorial_analysis}) and performance degradation on longer cascades in PBEBench (Figs.~\ref{fig:cascade_length_perf_gpt_oss_120b},\ref{fig:cascade_length_perf_gpt5}) show that LLMs struggle primarily with long cascades and complex BFCC interactions. Moreover, many models, including reasoning LLMs, fail to recover the correct program order in PBEBench-Lite-Perm even when given the correct shuffled programs for cascades of length 2 to 5, indicating that task difficulty arises from interactions within the cascade rather than individual steps.
\textbf{H3: Performance on PBEBench predicts performance on real SLI}: As discussed in Section~\ref{sec:results:real_sli_perf}, even the simpler PBEBench-Lite has higher correlation with the model rankings on the real SLI data compared to CLUTRR and SLR-Bench, indicating the predictive power of our approach.  
\textbf{H4: Scaling the thinking budget is more effective than increasing the sampling budget.} A direct comparison for gpt-oss-120b shows that increasing sequence length with sufficient sampling is more time and cost-efficient than increasing sampling budget alone, as it reduces wasted attempts from unterminated chains of thought.

\paragraph{Computational complexity and symbolic solvers.}
A key question is whether PBEBench instances are inherently hard or simply poorly matched to current models. Under a naive enumeration view, the search space grows exponentially with cascade length. Each rewrite rule selects substrings over an alphabet $\Sigma$ (e.g., of length $1$--$3$), yielding $O(|\Sigma|^k)$ possibilities per side and $O(|\Sigma|^{2k})$ candidate rules. A cascade of length $L$ therefore induces a hypothesis space of size $O(|\Sigma|^{2kL})$, which is exponential in $L$. While this is not a formal hardness proof, it highlights that exhaustive search quickly becomes intractable as cascade length increases. 
Unlike classical program synthesis settings such as FlashFill, our benchmark does not assume invertibility or structural constraints that enable dynamic programming or polynomial-time synthesis. Rule interactions are order-sensitive and do not decompose into independent subproblems, limiting the applicability of DP formulations. Constraint-based or symbolic solvers could, in principle, enumerate and verify candidate cascades, but absent a strong pruning structure, they still operate over an exponential search space. This positions PBEBench in a regime where both neural and symbolic methods require non-trivial search strategies.

\paragraph{LLMs as heuristic search and solution multiplicity.}
Given the combinatorial nature of the task, we do not view LLMs as complete solvers, but rather as heuristic proposal mechanisms that impose structure on the search space, motivating hybrid neuro-symbolic approaches. 
We also note that valid cascades are not necessarily unique, as multiple programs may satisfy the same input-output examples. 
Importantly, the source of multiplicity differs across settings. 
In the full PBEBench task, both the set of programs and their ordering are unknown, leading to a large (and potentially unbounded) space of valid solutions. 
In contrast, PBEBench-Lite-Perm isolates \emph{ordering} by providing the correct set of programs and requiring only recovery of a valid permutation under BFCC constraints. 
In this restricted setting, if a cascade has length $n$ and $k$ rules participate in BFCC constraints, the number of valid orderings is upper bounded by $(n-k+1)!$, since constrained rules must preserve relative order while the remaining blocks can be permuted. 
For PBEBench-Lite-Perm ($n \leq 5$ with at least one BFCC constraint), this yields at most $4! = 24$ solutions, which is modest and could be enumerated. Despite this, performance still degrades with cascade length and counter-relations, indicating that interaction structure, rather than solution multiplicity, is the primary source of difficulty.
\section{Conclusion and Future Work}
We show that the proposed benchmark is challenging for many models, especially non-reasoning LLMs, because despite simple individual PBE steps, its sequential structure and BFCC-induced ordering constraints create emergent complexity that requires long chains of thought at test time.
Increasing cascade length or interaction complexity breaks even the strongest reasoning models, including GPT-5 and gpt-oss-120b.
It is also practically meaningful: PBEBench-Lite performance best predicts real SLI performance, suggesting that LLMs’ difficulty with multi-step inductive reasoning partly explains their poor real-world SLI results.
In future work, we plan to synthesize progressively harder training data using PBEBench and explore hybrid neuro-symbolic data generation.
\section*{Acknowledgments}

This material is based upon work supported in part by the National Science Foundation under Grant No. 2504019. Any opinions, findings, and conclusions or recommendations expressed in this material are those of the author(s) and do not necessarily reflect the views of the National Science Foundation.

\section*{Limitations}
\label{sec:limitations}

Despite the contributions of this work, several limitations remain:
\begin{enumerate}[itemsep=1pt, leftmargin=*, parsep=0pt, topsep=-2pt, partopsep=0pt]
    \item We primarily evaluate the strongest open and closed-source models on PBEBench. Due to time and compute constraints, we don't extensively benchmark weaker models; however, results on PBEBench-Lite indicate that such models consistently underperform relative to gpt-oss-120b and GPT-5, especially on longer cascades.
    \item Owing to budget and time constraints, evaluations on CLUTRR, SLR-Bench, and real SLI data are limited to open-source models.
    \item For PBEBench-Lite, PBEBench-Lite-MoreEg, CLUTRR, and SLR-Bench, we use a single sample per problem (sampling budget of 1), which may underestimate performance under larger sampling budgets.
    \item For real SLI, we simplify the task by chunking cases with more than 200 examples into subsets of 50, relaxing program constraints by allowing longer cascades and larger substring rewrites. While this could enable degenerate solutions, most models still perform poorly even with increased thinking and sampling budgets.
    \item Due to cost, GPT-5 is evaluated only on the most challenging cascade lengths and with limited sampling budgets, restricting comprehensive analysis across settings.
    \item The random sampling procedure cannot perfectly balance constraints such as relation types and cascade lengths, particularly for long cascades where valid examples are rare, leading to high rejection rates and exhaustion of the patience parameter~$\tau$. Future work could explore more efficient sampling strategies.
    \item In PBEBench, individual programs modify relatively few examples on average, which is less representative of real sound law induction, where rules often apply to many words simultaneously.
\end{enumerate}

\section*{Ethics Statement}
Our work introduces a provably correct, fully automated, and scalable pipeline for generating multi-step Programming by Example (PBE) problems, enabling the evaluation of inductive reasoning in Large Language Models without human supervision. As no human subjects are involved (apart from the individuals who contributed the published data for the SLI task), we do not anticipate ethical concerns. Moreover, since the dataset focuses solely on assessing abstract inductive reasoning in LLMs, it is unlikely to be misused for harmful purposes.

\bibliography{custom}

\appendix
\label{sec:appendix}
\appendix

\section*{Appendix}

The Appendix is organized as follows:  
\begin{enumerate}[itemsep=1pt, leftmargin=*, parsep=0pt, topsep=-2pt, partopsep=0pt]
    \item Appendix~\ref{sec:appendix:reproducibility} discusses steps taken to ensure reproducibility.   
    \item Appendix~\ref{sec:theory} provides the theory and proofs for the relation type classifiers used to control the distribution.
    \item Appendix~\ref{sec:appendix:benchmark_details} details the benchmark and metrics, including licensing, statistics, and related information. 
    \item Appendix~\ref{sec:appendix:method_details} gives additional methodological details, including prompts, model selection, licensing, run costs, algorithm pseudocode, snapshots for closed-source models, and the K-fold evaluation used to simulate GPT-5’s sampling budget.  
    \item Appendix~\ref{sec:appendix:experimental_details} outlines experimental details such as computational environment, inference parameters, and strategies for CoT truncation/control with gpt-oss-120b. 
    \item Appendix~\ref{sec:appendix:additional_results} presents additional results, including data generation efficiency, factorial and logistic regression analyses of instance difficulty, confusion matrices comparing dataset ground truth and model predictions, and extended tables for scaling experiments, effects of additional examples, and performance breakdowns by cascade length.
\end{enumerate}

\section{Reproducibility Statement}
\label{sec:appendix:reproducibility} 
To facilitate the reproduction of our results, we thoroughly document all components of our work. The BFCC relation types, their detectors (Algorithm~\ref{alg:classify_bfcc}), and proofs of correctness are presented in Appendix~\ref{sec:theory}. The rejection-sampling-based data generation process is both described (Section~\ref{sec:problem_proposer}) and formally specified (Algorithm~\ref{alg:balanced_sampling}). 
Additionally, the algorithm for the feeding-bleeding swap permutation for creating the program reordering data is documented in Algorithm~\ref{alg:fb_swap_permutation}.
Detailed statistics and parameters for all generated datasets appear in Section~\ref{sec:benchmark_details} and Appendix~\ref{sec:appendix:benchmark_stats}. Our prompting and program extraction strategy, including scaling techniques, is reported in Section~\ref{sec:program_induction:prompting}, while the prompt template is provided in Appendix~\ref{sec:appendix:method_details:prompt_template}, and inference parameters for each model are given in Appendix~\ref{sec:appendix:experimental_details:sampling_parameters}. 
For closed-source models, we use specific dated snapshots/checkpoints of models as recorded in Appendix~\ref{sec:appendix:model_snapshots} to aid with reproducibility. To estimate the amount of variance between runs/experiments, we do a small controlled experiment for gpt-oss-120b (Table~\ref{tab:gpt_oss_120b_run_variance}), and compute k-fold average for GPT-5 (Appendix \ref{sec:appendix:gpt_5_k_fold_analysis})
The evaluation procedure and metrics are described in Section~\ref{sec:experiments:eval_metrics} and Section~\ref{sec:program_induction:prompting}.
\section{Theoretical Framework}
\label{sec:theory}
This section provides the proof of correctness of our proposed method for automatically classifying the type of relation between any pair of string-rewrite programs, which is one of our novel contributions. 

We propose the function $\mathsf{feeds}(\cdot, \cdot)$ (equation~\ref{eq:1}), which classifies pairs of rules as feeding or not feeding.
\begin{figure*}[t]
\centering
\begin{equation}
  \label{eq:1}
  \mathsf{feeds}(s_i \rightarrow t_i, s_j \rightarrow t_j) =
  \begin{cases}
    \top & t_i = \varepsilon \wedge |s_j| > 1\\
    \top & t_i \in \mathsf{Substr}(s_j) \wedge t_i \notin \mathsf{Substr}(s_i)\\
    \top & t_j \in (\mathsf{Substr}(t_i) \setminus \mathsf{Substr}(s_i))\\
    \top & \mathsf{Pref}(t_i) \setminus \mathsf{Substr}(s_i) \cap \mathsf{Suff}(s_j) \neq \emptyset\\
    \top & \mathsf{Suff}(t_i) \setminus \mathsf{Substr}(s_i) \cap \mathsf{Pref}(s_j) \neq \emptyset\\
    \bot & \text{otherwise}\\
  \end{cases}
\end{equation}
\end{figure*}

where $\mathsf{Pref}(s)$, $\mathsf{Suff}(s)$, and $\mathsf{Substr}(s)$ are the multisets of prefixes, suffixes, and substrings of $s$, respectively.


\begin{definition}[Feeding]
  Feeding is a relation between pairs of rules $p_i=s_i \rightarrow t_i$ and $p_j=s_j \rightarrow t_j$, such that $\exists s, t \in \Sigma^*$ such that $s \xrightarrow{p_i} t$ and $t$ includes a string $w$ that meets the structural description of $p_j$ but is not present in $s$.
\end{definition}
\begin{definition}[Bleeding]
Bleeding is a relation between pairs of rules $p_i=s_i \rightarrow t_i$ and $p_j=s_j \rightarrow t_j$, such that $\exists s, t' \in \Sigma^*$ such that $s_i \xrightarrow{p_i} t_i$ and $s_i$ includes a string $w$ that meets the structural description of $p_j$ but is not present in $t_i$.
\end{definition}
 \begin{definition}[$\mathsf{Substr}$]
   $\mathsf{Substr}(s)$ denotes the multiset of substrings of $s$, counting multiple occurances separately.
 \end{definition}
 \begin{lemma}
   \label{lemma:sufficient}
   If $\mathsf{feeds}(p_i, p_j)$ then $p_i \text{ feeds } p_j$.
 \end{lemma}
 \begin{proof}
   Given $u, v, o, s_i, t_i, s_j, t_j \in \Sigma^*$, $s_i \xrightarrow{p_i} t_i$, and $s_j \xrightarrow{p_j} t_j$ there are four types of transformations of $u$ by applying $p_i$ that will yield $v$ such that $s_j \sqsubseteq v$ (where $\sqsubseteq$ indicates ``is a substring of'').
   (1) \textbf{Deletion.} Assume that $t_i=\varepsilon$. $\exists wx \in \Sigma^+$ such that $ws_ix \xrightarrow{p_i} xw$. If $s_j=xw$ then $p_i$ feeds $p_j$.
   (2) \textbf{Containment.} $t_i \sqsubseteq s_j \land t_i \not\sqsubseteq s_i,\ \exists w, x \in \Sigma+$ such that $w \xrightarrow{p_i} x \land s_j \sqsubseteq x \land s_j \not\sqsubseteq x$.
   (3) \textbf{Subsumption.} Assume that $s_j \in \mathsf{Substr}(t_i) \setminus \mathsf{Substr}(s_i)$. Given $s_i \xrightarrow{p_i} t_i$, $t_i$ will always contain instances of $s_j$ not present in $s_i$, entailing that $p_i$ feeds $p_j$.
   (4) \textbf{Completion.} Assume that $t_i=uo$ and $s_j=ov$ (so that $o$ is a suffix of $t_i$ and a prefix of $s_j$). $s_iov \xrightarrow{p_i} t_iov = uov = us_j$, entailing that $p_i$ feeds $p_j$ (as with $t_i=ou$ and $s_j=vo$, \textit{mutatis mutandis}).
\end{proof}

\begin{lemma}
  \label{lemma:necessary}
  If $\neg\mathsf{feeds}(p_i, p_j)$ then $p_i \text{ does not feed } p_j$
\end{lemma}
 \begin{proof}
   Given $s_i,t_i,s_j,t_ju \in \Sigma^*$, assume for the sake of contradiction two rewrite rules $s_i \xrightarrow{p_i} t_i$ and $s_j \xrightarrow t_j$ such that $p_i$ feeds $p_j$ but $s_i$, $t_i$, and $s_j$ do not satisfy any of the following conditions: \textbf{Deletion.} $s_i \neq \varepsilon \lor s_j \neq wx \forall w, x \in \Sigma^+$, \textbf{Containment.} $t \neg\sqsubseteq s_j \lor t_i \sqsubseteq s_i$ \textbf{Subsumption.} $s_j$ does not occur in $t_i$ except where it occurs in $s_i$. \textbf{Completion.} $\not\exists u, o, v \text{ such that} (t_i=ou \land s_j=vo) \lor t_i=uo \land s_j=ov$. Either $t_i$ is a non-empty string neither containing nor being contained by $s_j$ and sharing no prefix or suffix with $s_j$ or replacing $s_i$ with $t_i$ derives no instances of $s_j$. The first case must be false, since the conditions exhaust the transformations that could yield a string containing $s_j$. The second case must be false, because it contradicts the definition of feeding.
\end{proof}

\begin{theorem}[Feeding]
  \label{theorem:feeding}
  A rule $s_i \rightarrow t_i$ feeds a rule $s_j \rightarrow t_j$ iff $\mathsf{feeds}(s_i \rightarrow t_i, s_j \rightarrow t_j)$
\end{theorem}
\begin{proof}
  Given two rules $p_i=s_i \rightarrow t_i$ and $s_j \rightarrow t_j$, Lemma \ref{lemma:sufficient} proves by enumerating cases that each of the conditions defined for $\mathsf{feed}(p_i, p_j)$ are sufficient for establishing that $p_i$ feeds $p_j$. Lemma \ref{lemma:necessary} proves by enumerating cases than $p_i$ does not feed $p_j$ if none of these conditions are satisfied.
\end{proof}
\begin{theorem}[Bleeding]
  \label{theorem:bleeding}
  A rule $p_i=s_i \rightarrow t_i$ bleeds a rule $p_j=s_j \rightarrow t_j$ iff $\mathsf{feeds}(t_i \rightarrow s_i, s_j \rightarrow t_j)$
  
\end{theorem}

\begin{proof}
  if $\exists u, v \in \Sigma^*, u \xrightarrow{t_i \rightarrow s_i} v$ such that $s_j \sqsubseteq v \land s_j \not\sqsubseteq u$, it follows that mapping $s_i \xrightarrow{p_i} t_i$ bleeds $p_j$ (where $s_j \xrightarrow{p_j} t_j$).
\end{proof}
\section{Benchmark Details}
\label{sec:appendix:benchmark_details}
In this section, we discuss issues such as licensing and the distributional statistics of all the data snapshots created and used in our work.

\subsection{Licensing}
We create a provably correct, fully automated, and scalable pipeline for generating multi-step Programming by Example (PBE) problems, enabling the evaluation of inductive reasoning in LLMs.
We plan to release all the benchmark snapshots (PBEBench-Lite, PBEBench-Lite-MoreEg, PBEBench, PBEBench (25, 30)) under the CC BY-SA 4.0 license.
Additionally, we produce code that allows you to generate more snapshots, which we also release under the MIT license.

\subsection{Benchmark Statistics}
\label{sec:appendix:benchmark_stats}
We report the distributional statistics (like distribution of ground truth cascade lengths or relation types) of all the benchmark snapshots used in this paper below:

\textbf{PBEBench-Lite:} We generate a relation type balanced dataset with 1008 instances, 5 examples per PBE problem, and 63 instances per relation type category with cascade lengths ranging from 2 to 5. 
The alphabet spans $\Sigma_{\text{lite}}=\{a,\dots,k,u,v,w,x,y,z\}$, each input example contains 2 to 6 letters, and each rule has 1-3 characters on either side of the replace function. 
The distribution of cascades and relation types is shown in Fig.~\ref{fig:pbebench_lite_cascade} and Fig.~\ref{fig:pbebench_lite_reln_type}.

\textbf{PBEBench-Lite-MoreEg:} We generate a relation type balanced dataset with 240 instances, 50 examples per PBE problem and 15 instances per relation type category with cascade lengths ranging from 1 to 5. 
The alphabet spans $\Sigma_{\text{lite}}=\{a,\dots,k,u,v,w,x,y,z\}$, each input example contains 2 to 6 letters and each rule has 1-3 characters on either side of the replace function.
The distribution of cascades and relation types is shown in Fig.~\ref{fig:pbebench_lite_more_eg_cascade} and Fig.~\ref{fig:pbebench_lite_more_eg_reln_type}.

\textbf{PBEBench:} We generate a cascade balanced dataset with 1216 instances, 50 examples per PBE problem, and 64 instances per cascade length ranging from 2 to 20. 
The alphabet spans $\Sigma=\{a,\dots,z,A,\dots,Z\}$, each input example contains 2 to 6 letters and each rule has 1-3 characters on either side of the replace function.
The distribution of cascades and relation types is shown in Fig.~\ref{fig:pbebench_cascade} and Fig.~\ref{fig:pbebench_reln_type}.

\textbf{PBEBench (25, 30):} We generate a cascade balanced dataset with 128 instances, 50 examples per PBE problem, and 64 instances per cascade for cascade length of 25 and 30. 
The alphabet spans $\Sigma=\{a,\dots,z,A,\dots,Z\}$, each input example contains 2 to 6 letters and each rule has 1-3 characters on either side of the replace function.
The distribution of cascades and relation types is shown in Fig.~\ref{fig:pbebench_25_30_cascade} and Fig.~\ref{fig:pbebench_25_30_reln_type}.

\subsection{Real SLI Benchmark}
\label{sec:appendix:real_sli_benchmark}
We construct a benchmark of 21 instances derived from real SLI data, covering 6 proto-language and attested-language pairs: prototangkhulic-huishu, austronesian-hawaiian, austronesian-niue, austronesian-tongan, austronesian-rarotongan, and austronesian-samoan. 
Our data is derived from \cite{naik2025programmingexamplesmeetshistorical, naik2024can} original datasets.
Each instance consists of 50 example input and output pairs and uses the same prompt template as the multi-step PBE task prompt employed for all synthetic datasets described above. 
However, the prompt constraints permit program cascades of up to 50 programs, the sizes of $\alpha_k$ and $\beta_k$ can be up to 5, and each LLM is allotted 32 solution attempts (sampling budget) per PBE instance to better reflect the real-world complexity of low-resource SLI.
We do note a limitation of these assumptions, as for some language pairs, such as prototangkhulic-huishu, certain input words are shorter than 5 characters, and in principle an LLM could exploit this by constructing a single string rewrite rule that replaces the entire input word with the target output word.
Nevertheless, despite this potential for gaming, the results indicate that most LLMs struggle with the task, and even the best-performing gpt-oss-120B model achieves a Pass@1 of only $0.33$, despite a sampling budget of 32.

\subsection{Metrics Details}
\label{sec:appendix:metrics_details}
For the multi-step PBE task, since a given list of inputs ($\vec{\imath}$) could be transformed into the outputs ($\vec{o}$) by multiple program cascades $\vec{p}$, we utilize metrics based on functional equivalence.
This is in line with programming by example literature \cite{li2024is} and historical linguistics \cite{hoenigswald1960language}. 

Concretely, we execute the model-generated solution $\hat{\vec{p}}$ on the inputs, treating the input-output pairs as test cases.
We then compare the predicted outputs $\hat{\vec{o}}=\hat{\vec{p}}(\vec{\imath})$ with the ground-truth outputs $\vec{o}$ at two different levels of granularity:
(1) coarse-grained evaluation, corresponding to solve rate (\texttt{Pass@1}), and
(2) fine-grained evaluation, corresponding to normalized edit similarity (\texttt{Edit\_Sim}).

Both metrics perform element-wise comparisons on the strings in the output vectors:
\\
\textbf{Coarse-grained Metric} (\texttt{Pass@1}):
This metric corresponds to \texttt{Pass@1} as introduced in \cite{chen2021evaluating}.
\[
\mathrm{pass@1} =  \frac{1}{|\mathcal{D}|}\sum_{\vec{o},\vec{\imath} \in \mathcal{D}} \mathrm{1}_{\hat{\vec{p}}(\vec{\imath})=\vec{o}}
\]
Here $\mathrm{1}_{\hat{\vec{p}}(\vec{\imath})=\vec{o}}$ is an indicator variable which equals 1 when $\hat{\vec{p}}(\vec{\imath})=\vec{o}$ and 0 otherwise.
\\
\textbf{Fine-grained Metric} (\texttt{Edit\_Sim}):
This metric is identical to the Reward@1 metric used by \cite{naik2025programming}.
\[
\mathrm{Edit \_Sim} =  \frac{1}{|\mathcal{D}|} \sum_{\vec{o},\vec{\imath} \in \mathcal{D}} 1-\frac{\mathrm{dist}(\hat{\vec{p}}(\vec{\imath}),\vec{o})}{\mathrm{dist}(\vec{\imath},\vec{o})}
\]
Here $\mathrm{dist}$ denotes the total Levenshtein edit distance, summed across the corresponding strings in the predicted and ground-truth output vectors.

In addition to these performance metrics, we evaluate the proportion of programs generated by an LLM that are valid, meaning they satisfy all constraints specified in the prompt (see Section~\ref{sec:program_induction:prompting}).
We refer to this metric as the \texttt{Valid\_Rate}.

Finally, for the program reordering task, given a model-predicted permutation $\hat{\sigma}$, we apply it to the originally permuted program sequence $\sigma(\vec{p})$ and evaluate whether it correctly recovers the original outputs.
Specifically, we execute $\hat{\sigma}(\sigma(\vec{p}))$ on the inputs and check whether $\hat{\sigma}(\sigma(\vec{p}))(\vec{\imath})=\vec{o}$.

For some instances, we explicitly construct the problem such that there exists only a single unique solution, in which case the inverse permutation $\hat{\sigma}=\sigma^{-1}$ is the only correct answer.
The evaluation metric used is accuracy (Acc), defined analogously to \texttt{Pass@1}:

\phantomsection
\label{eq:perm_acc}
\[
\mathrm{Acc} =  \frac{1}{|\mathcal{D}|}\sum_{\vec{o},\vec{\imath} \in \mathcal{D}} \mathrm{1}_{\hat{\sigma}(\sigma(\vec{p}))(\vec{\imath})=\vec{o}}
\]
We additionally report unique accuracy, defined as the accuracy computed over the subset of instances with a single valid solution.
In the PBEBench-Lite-Perm dataset, this subset consists of 242 instances.

\begin{figure}[!tbh]
    \centering
    \includegraphics[width=0.5\textwidth]{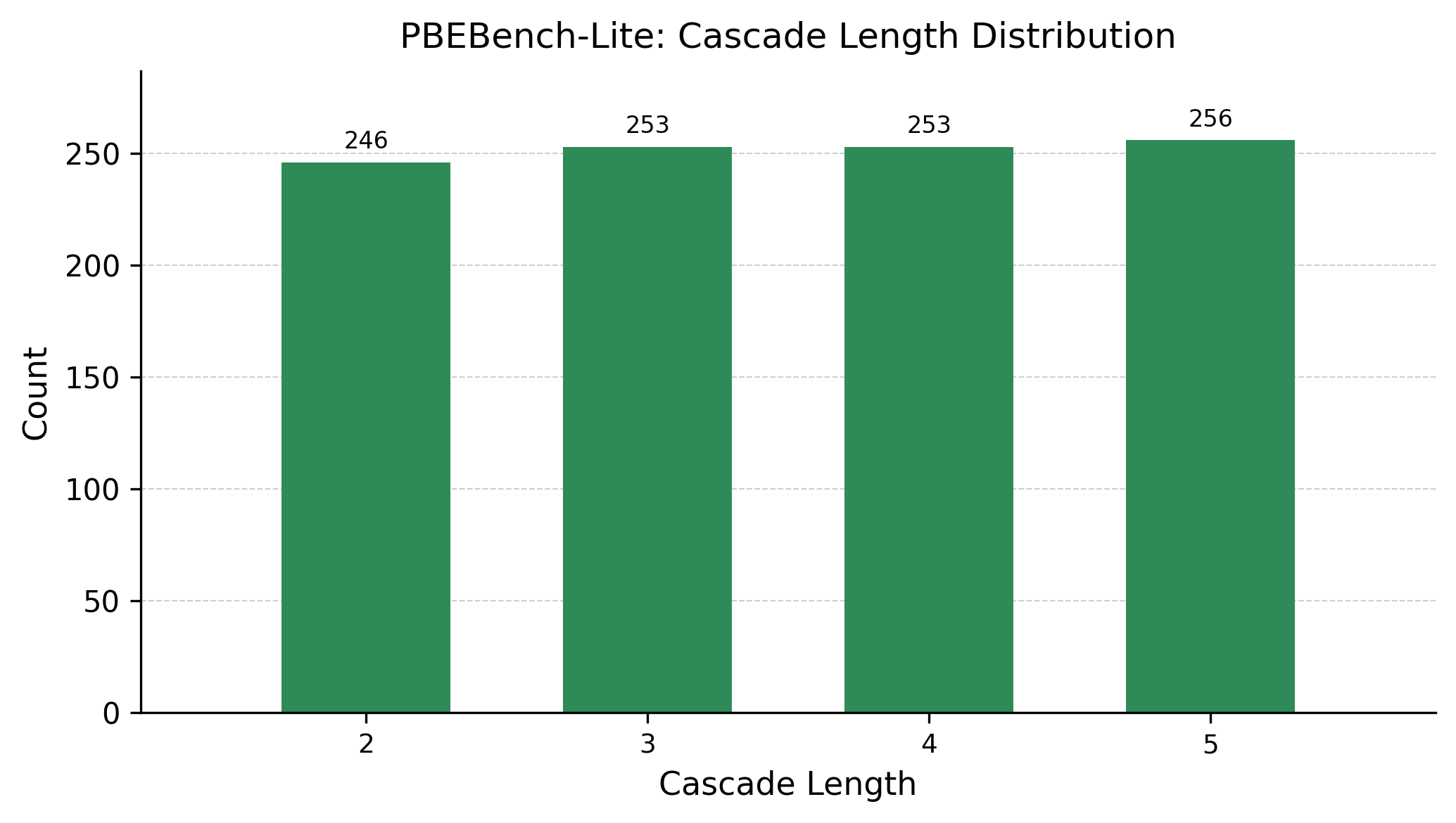}
    \caption{Cascade length distribution for PBEBench-Lite.}
\label{fig:pbebench_lite_cascade}
\end{figure}

\begin{figure}[!tbh]
    \centering
    \includegraphics[width=0.5\textwidth]{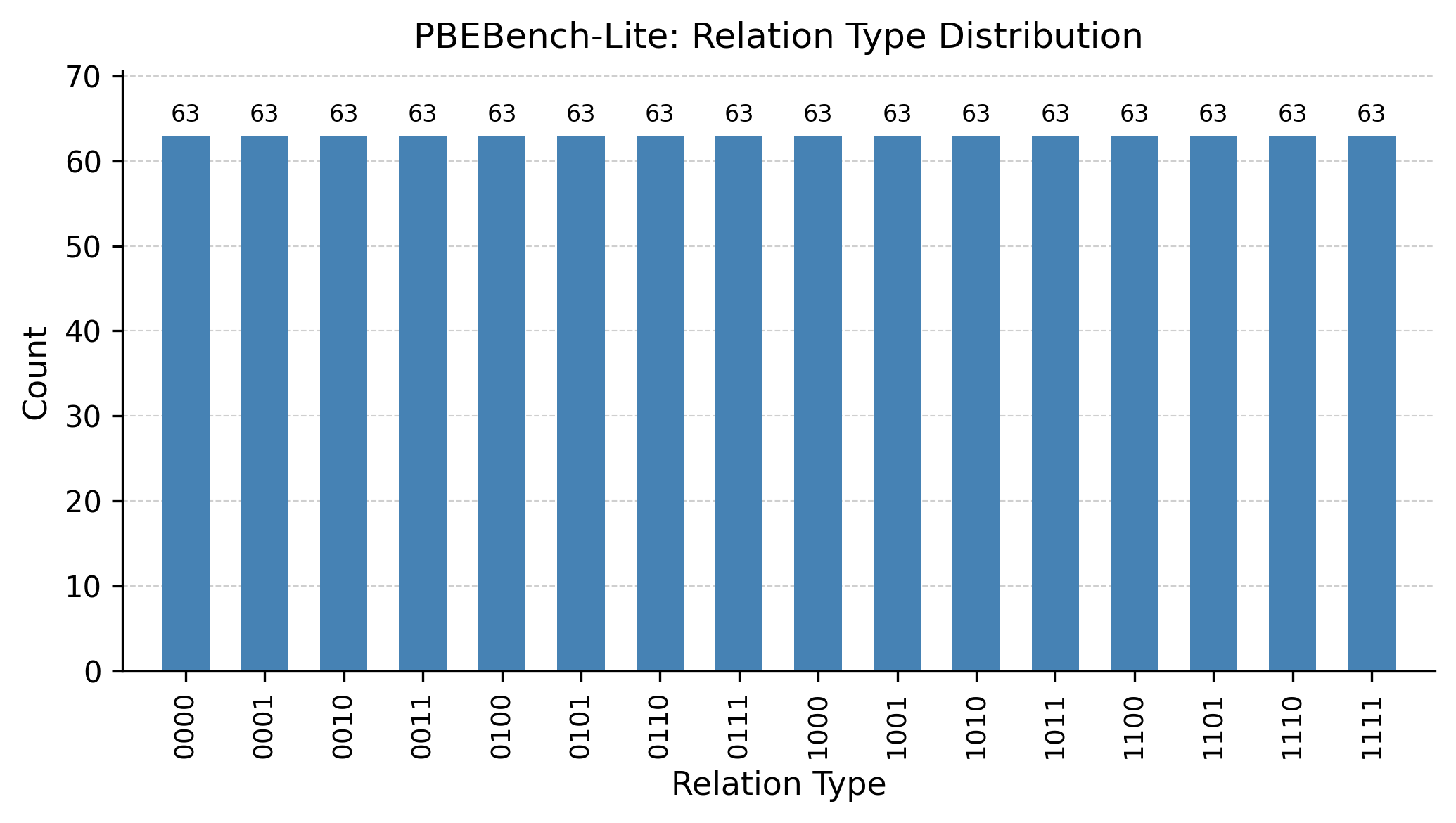}
    \caption{Relation type distribution for PBEBench-Lite.}
\label{fig:pbebench_lite_reln_type}
\end{figure}

\begin{figure}[!tbh]
    \centering
    \includegraphics[width=0.5\textwidth]{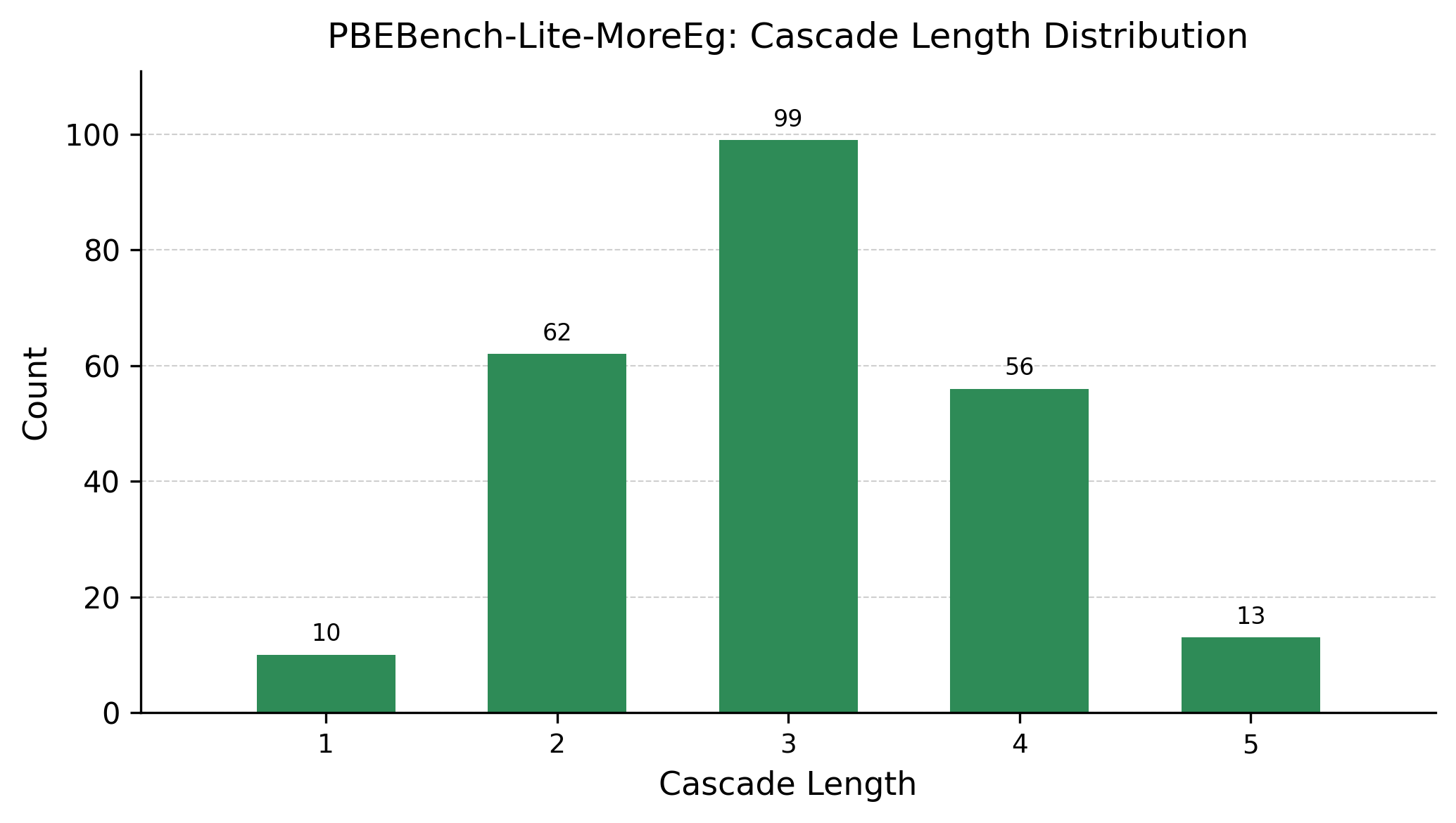}
    \caption{Cascade length distribution for PBEBench-Lite-MoreEg.}
\label{fig:pbebench_lite_more_eg_cascade}
\end{figure}

\begin{figure}[!tbh]
    \centering
    \includegraphics[width=0.5\textwidth]{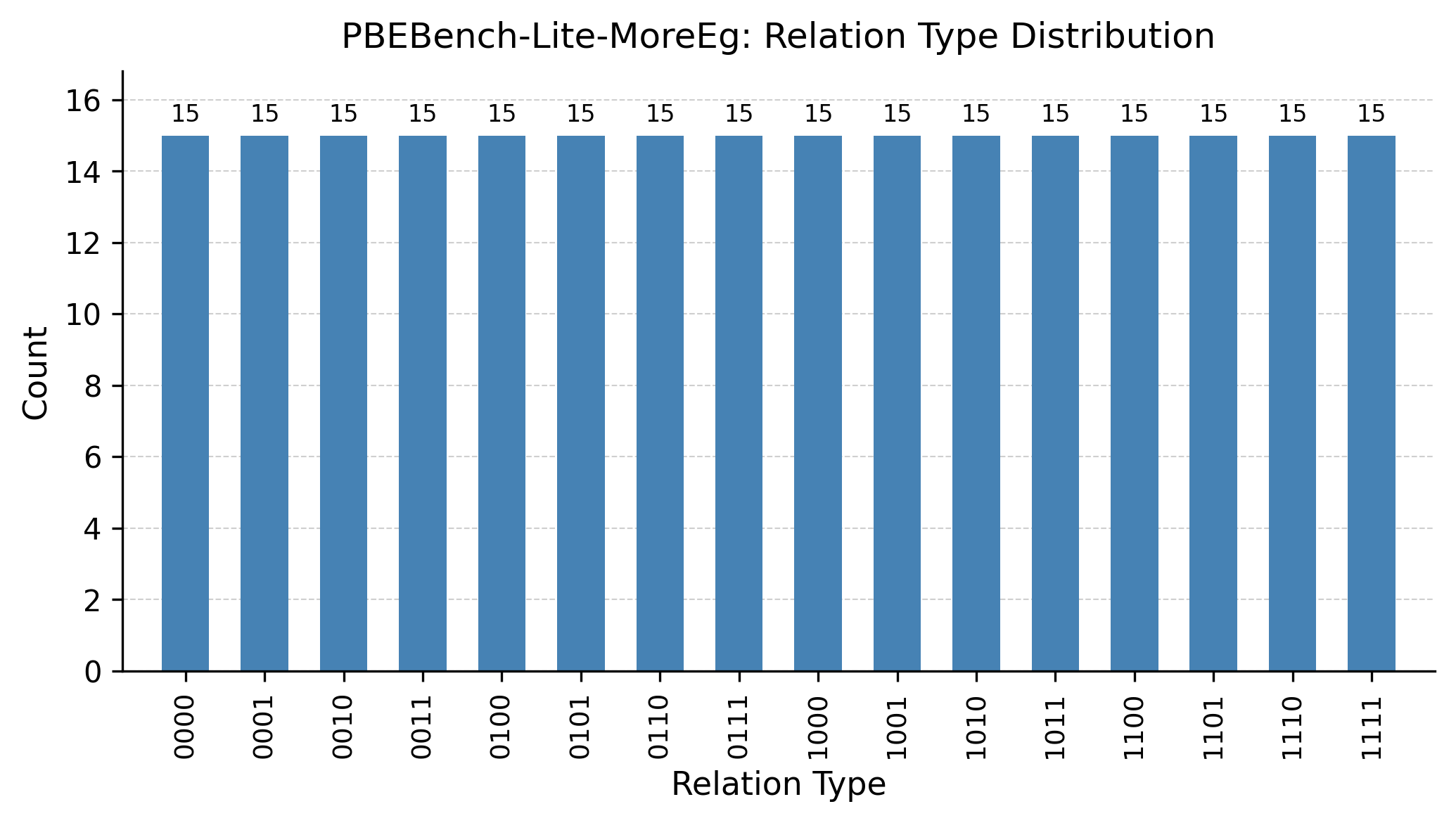}
    \caption{Relation type distribution for PBEBench-Lite-MoreEg.}
\label{fig:pbebench_lite_more_eg_reln_type}
\end{figure}

\begin{figure}[!tbh]
    \centering
    \includegraphics[width=0.5\textwidth]{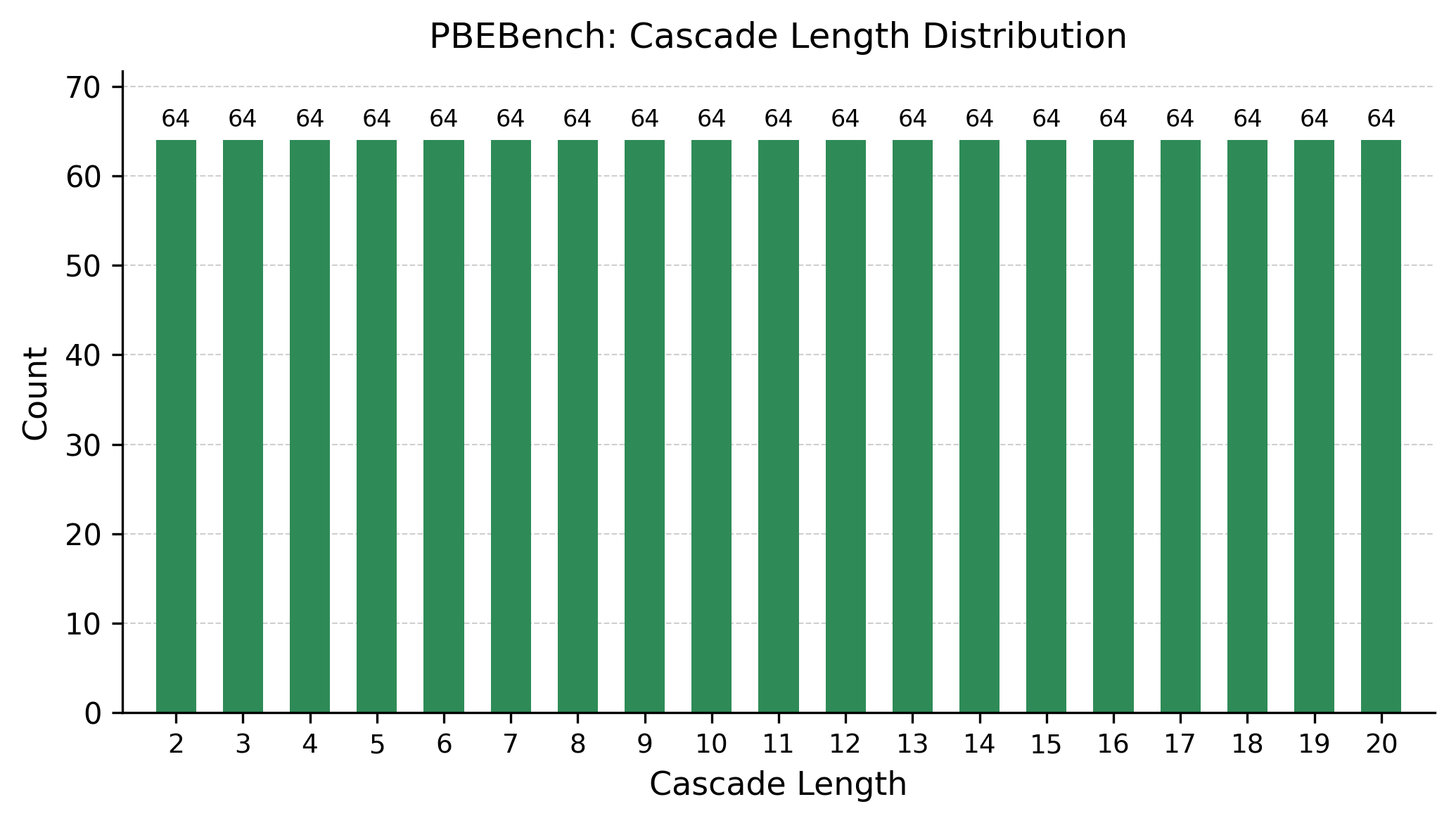}
    \caption{Cascade length distribution for PBEBench.}
\label{fig:pbebench_cascade}
\end{figure}

\begin{figure}[!tbh]
    \centering
    \includegraphics[width=0.5\textwidth]{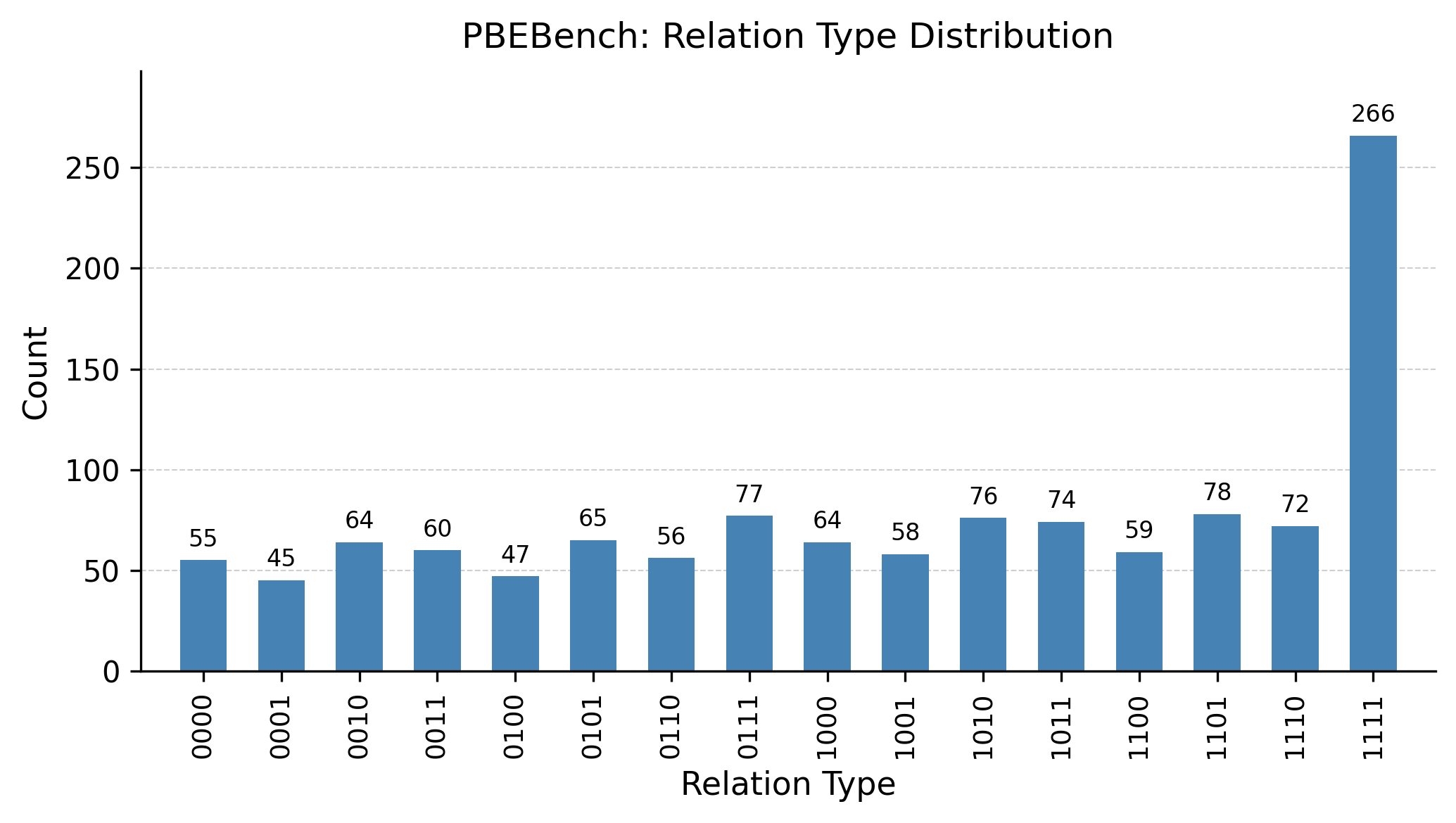}
    \caption{Relation type distribution for PBEBench.}
\label{fig:pbebench_reln_type}
\end{figure}

\begin{figure}[!tbh]
    \centering
    \includegraphics[width=0.5\textwidth]{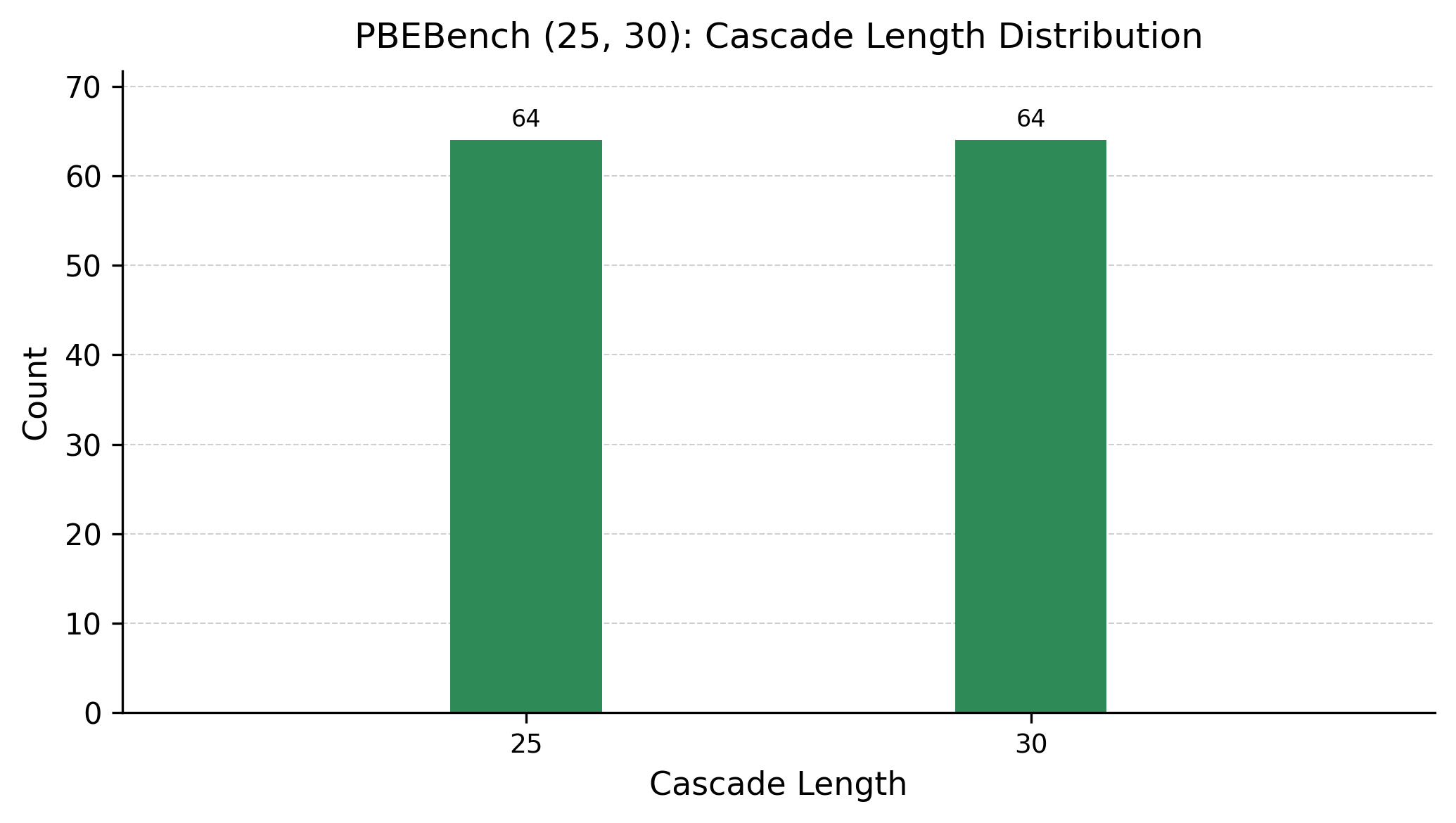}
    \caption{Cascade length distribution for PBEBench (25, 30).}
\label{fig:pbebench_25_30_cascade}
\end{figure}

\begin{figure}[!tbh]
    \centering
    \includegraphics[width=0.5\textwidth]{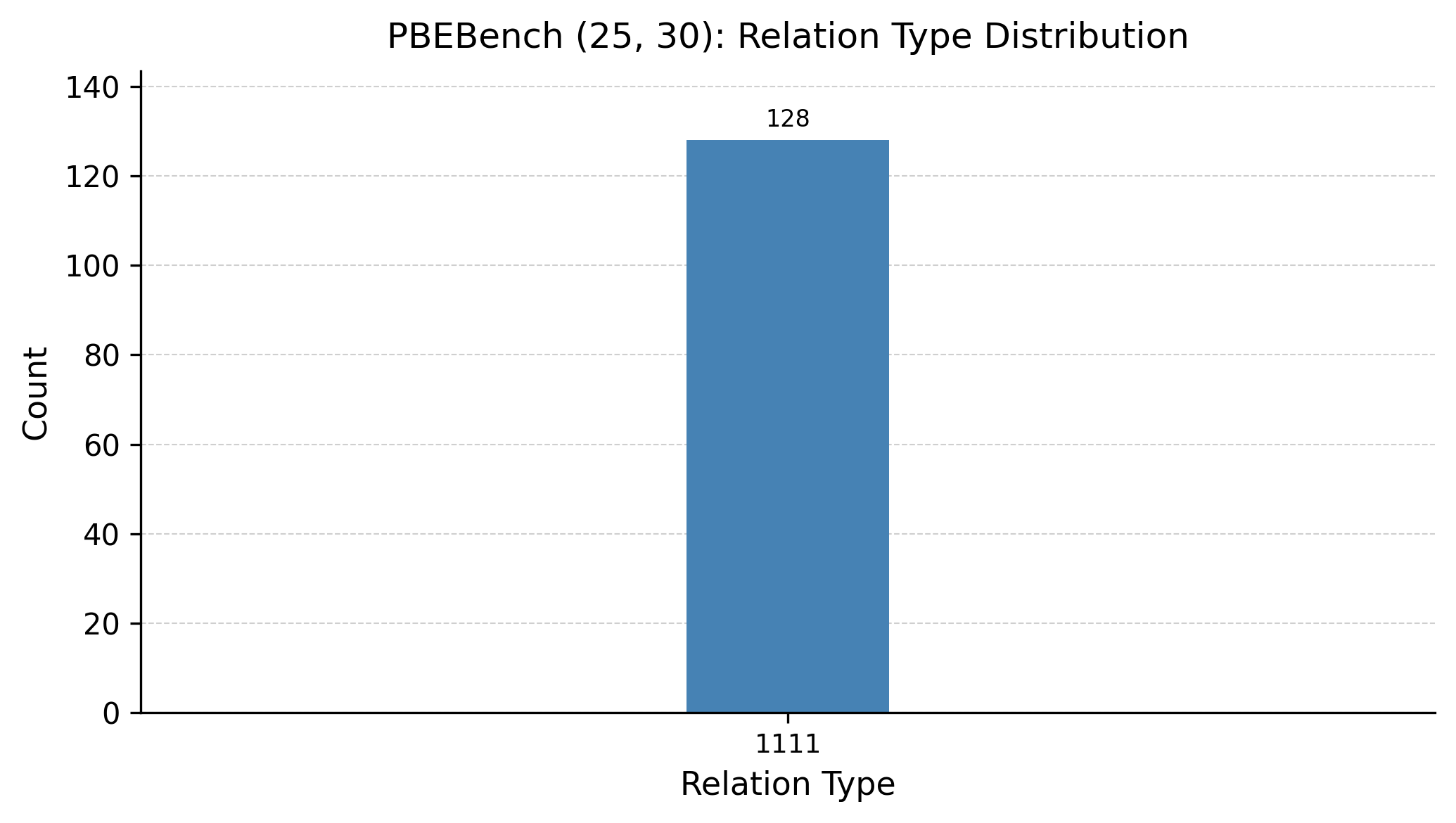}
    \caption{Relation type distribution for PBEBench (25, 30).}
\label{fig:pbebench_25_30_reln_type}
\end{figure}

\section{Method Details}
\label{sec:appendix:method_details}
This section provides additional details on the problem proposer and problem solver to support reproducibility. Specifically, it covers:
\begin{itemize}[itemsep=1pt, leftmargin=*, parsep=0pt, topsep=-2pt, partopsep=0pt]
    \item \textbf{Problem proposer:} the algorithm used for data generation.
    \item \textbf{Problem solver:} the prompt template for inference, the models used in our work (including details, licensing, and snapshots for closed-source models, as well as costs of running them), and the K-Fold analysis employed to efficiently simulate different sampling budgets with expensive closed-source models like GPT-5 while reducing variance.
\end{itemize}

\subsection{Algorithm Pseudocode}
This section formally expresses the pseudocode/logic behind the data generation algorithm proposed in our work, breaking it down into the rejection sampling subroutine (Algorithm~\ref{alg:balanced_sampling}) and the relation type classification subroutine (Algorithm~\ref{alg:classify_bfcc}), respectively. 

\begin{algorithm}[t]
\caption{Rejection Sampling for BFCC Dataset Generation}
\label{alg:balanced_sampling}
\begin{algorithmic}[1]
\Require Target size $n$; cascade bounds $[\ell_{\min}, \ell_{\max}]$; patience $\tau$
\Require Input sampler $p_\mathcal{I}$: generates $k$ random strings from vocabulary $\Sigma$
\Require Cascade sampler $p_\mathcal{P}(\cdot \mid X, \ell)$: generates $\ell$ programs where each replace$(a,b)$ has $a$ sampled from substrings of $X$ and $b$ sampled from $\Sigma$
\Ensure Dataset $\mathcal{D}$ balanced over 16 BFCC categories
\State Initialize quotas $q_c \gets \lfloor n/16 \rfloor$ for all $c \in \{0,1\}^4$
\State Initialize $\mathcal{D} \gets \emptyset$, seen signatures $\mathcal{S} \gets \emptyset$, steps $t \gets 0$
\While{$|\mathcal{D}| < n$}
    \State $t \gets t + 1$
    \State Sample length $\ell \sim \text{Uniform}\{\ell_{\min}, \ldots, \ell_{\max}\}$
    \State Sample inputs $X = \{x_1, \ldots, x_k\} \sim p_\mathcal{I}$
    \State Sample cascade $\pi = (f_1, \ldots, f_\ell) \sim p_\mathcal{P}(\cdot \mid X, \ell)$
    \State $\hat{\pi} \gets []$; $Y \gets X$
    \For{each program $f \in \pi$}
        \If{$f$ changes at least one element in $Y$}
            \State $\hat{\pi} \gets \hat{\pi} \cdot f$; $Y \gets f(Y)$
        \EndIf
    \EndFor
    \If{$|\hat{\pi}| < \ell_{\min}$ or $Y = X$} 
        \State \textbf{continue} \Comment{Reject: insufficient transformation}
    \EndIf
    \State $\sigma \gets (X, Y, \hat{\pi}, |\hat{\pi}|)$ 
    \If{$\sigma \in \mathcal{S}$} 
        \State \textbf{continue} \Comment{Reject: duplicate}
    \EndIf
    \State $c \gets \text{ClassifyBFCC}(\hat{\pi})$
    \If{$t < \tau$ and $q_c > 0$} \Comment{Before patience: enforce quotas}
        \State $\mathcal{D} \gets \mathcal{D} \cup \{(X, \hat{\pi}, Y, c)\}$
        \State $q_c \gets q_c - 1$; $\mathcal{S} \gets \mathcal{S} \cup \{\sigma\}$
    \ElsIf{$t \geq \tau$} \Comment{After patience: accept any valid instance}
        \State $\mathcal{D} \gets \mathcal{D} \cup \{(X, \hat{\pi}, Y, c)\}$
        \State $\mathcal{S} \gets \mathcal{S} \cup \{\sigma\}$
    \EndIf
\EndWhile
\State \Return $\mathcal{D}$
\end{algorithmic}
\end{algorithm}

\subsection{Program Extraction}
\label{sec:appendix:method_details:program_extraction}
To extract the program cascade from LLM predictions, we apply the following rules.
If the LLM fails to produce a valid, extractable Python code block, or produces a null response due to unterminated chain of thought (as observed in gpt-oss models), we mark the corresponding case as a null response and replace the prediction with an identity program like \texttt{replace("x", "x")} for evaluation.
For cases where a valid response is produced, we account for the fact that we evaluate a wide range of LLMs, including reasoning models that may generate intermediate programs and subsequently refine or improve them through reflection.
Accordingly, we evaluate both the first and the last Python code blocks produced by the model.
For the remaining valid cases, if a predicted program cascade $\hat{\vec{p}}$ contains more than $L_{max}$ programs, we retain only the first $L_{max}$ programs for evaluation.
If the predicted program violates any other constraint, it is replaced by an identity program $p^I$, which leaves the inputs unchanged.
Given the predicted cascade $\hat{\vec{p}}$, we compute the predicted outputs as $\hat{\vec{o}} = \hat{\vec{p}}(\vec{\imath})$.
Table~\ref{tab:pbebench_lite_first_and_last_code_block} reports performance on PBEBench-Lite for both the first and last code blocks.
We observe that, for most models, the last code block yields superior performance.
Therefore, in Table~\ref{tab:model_performance} in the main paper, as well as in any table where it is not explicitly specified otherwise, we report results corresponding to the last code block by default.

\begin{table*}[t]
\centering
\smallskip
\resizebox{\textwidth}{!}{
\begin{tabular}{@{}lrrrrrr@{}}
\toprule
\multirow{2}{*}{\textbf{Model}} & \multicolumn{3}{r}{\textbf{First Code Block}} & \multicolumn{3}{r}{\textbf{Last Code Block}} \\
\cmidrule(l){2-7}
 & \textbf{Pass@1} & \textbf{Edit Sim} & \textbf{Valid Rate} & \textbf{Pass@1} & \textbf{Edit Sim} & \textbf{Valid Rate} \\
\midrule
Codestral-22B & 0.0109 & -0.0130 & 0.8280 & 0.0109 & -0.0107 & 0.8254 \\
Qwen2.5-32B-Instruct & 0.0298 & 0.1232 & 0.8265 & 0.0298 & 0.1252 & 0.8288 \\
Qwen2.5Coder-32B-Instruct & 0.0397 & 0.1836 & 0.6883 & 0.0407 & 0.1884 & 0.6890 \\
Qwen3-32B & 0.0179 & 0.0964 & 0.7645 & 0.0179 & 0.0964 & 0.7645 \\
Qwen3-Coder-30B-A3B-Instruct \moeemoji & 0.0377 & 0.0857 & 0.8048 & 0.0377 & 0.0864 & 0.8134 \\
DeepSeek-R1-Distill-Qwen-32B \reasoningemoji & 0.2242 & 0.3486 & 0.8709 & 0.2242 & 0.3486 & 0.8709 \\
Qwen3-30B-A3B \reasoningemoji \moeemoji & 0.2887 & 0.3360 & \textbf{0.9905} & 0.2887 & 0.3360 & \textbf{0.9905} \\
QwQ-32B \reasoningemoji & 0.3591 & 0.4072 & 0.9500 & 0.3601 & 0.4092 & 0.9493 \\
Qwen3-32B (with CoT) \reasoningemoji & 0.3938 & 0.4803 & 0.9182 & 0.4187 & 0.5026 & 0.9676 \\
gpt-oss-20b \reasoningemoji \moeemoji & 0.4058 & 0.4619 & 0.9900 & 0.4058 & 0.4619 & 0.9900 \\
\textbf{gpt-oss-120b \reasoningemoji \moeemoji} & \textbf{0.6250} & \textbf{0.6985} & 0.9254 & \textbf{0.6250} & \textbf{0.6985} & 0.9254 \\
\midrule
Claude-3.5-Sonnet & 0.1845 & 0.4430 & 0.8208 & 0.1845 & 0.4434 & 0.8209 \\
Claude-3.7-Sonnet & 0.2212 & 0.4825 & 0.8409 & 0.2321 & 0.4996 & 0.8409 \\
Claude-4 Sonnet & 0.2956 & 0.5832 & 0.7720 & 0.2966 & 0.5870 & 0.7719 \\
Claude-3.7-Sonnet (Thinking) \reasoningemoji & 0.3343 & 0.5953 & 0.8127 & 0.3661 & 0.6154 & 0.8192 \\
Claude-4 Sonnet (Thinking) \reasoningemoji & 0.3571 & 0.6085 & 0.7788 & 0.3581 & 0.6079 & 0.7821 \\
Claude-4 Opus (Thinking)* \reasoningemoji & 0.5389 & 0.7497 & 0.8577 & 0.5389 & 0.7521 & 0.8561 \\
o3-mini \reasoningemoji & 0.5278 & 0.5954 & 0.9179 & 0.5278 & 0.5954 & 0.9179 \\
o4-mini \reasoningemoji & 0.6329 & 0.6907 & 0.9153 & 0.6329 & 0.6907 & 0.9153 \\
Gemini 2.5 Flash \reasoningemoji & 0.5833 & 0.6533 & 0.7911 & 0.5863 & 0.6562 & 0.7901 \\
\textbf{GPT-5 \reasoningemoji} & \textbf{0.7242} & \textbf{0.7645} & 0.9286 & \textbf{0.7242} & \textbf{0.7645} & 0.9286 \\
\bottomrule
\end{tabular}
}
\caption{\textbf{PBEBench-Lite Performance:} We compute the \texttt{Pass@1} and \texttt{Edit\_Sim} as the coarse and fine-grained evaluation, respectively, for each model. \moeemoji\ indicates mixture-of-experts (or MoE) model. \reasoningemoji\ indicates a reasoning model. * indicates evaluated on 20\% of the dataset due to cost.}
\label{tab:pbebench_lite_first_and_last_code_block}
\end{table*}

\begin{algorithm}[t]
\caption{BFCC Classification}
\label{alg:classify_bfcc}
\begin{algorithmic}[1]
\Function{ClassifyBFCC}{cascade $\hat{\pi} = [\text{replace}(a_1, b_1), \ldots, \text{replace}(a_m, b_m)]$}
\State Initialize category vector $(c_F, c_B, c_{CF}, c_{CB}) \gets (0, 0, 0, 0)$
\For{each ordered pair $(i,j)$ where $i \neq j$}
    \If{program $i$ feeds program $j$} \Comment{$b_i$ creates instances of $a_j$}
        \State Set $c_F \gets 1$ if $i < j$; set $c_{CF} \gets 1$ if $i > j$
    \EndIf
    \If{program $i$ bleeds program $j$} \Comment{$a_i$ removes instances of $a_j$}
        \State Set $c_B \gets 1$ if $i < j$; set $c_{CB} \gets 1$ if $i > j$
    \EndIf
\EndFor
\State \Return $(c_F, c_B, c_{CF}, c_{CB})$
\EndFunction
\end{algorithmic}
\end{algorithm}

\begin{algorithm}[t]
\caption{FB-Swap Permutation Heuristic}
\label{alg:fb_swap_permutation}
\begin{algorithmic}[1]
\Statex \textbf{FB.} ``FB'' refers to \emph{feeding} ($F$) and \emph{bleeding} ($B$) relations between programs (Algorithm~\ref{alg:classify_bfcc}).
\Statex \textbf{Execution.} $\textsc{ApplyPrograms}(X,\pi)$ applies the programs in $\pi$ left-to-right to each string in $X$.

\Function{FindFBSwap}{$X,\hat{\pi},Y,E_{FB}$}
\Require Cascade $\hat{\pi}=[f_0,\ldots,f_{m-1}]$ and outputs $Y=\textsc{ApplyPrograms}(X,\hat{\pi})$
\Require FB edges $E_{FB}$ is a list of triples $(i,r,j)$ with $r\in\{F,B\}$
\Ensure A permutation $\rho$ and permuted cascade $\hat{\pi}^{\rho}$ such that
        $\textsc{ApplyPrograms}(X,\hat{\pi}^{\rho}) \neq Y$, or $\varnothing$ if none found

\State $m \gets |\hat{\pi}|$
\State $S \gets \emptyset$ \Comment{unordered index pairs already tried}

\For{each edge $(i,r,j)\in E_{FB}$}
    \State $p \gets (\min(i,j),\max(i,j))$
    \If{$p \in S$}
        \State \textbf{continue} \Comment{avoid trying the same swap twice}
    \EndIf
    \State $S \gets S \cup \{p\}$

    \State $\rho \gets [0,1,\ldots,m-1]$ \Comment{identity permutation over program positions}
    \State \textbf{swap} $\rho_i$ and $\rho_j$ \Comment{single transposition}
    \State $\hat{\pi}^{\rho} \gets [f_{\rho_0},\ldots,f_{\rho_{m-1}}]$
    \If{$\textsc{ApplyPrograms}(X,\hat{\pi}^{\rho}) \neq Y$}
        \State \Return $(\rho,\hat{\pi}^{\rho})$ \Comment{first FB-related swap that changes outputs}
    \EndIf
\EndFor

\State \Return $\varnothing$
\EndFunction
\end{algorithmic}
\end{algorithm}

\begin{figure*}[t]
\centering
\begin{tcolorbox}[
    width=\textwidth,
    colback=orange!5,
    colframe=orange!40!black,
    title={\textbf{Multi-Step PBE Prompt}},
    coltitle=white,
    colbacktitle=orange!75!black,
]
Follow the instructions below to solve the code completion task:

We will provide the input corpus and corresponding output corpus. Each element in the corpus is a string, and the output is transformed from the corresponding input using an ordered sequence of ``replace'' programs. You need to find the correctly constructed and ordered sequence of ``replace'' programs to transform the entire input corpus into the output corpus. Note that the programs can interact with each other in a way that reduces or increases the number of times they are applied on a given input based on where they are ordered in the sequence. This makes it very important to apply them in the correct order. 

The programs should be written using only the Python \texttt{replace} function. For example, for a program that replaces all occurrences of ``ab'' with ``bc'' it should be written as: \verb|replace('ab', 'bc')|

Here is an example of the full task:
\begin{verbatim}
### Inputs 
["abc", "ebc", "aba"]

### Outputs
["edc", "edc", "aba"]

### Program Sequence
```python
["replace('bc','dc')", "replace('ad','ed')"]
```
\end{verbatim}

While generating the program sequence, you need to abide by the following restrictions:
\begin{enumerate}
\item Each program in the sequence should have the form \verb|replace(A, B)|, where \texttt{A} and \texttt{B} are both strings.
\item Both argument strings \texttt{A} and \texttt{B} in \verb|replace(A, B)| should have length $\leq \{program\_length\}$. \texttt{A} must have length $\geq 1$, while \texttt{B} may be empty (i.e., \verb|""|).
\item The maximum number of programs in a sequence is \{program\_num\}.
\item You should only consider the Python \texttt{replace} function for specifying programs (each program is a Python \texttt{replace} function). You cannot use any other Python modules or functions.
\item Strictly follow the markdown style convention while presenting your final program sequence, and make sure to enclose it in the \verb|```python| markdown style code block.
\end{enumerate}

Now, please generate the sequence of programs corresponding to the following input corpus and output corpus:

\begin{verbatim}
### Inputs 
{inputs_list}

### Outputs
{outputs_list}

### Program Sequence
\end{verbatim}

\end{tcolorbox}
\end{figure*}

\clearpage 
\subsection{Prompt Templates}
\label{sec:appendix:method_details:prompt_template}
\subsubsection{Multi-step PBE Task}
\label{sec:appendix:method_details:multi_step_pbe_prompt_template}
We show the prompt template used for the multi-step PBE task above (\texttt{Multi-Step PBE Prompt}). This prompt includes the exact instructions and examples given to the LLMs.

\subsubsection{Program Reordering Task}
\label{sec:appendix:method_details:program_reordering_prompt_template}
We show the prompt template used for the program reodering task below (\texttt{Program Reordering Prompt}).
This prompt includes the exact instructions and examples given to all the LLMs for performing this task.

\subsection{Model Selection Details}
Table~\ref{tab:selected_models_details} details all the models chosen for our benchmark and their various attributes to ensure we evaluate a diverse and representative set of LLMs to evaluate which of them excel at inductive reasoning.

\begin{table*}[!tbh]
\centering
\begin{tabular}{@{}llllll@{}}
\toprule
Model Name & Reasoning & Citation & Parameters & MoE & Source \\ \midrule
QwQ-32B & Yes & \citet{qwq32b} & 32B & No & Closed \\
DeepSeek-R1-Distill-Qwen-32B & Yes & \citet{deepseekai2025deepseekr1incentivizingreasoningcapability} & 32B & No & Closed \\
o3-mini & Yes & \citet{openai2024_gpto3mini} & – & No & Closed \\
o4-mini & Yes & \citet{openai2025_gpto4mini} & – & No & Closed \\
Qwen3-30B-A3B (Thinking) & Yes & \citet{qwen3} & 30B & Yes & Open \\
Qwen3-32B & Yes & \citet{qwen3} & 32B & No & Open \\
Gemini 2.5 Flash & Yes & \citet{comanici2025gemini} & – & No & Closed \\
Claude-3.7-Sonnet & Yes & \citet{anthropic2025_claude37_sonnet} & – & No & Closed \\
Claude-4 Sonnet (Thinking) & Yes & \citet{anthropic2025_claude_sonnet4} & – & No & Closed \\
Claude-4 Opus (Thinking) & Yes & \citet{anthropic_claude_opus4} & – & No & Closed \\
gpt-oss-20b & Yes & \citet{openai_gpt_oss} & 20B & No & Open \\
gpt-oss-120b & Yes & \citet{openai_gpt_oss} & 120B & No & Open \\
GPT-5 (Thinking) & Yes & \citet{openai_gpt5} & – & No & Closed \\
Qwen2.5-32B-Instruct & No & \citet{qwen2.5} & 32B & No & Open \\
Claude-3.5-Sonnet & No & \citet{anthropic2024_claude35_sonnet} & – & No & Closed \\
GPT-5 (Non-Thinking) & No & \citet{openai_gpt5} & – & No & Closed \\
Codestral-22B & No & \citet{mistral2024codestral22b} & 22B & No & Open \\
Qwen2.5Coder-32B-Instruct & No & \citet{qwen2.5} & 32B & No & Open \\
Qwen3-Coder-30B-A3B-Instruct & No & \citet{qwen3} & 30B & Yes & Open \\ \bottomrule
\end{tabular}
\caption{\textbf{Model Selection:} This table details the characteristics of the models benchmarked on PBEBench-Lite. The columns discuss cover the model name, reasoning ability, citation, parameter count, architecture style (MoE vs Dense), and open/closed soruce nature of the chosen models to showcase the diversity of the evaluated models.}
\label{tab:selected_models_details}
\end{table*}

\subsection{Licenses for Evaluated Models}
\label{sec:appendix:model_licenses}
We list the licenses used for each evaluated open and closed source models in Table~\ref{tab:model_licenses}.

\begin{table*}[!tbh]
\centering
\begin{tabular}{@{}l r@{}}
\toprule
\textbf{Model} & \textbf{License} \\
\midrule
Codestral-22B & Mistral Non-Production License (MNPL) \\
Qwen2.5-32B-Instruct & Apache 2.0 \\
Qwen2.5Coder-32B-Instruct & Apache 2.0 \\
Qwen3-32B & Apache 2.0 \\
Qwen3-Coder-30B-A3B-Instruct & Apache 2.0 \\
QwQ-32B & Apache 2.0 \\
Qwen3-32B & Apache 2.0 \\
Qwen3-30B-A3B & Apache 2.0 \\
Qwen3-Coder-30B-A3B-Instruct & Apache 2.0 \\
DeepSeek-R1-Distill-Qwen-32B & MIT \\
o3-mini & API (OpenAI EULA) \\
o4-mini & API (OpenAI EULA) \\
GPT-5 & API (OpenAI EULA) \\
gpt-oss-20b & Apache 2.0 \\
gpt-oss-120b & Apache 2.0 \\
Gemini 2.5 Flash Preview 04-17 & API (Google EULA) \\
Claude-3.5-Sonnet & API (Anthropic EULA) \\
Claude-3.7-Sonnet & API (Anthropic EULA) \\
Claude-4-Sonnet & API (Anthropic EULA) \\
Claude-4-Opus & API (Anthropic EULA) \\
\bottomrule
\end{tabular}
\caption{Licenses for open and closed source models.}
\label{tab:model_licenses}
\end{table*}

\subsection{Costs for Experiment Runs}
\label{sec:appendix:model_costs}
We document the costs of the expensive experiments carried out for closed source models in Table~\ref{tab:model_costs}.

\begin{table*}[!tbh]
\centering
\begin{tabular}{@{}l l r@{}}
\toprule
\textbf{Model} & \textbf{Experiment} & \textbf{Cost} \\
\midrule
Claude-4.1-Opus & PBE-Bench Lite Performance & \$40 (20\% of dataset) \\
GPT-5 & Cascade Length Experiment & \$190 \\
GPT-5 & Sampling Experiment & \$165 \\
GPT-5 & CoT Experiment & \$50 \\
GPT-5 & PBE-Bench Lite Performance & \$50 \\
Claude Sonnet Thinking 3.7 & PBE-Bench Lite Performance & \$30 \\
Claude Sonnet Thinking 4 & PBE-Bench Lite Performance & \$30 \\
o3-mini & PBE-Bench Lite Performance & \$65 \\
o4-mini & PBE-Bench Lite Performance & \$65 \\
\bottomrule
\end{tabular}
\caption{Documented costs for select experiment runs.}
\label{tab:model_costs}
\end{table*}

\subsection{Snapshots used for closed source models}
\label{sec:appendix:model_snapshots}
We document the Snapshots used for closed source models Table~\ref{tab:model_snapshots}.

\begin{table*}[!tbh]
\centering
\begin{tabular}{@{}l l r@{}}
\toprule
\textbf{Model} & \textbf{Snapshot} \\
\midrule
o3-mini & o3-mini-2025-01-31 \\
o4-mini & o4-mini-2025-04-16 \\
Claude 3.5 Sonnet & claude-3-5-sonnet-20241022 \\
Claude 3.7 Sonnet & claude-3-7-sonnet-20250219 \\
Claude 4 Sonnet & claude-sonnet-4-20250514 \\
Claude 4.1 Opus & claude-opus-4-1-20250805 \\
GPT-5 & gpt-5-2025-08-07 \\
Gemini 2.5 & gemini-2.5-flash \\
\bottomrule
\end{tabular}
\caption{Exact snapshots used for closed source models.}
\label{tab:model_snapshots}
\end{table*}

\subsection{K-Fold analysis for GPT-5} \label{sec:appendix:gpt_5_k_fold_analysis}
We observed a variance of up to 10\% in Pass@1 for GPT-5. To account for this variance, we report the aggregated score over all available samples when computing Pass@1 values. For instance, in the sampling experiment for GPT-5 shown in Fig.~\ref{fig:sampling_budget_scaling_gpt5}, we perform 8 independent runs and compute the average score over all possible k-run combinations, yielding scores for sampling budgets $1 \leq k \leq 8$.
\section{Experimental Details}
\label{sec:appendix:experimental_details}
In this section, we provide additional experimental details, including the computational environment and the inference and sampling parameters used for all the LLMs evaluated in our work. We also briefly discuss the strategies explored to achieve finer-grained control over the thinking budget of gpt-oss-120b for scaling experiments. These strategies were ultimately unsuccessful due to gpt-oss-120b's test-time behavior, leading us to instead study the effect of varying the maximum sequence length directly.



\subsection{Computational Enviornment}
We conduct experiments on a Linux server equipped with NVIDIA A100 80GB GPUs (Ampere architecture), CUDA 12.9, and driver version 575.51.03. Each job had access to 100 GB of CPU memory and up to 16 CPU cores. The GPU allocation varied with model size with gpt-oss-120b and most 32B models requiring 2 A100 GPUs. The experiments on PBEBench-Lite took multiple hours, while each cascade on PBEBench took nearly 8 hours for sampling budget of 32 and 16384 max sequence length, prompting use to parallely run multiple cascades across several GPUs.
We used vLLM for inference of all the open weight models and multi-threading for inference of closed source models like GPT-5 to speed up all inference experiments.

\subsection{Inference/Sampling Parameters}
\label{sec:appendix:experimental_details:sampling_parameters}

We show the sampling parameters used for all the models in Table~\ref{tab:sampling_parameters}. The max tokens are the total output tokens the model can generate (including thinking tokens), while the thinking budget(s) captures only the chain-of-thought or reasoning related tokens. The top-p is the cumulative probability cutoff used for nucleus sampling, while the temperature is for controlling the degree of randomness in the sampling.
We report the max tokens and thinking tokens wherever possible based on the providers (for some models you can only control the total tokens, while for some you can only control thinking tokens).
For some models like Gemini 2.5 Flash Preview, the model has a mode where it first reasons about how much thinking is required based on how complex it determines the problem to be.
We use this setting for the experiments in Table~\ref{tab:model_performance}.
However, we also do experiments comparing the effect of varying token budgets (2048, 4096, 8192) for QwQ and Gemini 2.5 Flash Preview, hence we highlight the default setting used for Table~\ref{tab:model_performance} for these models in bold.

\begin{table*}
\centering
\resizebox{\textwidth}{!}{
\begin{tabular}{@{}lrrrr@{}}
\toprule
\textbf{Model} & \textbf{Max Tokens} & \textbf{Top P} & \textbf{Temperature} & \textbf{Thinking Budget(s)} \\ \midrule
Codestral-22B & 2048 & 0.95 & 0.7 & - \\
Qwen2.5-32B-Instruct & 512 & 0.95 & 0.7 & - \\
Qwen2.5Coder-32B-Instruct & 512 & 0.95 & 0.7 & - \\
QwQ-32B & 8192 & 0.95 & 0.7 & - \\
Qwen/Qwen3-32B (with CoT) & 8192 & 0.95 & 0.7 & - \\
Qwen/Qwen3-32B & 8192 & 0.95 & 0.7 & - \\
Qwen3-30B-A3B & 8192 & 0.95 & 0.7 & - \\
DeepSeek-R1-Distill-Qwen-32B & 8192 & 0.95 & 0.7 & - \\
o3-mini & 8192 & - & - & reasoning\_effort="medium" \\
o4-mini & 8192 & - & - & reasoning\_effort="medium" \\
Gemini 2.5 Flash & dynamic & 0.95 & 0.7 & dynamic \\
Claude-3.5-Sonnet & 8192 & 0.95 & 0.7 & - \\
Claude-3.7-Sonnet & 10000 & 0.95 & 0.7 & - \\
Claude-3.7-Sonnet (Thinking) & 10000 & 0.95 & 1 (default) & 2048 \\
gpt-oss-20b & 8192 & 0.95 & 0.7 & - \\
gpt-oss-120b & 8192 & 0.95 & 0.7 & - \\
GPT-5 & 8192 & - & - & reasoning\_effort="medium" \\
Claude-4 sonnet & 8192 & 0.95 & 0.7 & - \\
Claude-4 sonnet (Thinking) & 8192 & 0.95 & 1 (default) & 2048 \\
Claude-4 opus (Thinking) (20\% of dataset) & 8192 & 0.95 & 1 (default) & 2048 \\
Qwen/Qwen3-Coder-30B-A3B-Instruct & 2048 & 0.95 & 0.7 & - \\
\bottomrule
\end{tabular}
}
\caption{Sampling parameters used for inference across all models runs. ``Max tokens'' refers to the total number of tokens (output + thinking tokens) for models that support it. "Top-p" controls nucleus sampling. "Temperature" sets the randomness of token selection. "Thinking budget" is the number of thinking tokens, applicable only to models that support this feature. GPT-5, o3-mini, and o4-mini support "reasoning\_effort" parameter, which is a qualitative measure of "Thinking Budget". These models also do not support temperature and top\_p parameters.}
\label{tab:sampling_parameters}
\end{table*}

\subsection{CoT Truncation experiment details}  \label{sec:cot_truncation_experiment}
For gpt-oss-120b, we attempt to reduce the model's \emph{thinking budget} and introduce it as a parameter independent of \emph{max tokens}. To achieve this, we run two inferences:

\begin{enumerate}
    \item In the first pass, we set \emph{max tokens} equal to the desired thinking budget. We then check whether the Chain-of-Thought Truncation token appears in the response.
    \item If it does not appear, we run a second inference, appending the following string as assistant context to the model's input:
\end{enumerate}

\begin{verbatim}
early_stop_instruction = "Considering the 
thinking token budget, I will not generate 
any more reasoning tokens, and provide the 
final answer `.\n"
\end{verbatim}

In the second generation, however, the model begins with 
``We need to produce final answer'' and then continues reasoning 
as usual, ignoring our instruction.

In variations, we appended {THINKING\_END\_TOKEN}, and {<FINAL\_OUTPUT\_START\_TOKEN>} 
to the \emph{early\_stop\_instruction}, but observed the same behavior. We also tried 
placing a modified version of \emph{early\_stop\_instruction} in 
the \emph{user} role instead of the assistant role, again without effect. Finally, we reduced \emph{max tokens} to 300 in the second generation. 
This did not prompt the model to produce a final answer either; instead, 
It significantly increased the rate of null outputs, rising from 11\% 
to 77\%. We therefore conclude that for gpt-oss-120b, it is not possible 
to enforce the truncation of the chain-of-thought budget independently of the 
total output token budget.
\section{Additional Results}
\label{sec:appendix:additional_results}
This section presents some additional detailed results over the PBEBench-Lite and PBEBench, and PBEBench (25, 30) snapshots, such as the detailed tables for the performance vs ground truth cascade length, scaling ablations, etc.
We also report results on related inductive reasoning benchmarks as well as on real SLI data.
It also contains the results of analyzing the effect of changing the number of examples (PBEBench-Lite-MoreEg snapshot).
It also contains the results of logistic regression analysis on factors affecting instance difficulty on PBEBench with gpt-oss-120b and factorial analysis on reasoning-capable models, as well as some representative models on PBEBench-Lite.
Finally, we also present results for two types of confusion matrices that visualize the distributional differences between cascade lengths and relation types of the ground truth and predicted cascades.


\begin{figure}[!tbh]
    \centering
    \begin{subfigure}{0.48\textwidth}
        \centering
        \includegraphics[width=\linewidth]{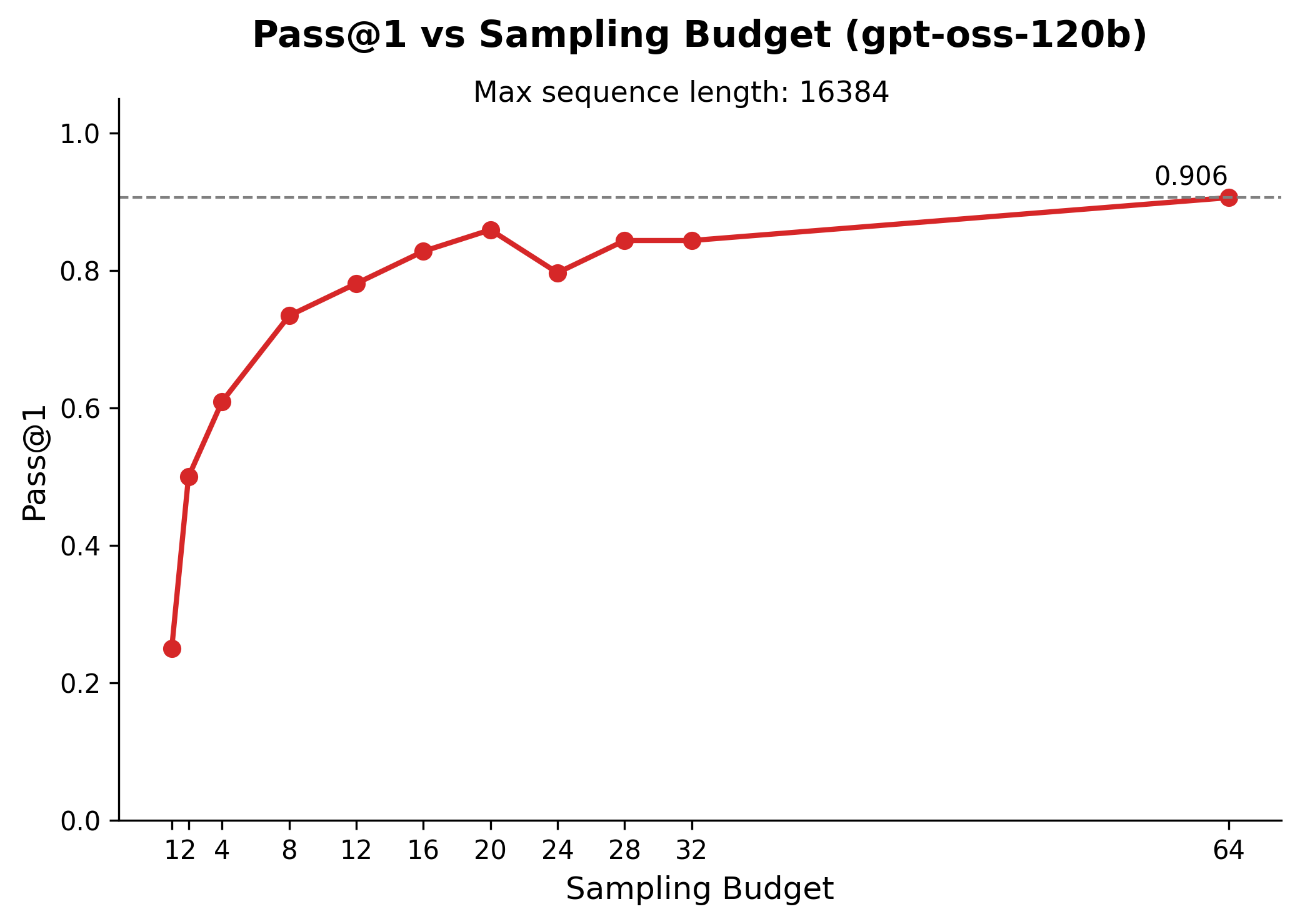}
        \caption{Pass@1 vs Sampling Budget (max seq len: 16384)}
        \label{fig:sampling_budget_scaling_gpt_oss_120b}
    \end{subfigure}
    \hfill
    \begin{subfigure}{0.48\textwidth}
        \centering
        \includegraphics[width=\linewidth]{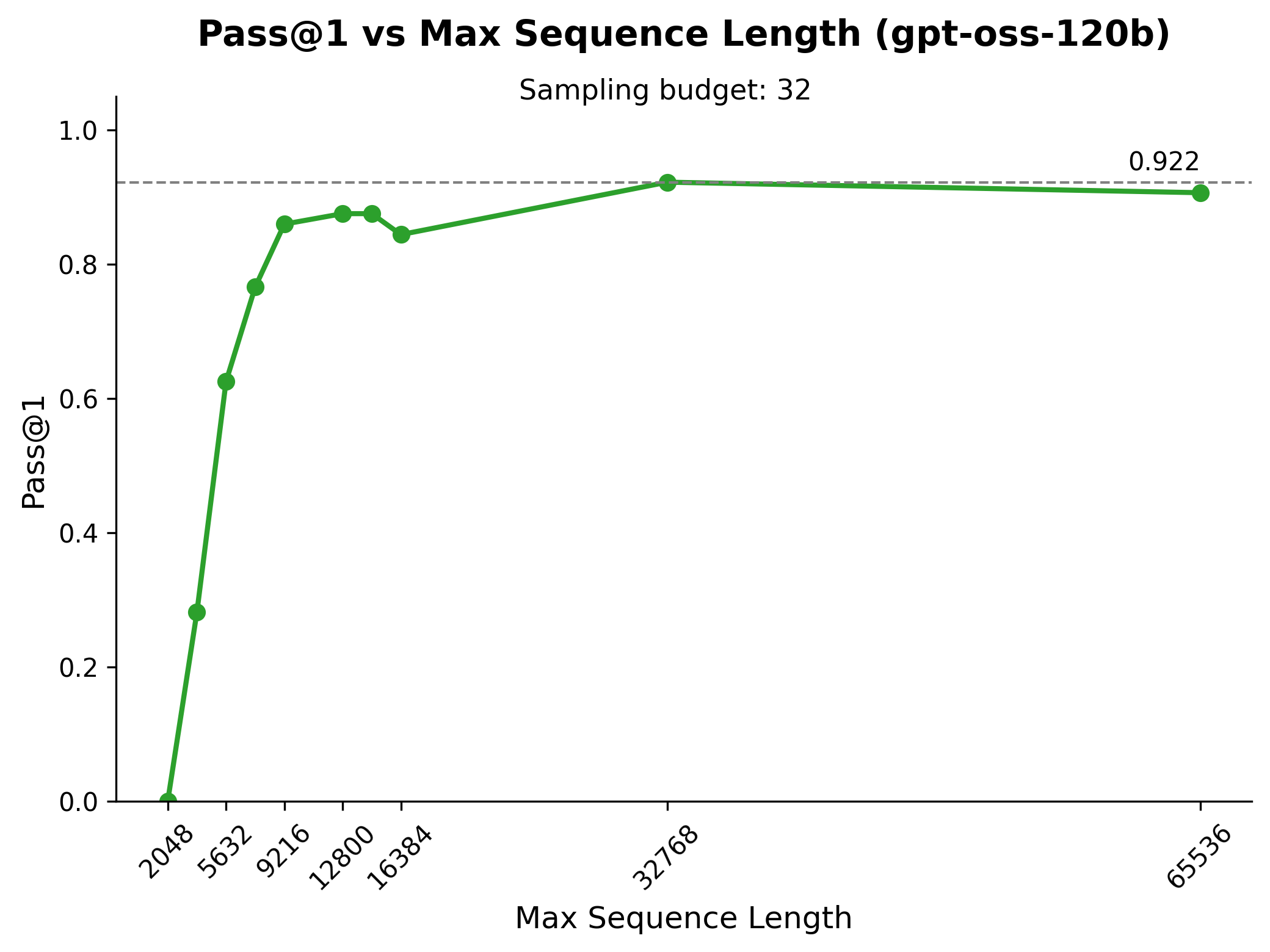}
        \caption{Pass@1 vs Max Seq Len (sampling budget: 32)}
        \label{fig:max_seq_len_scaling_gpt_oss_120b}
    \end{subfigure}
    \caption{\textbf{Effect of Scaling Strategies on PBEBench (gpt-oss-120b):} Comparison of successive sampling and increased thinking budgets (via max sequence length) on PBEBench instances with ground-truth cascades of length 10, the most complex balanced subset, unsolved yet nearly solvable under greater scaling, allowing a meaningful strategy comparison.}
\end{figure}

\begin{figure}[!tbh]
    \centering
    \begin{subfigure}{0.48\textwidth}
        \centering
        \includegraphics[width=\linewidth]{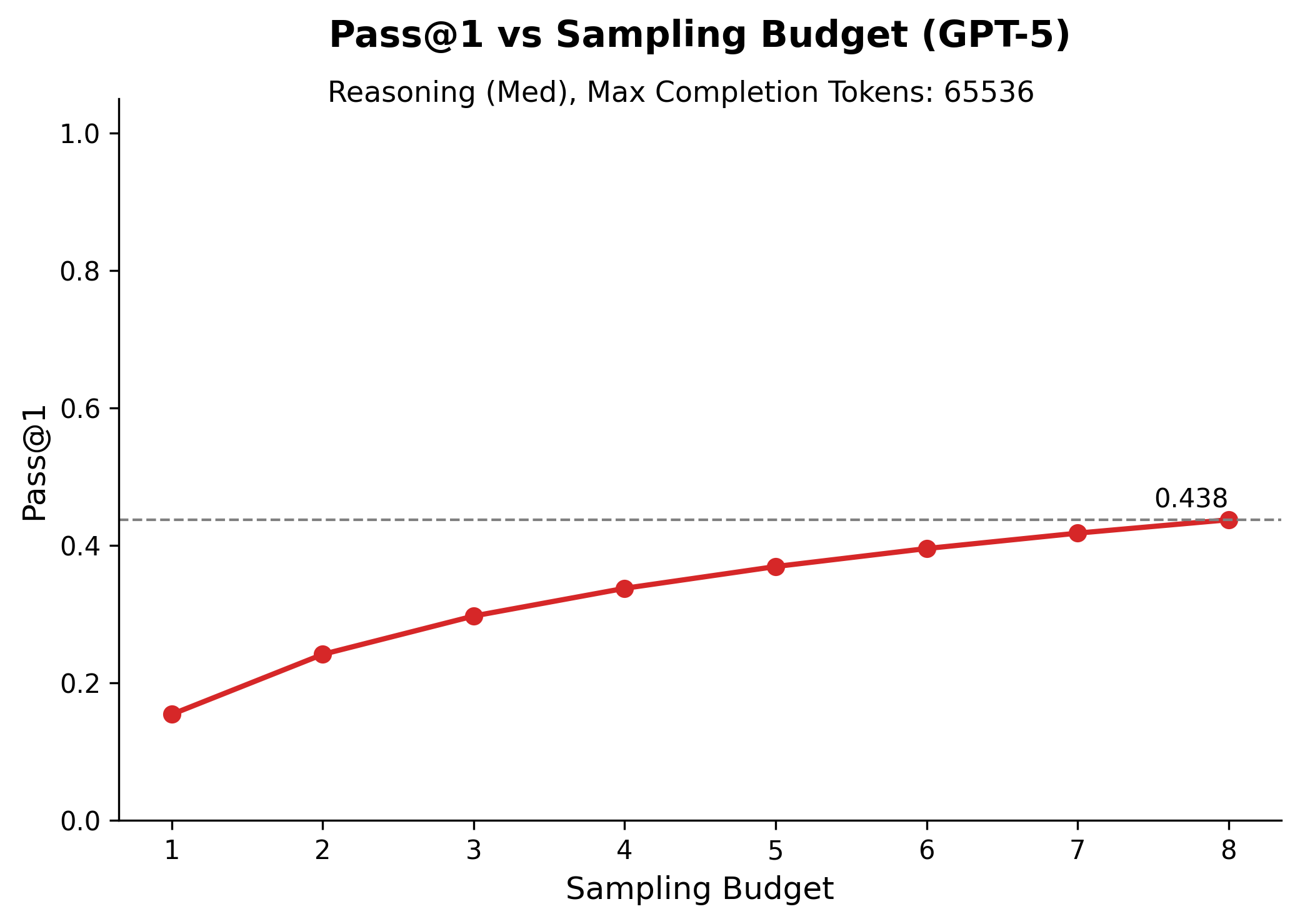}
        \caption{Pass@1 vs Sampling Budget (reasoning (med), max completion tokens: 65536)}
        \label{fig:sampling_budget_scaling_gpt5}
    \end{subfigure}
    \hfill
    \begin{subfigure}{0.48\textwidth}
        \centering
        \includegraphics[width=\linewidth]{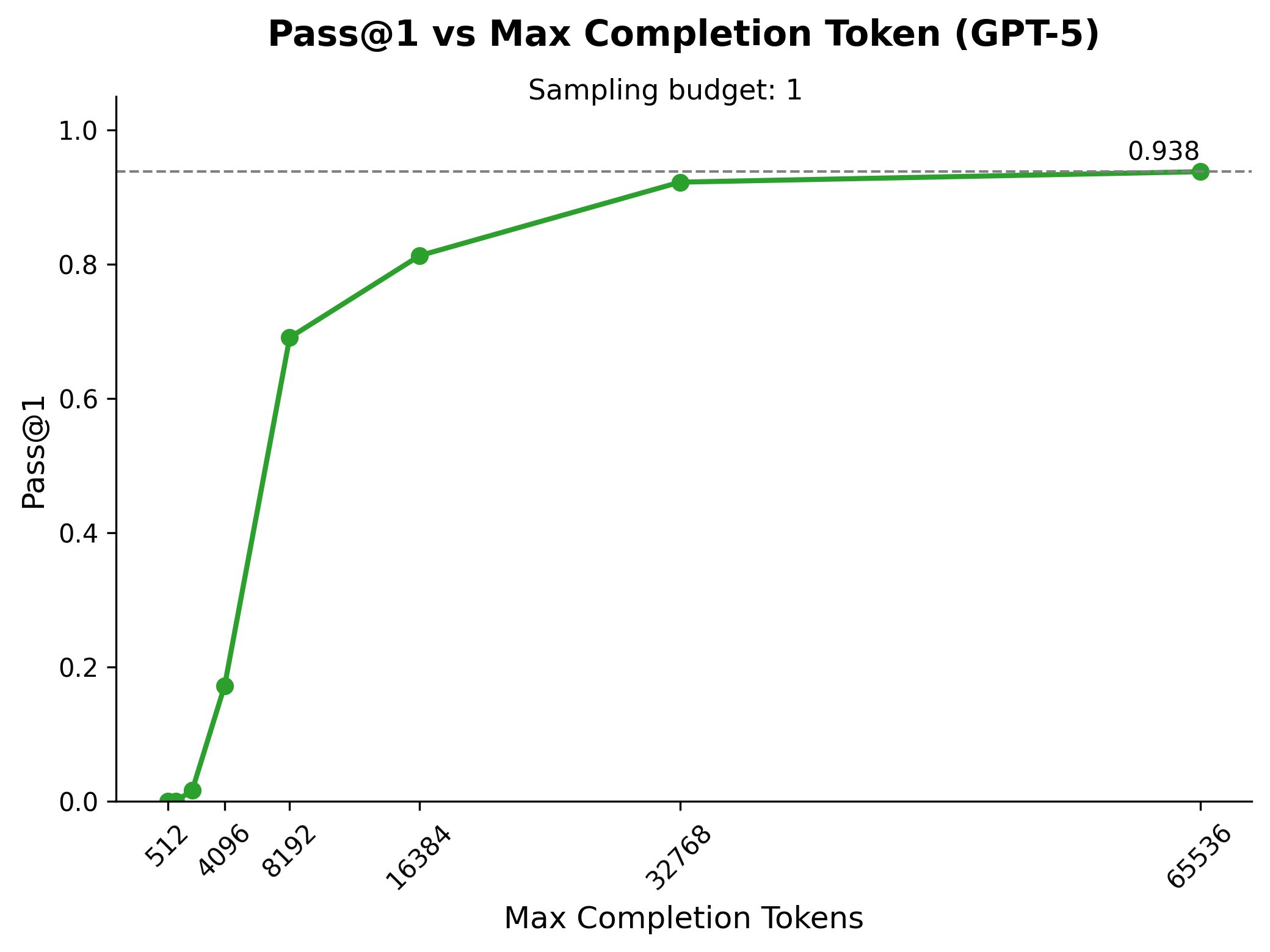}
        \caption{Pass@1 vs Max Completion Tokens (sampling budget: 1)}
        \label{fig:max_seq_len_scaling_gpt5}
    \end{subfigure}
    \caption{\textbf{Effect of Scaling Strategies on PBEBench (GPT-5):} Comparison of successive sampling and larger thinking budgets (via max completion tokens and reasoning effort) on PBEBench with ground-truth cascade lengths 20 and 10 respectively. Cascade 20 is used for the sampling budget experiment as we observe high \texttt{Pass@1} (0.938) for cascade 10 with just 1 sample, in the max completion tokens scaling experiment.}
\end{figure}

\subsection{PBEBench Performance Details}
\label{sec:appendix:pbebench_performance}
We evaluate gpt-oss-120b and GPT-5 on PBEBench, which includes cascades of length 2--20 and harder snapshots of length 25 and 30.
On cascades of length 2--20, gpt-oss-120b achieves \texttt{Pass@1} of 0.67, average \texttt{Edit\_Sim} of 0.95, and \texttt{Valid\_Rate} of 0.96, with per-cascade \texttt{Pass@1} reported in Table~\ref{tab:gpt_oss_120b_cascade_len_perf}.
GPT-5 substantially outperforms gpt-oss-120b, achieving \texttt{Pass@1} of 0.94 on cascades of length 10 with a single sample (vs.\ 0.84 with 32 samples for gpt-oss-120b) and 0.44 on cascades of length 20 with 4 samples (vs.\ 0.05 with 32 samples).
Since GPT-5 maintains meaningful performance at length 20, unlike gpt-oss-120b which collapses to 5\%, we further evaluate GPT-5 on cascades of length 20, 25, and 30, where it averages \texttt{Pass@1} of 0.175 (Table~\ref{tab:gpt5_cascade_len_perf_and_sampling_budget_scaling}).
Performance trends are shown in Fig.~\ref{fig:cascade_length_perf_gpt_oss_120b} and Fig.~\ref{fig:cascade_length_perf_gpt5}.

Both models follow the scaling strategies described in Section~\ref{sec:program_induction:prompting}.
gpt-oss-120b uses a sampling budget of 32 with maximum sequence length 16384, while GPT-5 uses a sampling budget of 4, a completion limit of 65536 tokens, and moderate reasoning effort.
Performance drops sharply with increasing cascade length: gpt-oss-120b falls to roughly 50\% and 5\% at lengths 15 and 20, while GPT-5 falls to 14\% and 5\% at lengths 25 and 30.
A logistic regression on gpt-oss-120b identifies cascade length as a strong negative predictor of \texttt{Pass@1}, with feeding, counter-feeding, and counter-bleeding predicting failure, and bleeding associated with success (Table~\ref{tab:logistic_regression_gpt_oss_120B}).

\subsection{Ablation Details}
\label{sec:appendix:ablations}
We evaluate gpt-oss-20b and gpt-oss-120b on the PBEBench-Lite-MoreEg snapshot, with results reported in Table~\ref{tab:pbebench_lite_more_eg_perf}.
Performance is similar to or slightly worse than PBEBench-Lite.
To analyze this effect, we compute the fraction of examples modified by each program: 1.56/5 (31\%) for PBEBench-Lite and 7.07/50 (14\%) for PBEBench-Lite-MoreEg.
Although more examples are modified in absolute terms, the relative signal per example is weaker.
However, the overall conclusion is that the effect of scaling input-output examples per PBE step from 5 to 50 (10x) is still relatively minor.
However, after a certain point, the LLM is expected to hit a hard wall, as experiments with real SLI data, where we attempted to have 20 examples per PBE step, led to models like gpt-oss-120b giving up and providing empty responses.

We further study scaling with respect to sampling budget and maximum sequence length.
For gpt-oss-120b, results are shown in Fig.~\ref{fig:sampling_budget_scaling_gpt_oss_120b} and Fig.~\ref{fig:max_seq_len_scaling_gpt_oss_120b}.
For GPT-5, corresponding results are shown in Fig.~\ref{fig:sampling_budget_scaling_gpt5} and Fig.~\ref{fig:max_seq_len_scaling_gpt5}, with k-fold averaging used to control variance (Section~\ref{sec:appendix:gpt_5_k_fold_analysis}).
Across all settings, we observe rapid initial gains followed by diminishing returns and eventual saturation.
Interestingly, for the more head-on comparison done for gpt-oss-120b, we see that scaling the max sequence length (or thinking budget) once you have a high enough sampling budget is more time efficient and cheaper.
This stems from the fact that a high sampling budget is more wasteful if the LLM is token-bound per attempt, as a lot of attempts fail to terminate the chain-of-thought and lead to unhelpful null responses.

\subsection{Other Inductive Reasoning Benchmarks}
\label{sec:appendix:other_inductive_reasoning_benchmarks}
We evaluate all open-source models on two additional inductive reasoning benchmarks: CLUTRR \cite{sinha2019clutrr} and SLR-Bench \cite{helff2025slr}.

CLUTRR is a well-established NLU benchmark that requires LLMs to read a short story, extract relationships between characters, and answer a kinship query about a target pair. 
The answer space consists of a small, fixed set of kinship relations, making the task a multi-class classification problem.
Successfully solving CLUTRR requires both accurate extraction of relational facts and correct inference over the underlying logical rules governing kinship.
The dataset is constructed semi-automatically from a kinship knowledge base via four steps: random kinship graph generation, target fact (relation) sampling, backward chaining, and natural language realization.
CLUTRR also explicitly probes robustness and generalization by injecting distracting facts into the stories, including supporting, irrelevant, disconnected, and noisy facts.
Following prior work, we concatenate all six test sets introduced in the original paper to form a combined evaluation set of 3{,}977 instances.
While the original CLUTRR work involves supervised training, we instead evaluate models in a zero-shot question answering setting using a fixed prompt (\texttt{CLUTRR Prompt}).
Results on the combined test set are reported in Table~\ref{tab:clutrr_results}.
We report standard kinship prediction accuracy, along with additional metrics for gpt-oss models, including non-null accuracy and the number of null predictions.
Null outputs arise when the chain of thought fails to terminate within the specified maximum sequence length, typically due to overthinking.
To further analyze this behavior, we vary both the reasoning effort and maximum sequence length for these models and report the corresponding performance.
Overall, LLMs perform poorly on CLUTRR, likely due to the presence of distractor facts and the absence of training or explicit graph-based representations for tracking relations.
In contrast to PBEBench, however, performance differences across models are relatively small, with a narrower gap between reasoning and non-reasoning models.
This trend is also reflected in the limited performance variation across different reasoning effort settings for gpt-oss models, and in the degradation observed at high reasoning effort, where increased null predictions from unterminated chains of thought reduce accuracy.
Finally, to ensure comparability with PBEBench-Lite, we restrict each model to a single attempt per instance.

\begin{figure*}[t]
\centering
\begin{tcolorbox}[
    width=\textwidth,
    colback=blue!5,
    colframe=blue!40!black,
    title={\textbf{CLUTRR Prompt}},
    coltitle=white,
    colbacktitle=blue!75!black,
]
You will be given a story containing characters whose names appear in square brackets, such as "[James]". After reading the story, you must identify the kinship relation between two characters.

\medskip

When answering:

\begin{itemize}
\item Respond with only the kinship term, with no explanation and no extra words.
\item Your answer must be one of the following options:
aunt, brother, daughter, daughter-in-law, father, father-in-law, granddaughter, grandfather, grandmother, grandson, mother, mother-in-law, nephew, niece, sister, son, son-in-law, uncle.
\item Do not use any text outside these options.
\item Give just the final relation.
\end{itemize}

Example Story: [Kristin] and her son [Justin] went to visit her mother [Carol] on a nice Sunday afternoon. They went out for a movie together and had a good time.

\medskip

Example Question: How is Carol related to Justin?

\medskip

Example Answer: grandmother

\medskip

Now do the same for the story and question below:

\medskip

Story: \verb|{story}|

\medskip

Question: \verb|{question}|

\medskip

Answer:
\end{tcolorbox}
\end{figure*}

\begin{table*}[!tbh]
\centering
\begin{tabular}{@{}lrrrr@{}}
\toprule
Model & Max Seq Len & \multicolumn{1}{l}{Acc} & \multicolumn{1}{l}{Non Null Acc} & \multicolumn{1}{l}{Nulls} \\ \midrule
Codestral-22B & 2048 & 3.72 & 3.72 & 0 \\
Qwen/Qwen2.5-32B-Instruct & 8192 & 34.98 & 34.98 & 0 \\
Qwen/Qwen2.5-Coder-32B-Instruct & 8192 & 26.88 & 26.88 & 0 \\
Qwen/QwQ-32B & 8192 & 52.07 & 52.07 & 0 \\
Qwen/Qwen3-32B (with CoT) & 8192 & 51.62 & 51.62 & 0 \\
Qwen/Qwen3-32B & 8192 & 22.53 & 22.53 & 0 \\
Qwen/Qwen3-30B-A3B & 8192 & 51.47 & 51.47 & 0 \\
DeepSeek-R1-Distill-Qwen-32B & 8192 & 47.67 & 47.67 & 0 \\
gpt-oss-20B (low) & 8192 & 42.59 & 42.59 & 0 \\
gpt-oss-20B (medium) & 8192 & 52.5 & 54.53 & 148 \\
\multirow{3}{*}{gpt-oss-20B (high)} & 8192 & 45.64 & 60.76 & 990 \\
 & 16384 & 50.64 & 60.04 & 621 \\
 & 32768 & 53.1 & 58.46 & 330 \\
gpt-oss-120B (low) & 8192 & 47.55 & 47.55 & 0 \\
gpt-oss-120B (medium) & 8192 & 57.53 & 57.53 & 0 \\
\multirow{2}{*}{gpt-oss-120B (high)} & 8192 & 51.72 & 63.43 & 631 \\
 & 16384 & 57.98 & 61.81 & 246 \\
Qwen/Qwen3-Coder-30B-A3B-Instruct & 8192 & 26.4 & 26.4 & 0 \\ \bottomrule
\end{tabular}
\caption{\textbf{CLUTRR Performance:} We compute the kinship relation type prediction accuracy (\texttt{Acc}) as well as non-null accuracy (\texttt{Non Null Acc}) which excludes null predictions and the number of null predictions (\texttt{Nulls}). We also try increased max sequence length and vary reasoning effort for the gpt-oss models.}
\label{tab:clutrr_results}
\end{table*}

\begin{table*}[!tbh]
\centering
\resizebox{\textwidth}{!}{
\begin{tabular}{@{}lrrrrrrrrrrrrr@{}}
\toprule
\multirow{2}{*}{Model} & \multicolumn{1}{l}{\multirow{2}{*}{Max Seq Len}} & \multicolumn{2}{r}{Gen Train234} & \multicolumn{2}{r}{Gen Train23} & \multicolumn{2}{r}{Rob Clean} & \multicolumn{2}{r}{Rob Disc} & \multicolumn{2}{r}{Rob Irr} & \multicolumn{2}{r}{Rob Sup} \\ \cmidrule(l){3-14} 
 & \multicolumn{1}{l}{} & Acc & NNAcc & Acc & NNAcc & Acc & NNAcc & Acc & NNAcc & Acc & NNAcc & Acc & NNAcc \\ \midrule
Codestral-22B & 2048 & 3.72 & 3.72 & 4.71 & 4.71 & 0.22 & 0.22 & 2.25 & 2.25 & 3.6 & 3.6 & 6.26 & 6.26 \\
Qwen/Qwen2.5-32B-Instruct & 8192 & 42.08 & 42.08 & 45.03 & 45.03 & 21.25 & 21.25 & 25.39 & 25.39 & 23.42 & 23.42 & 27.29 & 27.29 \\
Qwen/Qwen2.5-Coder-32B-Instruct & 8192 & 35.02 & 35.02 & 36.47 & 36.47 & 14.54 & 14.54 & 19.33 & 19.33 & 11.26 & 11.26 & 18.57 & 18.57 \\
Qwen/QwQ-32B & 8192 & 61.26 & 61.26 & 63.61 & 63.61 & 41.39 & 41.39 & 38.88 & 38.88 & 37.61 & 37.61 & 39.15 & 39.15 \\
Qwen/Qwen3-32B (with CoT) & 8192 & 60.78 & 60.78 & 63.26 & 63.26 & 37.81 & 37.81 & 41.12 & 41.12 & 37.39 & 37.39 & 38.7 & 38.7 \\
Qwen/Qwen3-32B & 8192 & 27.1 & 27.1 & 27.4 & 27.4 & 13.87 & 13.87 & 15.06 & 15.06 & 16.67 & 16.67 & 21.25 & 21.25 \\
Qwen/Qwen3-30B-A3B & 8192 & 61.07 & 61.07 & 61.95 & 61.95 & 37.58 & 37.58 & 41.12 & 41.12 & 37.84 & 37.84 & 39.82 & 39.82 \\
DeepSeek-R1-Distill-Qwen-32B & 8192 & 55.92 & 55.92 & 59.16 & 59.16 & 36.02 & 36.02 & 35.51 & 35.51 & 34.23 & 34.23 & 36.02 & 36.02 \\
gpt-oss-20B (low) & 8192 & 46.09 & 46.09 & 46.95 & 46.95 & 36.69 & 36.69 & 41.8 & 41.8 & 36.04 & 36.04 & 36.47 & 36.47 \\
gpt-oss-20B (medium) & 8192 & 61.55 & 65.68 & 62.04 & 66.51 & 40.27 & 40.27 & 44.49 & 44.59 & 39.19 & 39.46 & 40.27 & 40.36 \\
\multirow{3}{*}{gpt-oss-20B (high)} & 8192 & 48.28 & 77.73 & 54.36 & 80.08 & 40.27 & 45.11 & 41.12 & 49.06 & 35.59 & 40.51 & 36.91 & 41.67 \\
 & 16384 & 56.49 & 76.19 & 59.77 & 77.14 & 44.07 & 46.03 & 45.39 & 48.56 & 36.4 & 38.28 & 40 & 41.72 \\
 & 32768 & 61.35 & 71.6 & 63.53 & 73.17 & 44.07 & 44.57 & 46.52 & 47.37 & 39.19 & 40 & 40.94 & 41.59 \\
gpt-oss-120B (low) & 8192 & 50 & 50 & 52.53 & 52.53 & 46.31 & 46.31 & 45.17 & 45.17 & 39.19 & 39.19 & 40.94 & 40.94 \\
gpt-oss-120B (medium) & 8192 & 66.22 & 66.22 & 69.28 & 69.28 & 48.99 & 48.99 & 46.74 & 46.74 & 40.09 & 40.09 & 43.62 & 43.62 \\
\multirow{2}{*}{gpt-oss-120B (high)} & 8192 & 59.64 & 82.45 & 61.43 & 82.05 & 44.74 & 47.62 & 44.72 & 48.77 & 35.14 & 39.8 & 38.7 & 42.51 \\
 & 16384 & 68.51 & 76.46 & 71.03 & 78.04 & 45.86 & 46.17 & 47.42 & 48.28 & 38.74 & 40 & 41.61 & 42.47 \\
Qwen/Qwen3-Coder-30B-A3B-Instruct & 8192 & 28.63 & 28.63 & 27.4 & 27.4 & 21.92 & 21.92 & 24.49 & 24.49 & 26.8 & 26.8 & 24.61 & 24.61 \\ \bottomrule
\end{tabular}
}
\caption{\textbf{CLUTRR Performance Test Splits:} We show the kinship relation type prediction accuracy (\texttt{Acc}) as well as non-null accuracy (\texttt{NNAcc}) for all the test splits within CLUTRR.}
\label{tab:clutrr_results_test_splits}
\end{table*}

SLR-Bench is a large-scale, automatically generated benchmark for logical inductive reasoning, constructed using a fully automated framework that synthesizes prompts, validation programs, and latent rules without any human annotation.
It comprises 19k tasks organized into a curriculum of increasing relational, arithmetic, and recursive complexity, enabling fine-grained evaluation of logical inference capabilities in LLMs.
The benchmark includes validation programs, and the solutions produced by the LLMs are expressed as Prolog code, which is subsequently evaluated by a symbolic judge that executes the validation program over the LLM generated solution.
In the original work, the curriculum is used for training and is shown to provide benefits on the SLR-Bench test set as well as on other general reasoning benchmarks.
However, to remain consistent with our experimental setting, we perform a purely zero-shot evaluation on the combined test set, which includes the basic, easy, medium, and hard tiers, yielding a dataset of 1000 instances.
We report all metrics introduced by the original authors, including accuracy (\texttt{Acc}), partial score (\texttt{PS}), and syntax score (\texttt{SS}), which correspond respectively to the fraction of instances that are fully solved, the average fraction of examples correctly classified (analogous to test cases passed), and the fraction of instances for which a syntactically valid Prolog program is produced.
We use the exact prompts provided in the dataset, rather than creating our own templates, in order to ensure faithful and directly comparable evaluation.
The aggregate results are shown in Table~\ref{tab:slrbench_results}, while a breakdown by curriculum tier is provided in Table~\ref{tab:slrbench_results_test_splits}.
Overall, the results follow a trend similar to that observed on CLUTRR, where reasoning-oriented models generally outperform non-reasoning models, although the performance gap is smaller than CLUTRR.
Among all evaluated models, \texttt{gpt-oss-120b} with medium reasoning effort achieves the best performance, reaching 52.5\% accuracy.
Also consistent with CLUTRR, we observe that performance for both \texttt{gpt-oss} models peaks at medium reasoning effort and subsequently declines, primarily due to an increase in null outputs caused by unterminated chains-of-thought associated with overthinking.
A key complicating factor relative to CLUTRR and PBEBench-Lite is the inherent difficulty of producing syntactically valid Prolog programs.
Several models, including \texttt{QwQ-32B}, \texttt{Qwen3-32B}, and \texttt{gpt-oss-20b} with medium reasoning effort, achieve overall scores around 50\%, indicating substantial difficulty in generating syntactically correct code; in the case of \texttt{gpt-oss-20b}, this behavior is largely explained by a higher proportion of null outputs.
In contrast, coder-oriented models such as \texttt{Codestral-22B} and \texttt{Qwen2.5-Coder-32B-Instruct}, as well as \texttt{Qwen2.5-32B-Instruct}, excel at producing syntactically valid programs, achieving nearly 100\% syntax scores despite comparatively weak overall task performance.
Finally, for fairness and consistency, we use sequence lengths comparable to those employed for CLUTRR and PBEBench-Lite, and we allow a single attempt per instance for each LLM.

\begin{table*}[!tbh]
\centering
\begin{tabular}{@{}lrrrr@{}}
\toprule
Model & Acc & PS & SS & Nulls \\ \midrule
Codestral-22B & 15.5 & 59.99 & 99.2 & 0 \\
Qwen/Qwen2.5-32B-Instruct & 27.3 & 72.1 & 99.8 & 0 \\
Qwen/Qwen2.5-Coder-32B-Instruct & 28.4 & 71.01 & 99.9 & 0 \\
Qwen/QwQ-32B & 43.5 & 44.66 & 45.6 & 0 \\
Qwen/Qwen3-32B (with CoT) & 44.4 & 46.49 & 47.8 & 0 \\
Qwen/Qwen3-32B & 32.4 & 71.8 & 99 & 0 \\
Qwen/Qwen3-30B-A3B & 41.7 & 49.77 & 61.1 & 0 \\
DeepSeek-R1-Distill-Qwen-32B & 43.2 & 59.9 & 85.6 & 0 \\
gpt-oss-20B (low) & 36 & 57.6 & 76.5 & 0 \\
gpt-oss-20B (medium) & 44.4 & 49.91 & 52.6 & 468 \\
gpt-oss-20B (high) & 40.2 & 40.68 & 41 & 589 \\
gpt-oss-120B (low) & 43.2 & 76.75 & 99.1 & 0 \\
gpt-oss-120B (medium) & 52.5 & 67.47 & 81.4 & 175 \\
gpt-oss-120B (high) & 47.2 & 47.29 & 47.3 & 524 \\
Qwen/Qwen3-Coder-30B-A3B-Instruct & 26.5 & 53.98 & 73.3 & 0 \\ \bottomrule
\end{tabular}
\caption{\textbf{SLR-Bench Performance:} We compute the accuracy (fraction of cases successfully solved) (\texttt{Acc}), partial score (average number of examples correctly classified) (\texttt{PS}), \texttt{SS} (fraction of syntactically correct programs) and the number of null predictions (\texttt{Nulls}). We use the same max sequence length per model as CLUTRR.}
\label{tab:slrbench_results}
\end{table*}

\begin{table*}[!tbh]
\centering
\resizebox{\textwidth}{!}{
\begin{tabular}{@{}lrrrrrrrrrrrr@{}}
\toprule
\multirow{2}{*}{Model} & \multicolumn{3}{r}{Basic} & \multicolumn{3}{r}{Easy} & \multicolumn{3}{r}{Medium} & \multicolumn{3}{r}{Hard} \\ \cmidrule(l){2-13} 
 & Acc & PS & SS & Acc & PS & SS & Acc & PS & SS & Acc & PS & SS \\ \midrule
Codestral-22B & 58.8 & 58.8 & 100 & 2.4 & 2.4 & 99.6 & 0.8 & 0.8 & 99.6 & 0 & 0 & 97.6 \\
Qwen/Qwen2.5-32B-Instruct & 90.8 & 90.8 & 100 & 15.6 & 15.6 & 100 & 1.2 & 1.2 & 99.6 & 1.6 & 1.6 & 99.6 \\
Qwen/Qwen2.5-Coder-32B-Instruct & 92.4 & 92.4 & 100 & 19.6 & 19.6 & 100 & 1.2 & 1.2 & 100 & 0.4 & 0.4 & 99.6 \\
Qwen/QwQ-32B & 96 & 96 & 96.8 & 62.8 & 62.8 & 63.2 & 9.6 & 9.6 & 12.8 & 5.6 & 5.6 & 9.6 \\
Qwen/Qwen3-32B (with CoT) & 94 & 94 & 94 & 69.2 & 69.2 & 74 & 11.2 & 11.2 & 14.4 & 3.2 & 3.2 & 8.8 \\
Qwen/Qwen3-32B & 92.8 & 92.8 & 100 & 35.2 & 35.2 & 100 & 1.2 & 1.2 & 97.2 & 0.4 & 0.4 & 98.8 \\
Qwen/Qwen3-30B-A3B & 99.6 & 99.6 & 100 & 61.6 & 61.6 & 80 & 5.2 & 5.2 & 27.6 & 0.4 & 0.4 & 36.8 \\
DeepSeek-R1-Distill-Qwen-32B & 99.2 & 99.2 & 100 & 63.6 & 63.6 & 88.8 & 8.8 & 8.8 & 70.4 & 1.2 & 1.2 & 83.2 \\
gpt-oss-20B (low) & 93.6 & 93.6 & 100 & 47.2 & 47.2 & 84.4 & 2.8 & 2.8 & 54 & 0.4 & 0.4 & 67.6 \\
gpt-oss-20B (medium) & 96.8 & 96.8 & 100 & 70 & 70 & 77.6 & 8.4 & 8.4 & 23.6 & 2.4 & 2.4 & 9.2 \\
gpt-oss-20B (high) & 96 & 96 & 98.4 & 59.2 & 59.2 & 60 & 5.2 & 5.2 & 5.2 & 0.4 & 0.4 & 0.4 \\
gpt-oss-120B (low) & 96.8 & 96.8 & 100 & 67.6 & 67.6 & 100 & 5.6 & 5.6 & 99.6 & 2.8 & 2.8 & 96.8 \\
gpt-oss-120B (medium) & 98 & 98 & 99.6 & 87.6 & 87.6 & 96.8 & 19.6 & 19.6 & 58.8 & 4.8 & 4.8 & 70.4 \\
gpt-oss-120B (high) & 98.8 & 98.8 & 98.8 & 75.2 & 75.2 & 75.6 & 12.4 & 12.4 & 12.4 & 2.4 & 2.4 & 2.4 \\
Qwen/Qwen3-Coder-30B-A3B-Instruct & 88.4 & 88.4 & 94 & 16 & 16 & 66.4 & 0.8 & 0.8 & 68.4 & 0.8 & 0.8 & 64.4 \\ \bottomrule
\end{tabular}
}
\caption{\textbf{SLR-Bench Performance Test Splits:} We compute the accuracy (fraction of cases successfully solved) (\texttt{Acc}), partial score (average number of examples correctly classified) (\texttt{PS}), and \texttt{SS} (fraction of syntactically correct programs) for each split/curriculum tier of the test set.}
\label{tab:slrbench_results_test_splits}
\end{table*}
 
\subsection{Effect of More Examples}
Table~\ref{tab:pbebench_lite_more_eg_perf} illustrates the effect of varying the number of examples per PBE instance while keeping other factors, such as cascade distribution and relation type balance, constant using the PBEBench-Lite-MoreEg snapshot. The results for gpt-oss-20b and gpt-oss-120b show largely similar performance for gpt-oss-20b, but a decrease for gpt-oss-120b. This suggests that increasing the number of examples can sometimes make the task harder for LLMs, which is counterintuitive, as more examples would ideally simplify the task. Analysis of the average number of changes per program reveals that PBEBench-Lite has fewer absolute changes (1.56 words out of 5 on average) compared to PBEBench-Lite-MoreEg (7.07 words out of 50 on average). However, the relative number of changes is lower in PBEBench-Lite-MoreEg (14\% vs.\ higher in PBEBench-Lite), indicating lower information density, which may explain why the task becomes harder despite having more examples.

\subsection{Real SLI Results}
We evaluate open-source LLMs on the real SLI dataset using \texttt{Pass@1}, \texttt{Edit\_Sim}, and \texttt{Valid Rate}, with results reported in Table~\ref{tab:real_sli_results}.
Evaluation is performed on 21 PBE instances under substantially relaxed decoding constraints to estimate an upper bound on model performance.
Specifically, we use a maximum sequence length four times larger than the corresponding PBEBench-Lite settings and a sampling budget of 32 attempts per instance.
Models are additionally allowed to generate cascades of up to length 50, with each rewrite rule permitting $\alpha$ and $\beta$ substrings in the \texttt{replace($\alpha$, $\beta$)} of up to five characters.
Despite these concessions, all evaluated models struggle on real SLI.
Only the gpt-oss family, and DeepSeek-R1-Distill-Qwen and Qwen3-32B with chain-of-thought prompting achieve \texttt{Pass@1} of 10\% or higher.
Among them, gpt-oss-120b performs the best, reaching a \texttt{Pass@1} of 33\%.

\begin{table*}[!tbh]
\centering
\resizebox{\textwidth}{!}{
\begin{tabular}{@{}lrrrrrrr@{}}
\toprule
\multirow{2}{*}{Model} & \multirow{2}{*}{Max Seq Len} & \multicolumn{3}{r}{First Code Block} & \multicolumn{3}{r}{Last Code Block} \\ \cmidrule(l){3-8} 
 &  & Pass@1 & Edit Sim & Valid Rate & Pass@1 & Edit Sim & Valid Rate \\ \midrule
Codestral-22B & 8192 & 0 & 31 & 81 & 0 & 31 & 81 \\
Qwen2.5-32B-Instruct & 8192 & 0 & 32 & 88 & 0 & 32 & 88 \\
Qwen2.5Coder-32B-Instruct & 8192 & 0 & 31 & 82 & 0 & 33 & 80 \\
QwQ-32B & 32768 & 5 & 63 & 94 & 5 & 63 & 94 \\
Qwen/Qwen3-32B (with CoT) & 32768 & 10 & 60 & 97 & 10 & 60 & 97 \\
Qwen/Qwen3-32B & 8192 & 5 & 60 & 79 & 5 & 60 & 79 \\
Qwen3-30B-A3B & 32768 & 5 & 9 & 97 & 5 & 9 & 97 \\
DeepSeek-R1-Distill-Qwen-32B & 32768 & 10 & 51 & 95 & 10 & 51 & 95 \\
gpt-oss-20B (high) & 32768 & 24 & 82 & 88 & 24 & 82 & 88 \\
gpt-oss-120B (high) & 32768 & 33 & 86 & 88 & 33 & 86 & 88 \\
Qwen/Qwen3-Coder-30B-A3B-Instruct & 8192 & 5 & 23 & 87 & 5 & 20 & 96 \\ \bottomrule
\end{tabular}
}
\caption{\textbf{Real SLI Performance:} We evaluate all open source LLMs using the same metrics and prompts as PBEBench-Lite. Each LLM is given 32 attempts to solve the 21 PBE instances and each instance contains 50 input-output examples.}
\label{tab:real_sli_results}
\end{table*}

\begin{table}[!tbh]
\centering
\begin{tabular}{@{}lrr@{}}
\toprule
Dataset & r & p \\ \midrule
CLUTRR & 0.6364 & 0.035287 \\
SLR-Bench & 0.8091 & 0.002559 \\
PBEBench-Lite & 0.8273 & 0.001677 \\ \bottomrule
\end{tabular}
\caption{\textbf{Benchmark Ranking Correlations:} Spearman rank correlations between open-source model rankings induced by inductive reasoning benchmarks and real SLI performance.
For PBEBench-Lite and real SLI, rankings we use \texttt{Pass@1} as the primary metric with \texttt{Edit\_Sim} as a tie breaker.
For CLUTRR, Acc and NNAcc are used, while for SLR-Bench, Acc and Partial Score serve as the primary and tie-breaking metrics, respectively.}
\label{tab:benchmark_ranking_correlations}
\end{table}

\subsection{Efficiency of Problem Proposer}
\label{sec:appendix:efficiency}
For each desired ground-truth cascade length, our data generation procedure is designed to yield a balanced distribution across relation type categories, where each category is represented by a binary vector. Let $U$ denote the ideal uniform distribution over all categories, and let $Q$ denote the empirical distribution obtained from the generated data. To quantify the deviation of $Q$ from the ideal $U$, we use the Kullback–Leibler (KL) divergence:
\[
D_{\mathrm{KL}}(U || Q) \;=\; \sum_{x} U(x) \log \frac{U(x)}{Q(x)} 
\]
where $x$ ranges over all relation type categories. By construction, $D_{\mathrm{KL}}(U || Q) \geq 0$ and equals zero only when $Q=U$. We also \textbf{apply smoothing} for categories where zero instances are observed in the empirical distribution (missing categories) to prevent the divergence from shooting up to infinity for these cases. 

Our data generation algorithm (the \emph{problem proposer}) employs rejection sampling to enforce balance across categories. Initially, each sampled datapoint is accepted only if it maintains near-uniform coverage across categories. As sampling progresses, however, this constraint becomes increasingly difficult to satisfy, and the efficiency of rejection sampling deteriorates due to a growing rejection rate. To mitigate this, we introduce a \emph{patience} parameter that limits the number of constrained steps. Once the patience threshold (e.g., 100000) is exhausted, the algorithm relaxes the balancing constraint and accepts datapoints unconditionally. This enables continued large-scale data generation, though at the cost of increased divergence $D_{\mathrm{KL}}(U || Q)$ from the uniform distribution. Importantly, this relaxation applies only to category balancing: the cascade length constraints remain strictly enforced, and each generated program must still achieve the target ground-truth cascade length and modify at least one example.

We visualize the efficiency of our data generation process for different values of the patience parameter $\tau$—100{,}000; 250{,}000; 500{,}000; 750{,}000; and 1{,}000{,}000—across cascade lengths 5, 10, 15, 20, and 25, annotating each point with the KL divergence from the desired uniform distribution over all 16 relation type categories. Points with zero divergence (perfectly balanced distribution) are marked in green, while others are marked in red. The plots of efficiency versus patience for each cascade length are shown in Fig.~\ref{fig:efficiency_patience_tradeoff_cascade_len_5}--\ref{fig:efficiency_patience_tradeoff_cascade_len_25}.
In all plots, the y-axis represents percentage efficiency (so 1 corresponds to 1\% or 0.01). For cascade lengths 5 and 10, patience has little effect since perfect KL is always achievable; variations reflect random fluctuations across runs. For cascade lengths 15--25, perfect KL is no longer guaranteed and patience begins to influence efficiency. As expected, efficiency generally decreases with higher patience due to discarding more examples. For lengths 15 and 25, slightly lower KL can be achieved with greater patience, while for length 20, KL remains unchanged, indicating diminishing returns: higher patience does not necessarily reduce divergence from uniform distribution. 
These results motivate selecting a reasonable, but not excessive, patience value.

\begin{figure}
    \centering
    \includegraphics[width=0.5\textwidth]{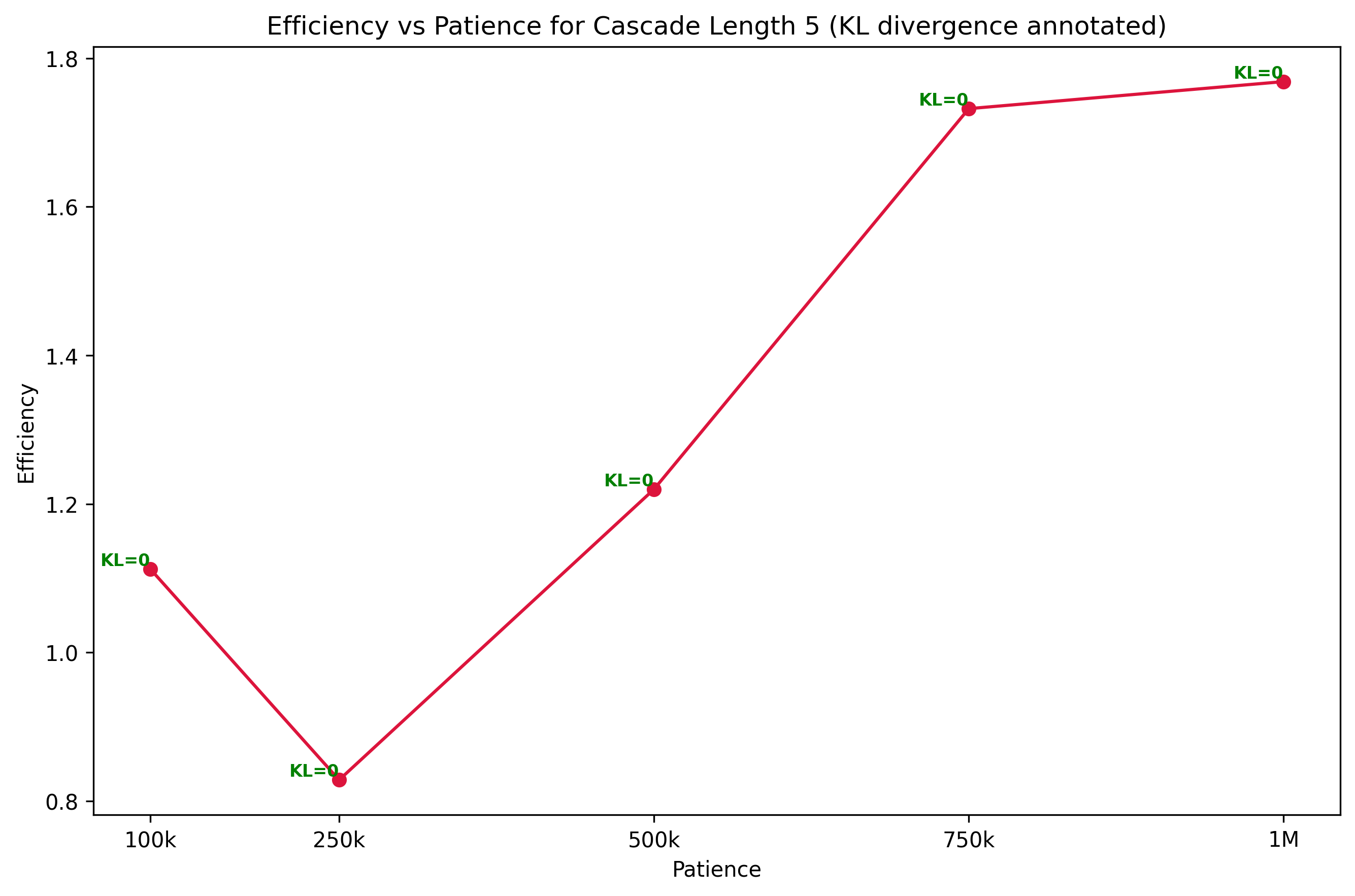}
    \caption{\textbf{Efficiency vs Patience:} Efficiency of the rejection sampling process for generating ground truth cascades of length 5 for various values of the patience parameter. The KL divergence from the ideal balanced distrbution is annotated per point with zero corresponding to achieving a perfectly balanced distribution.}
\label{fig:efficiency_patience_tradeoff_cascade_len_5}
\end{figure}

\begin{figure}
    \centering
    \includegraphics[width=0.5\textwidth]{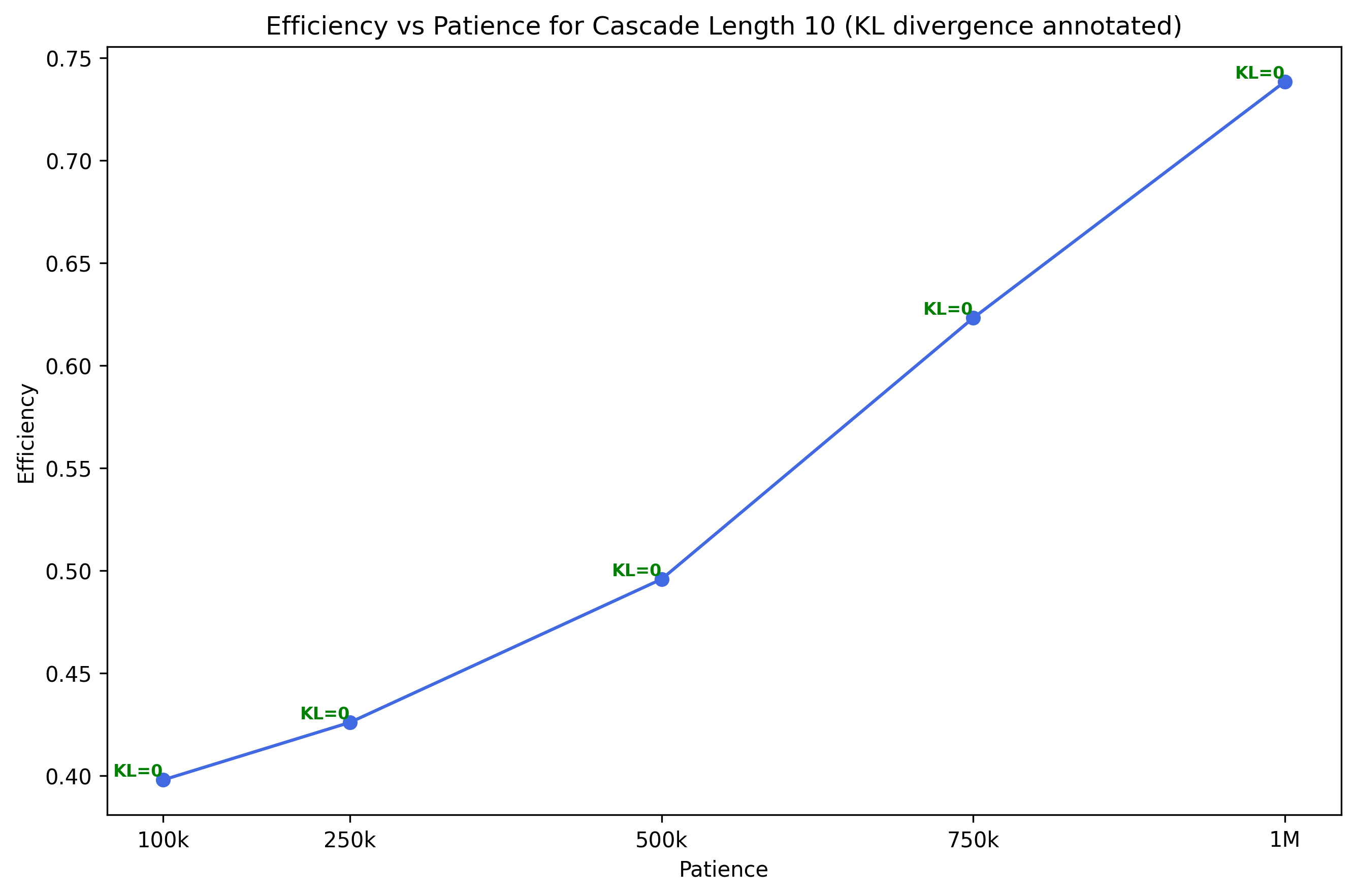}
    \caption{\textbf{Efficiency vs Patience:} Efficiency of the rejection sampling process for generating ground truth cascades of length 10 for various values of the patience parameter. The KL divergence from the ideal balanced distrbution is annotated per point with zero corresponding to achieving a perfectly balanced distribution.}
\label{fig:efficiency_patience_tradeoff_cascade_len_10}
\end{figure}

\begin{figure}
    \centering
    \includegraphics[width=0.5\textwidth]{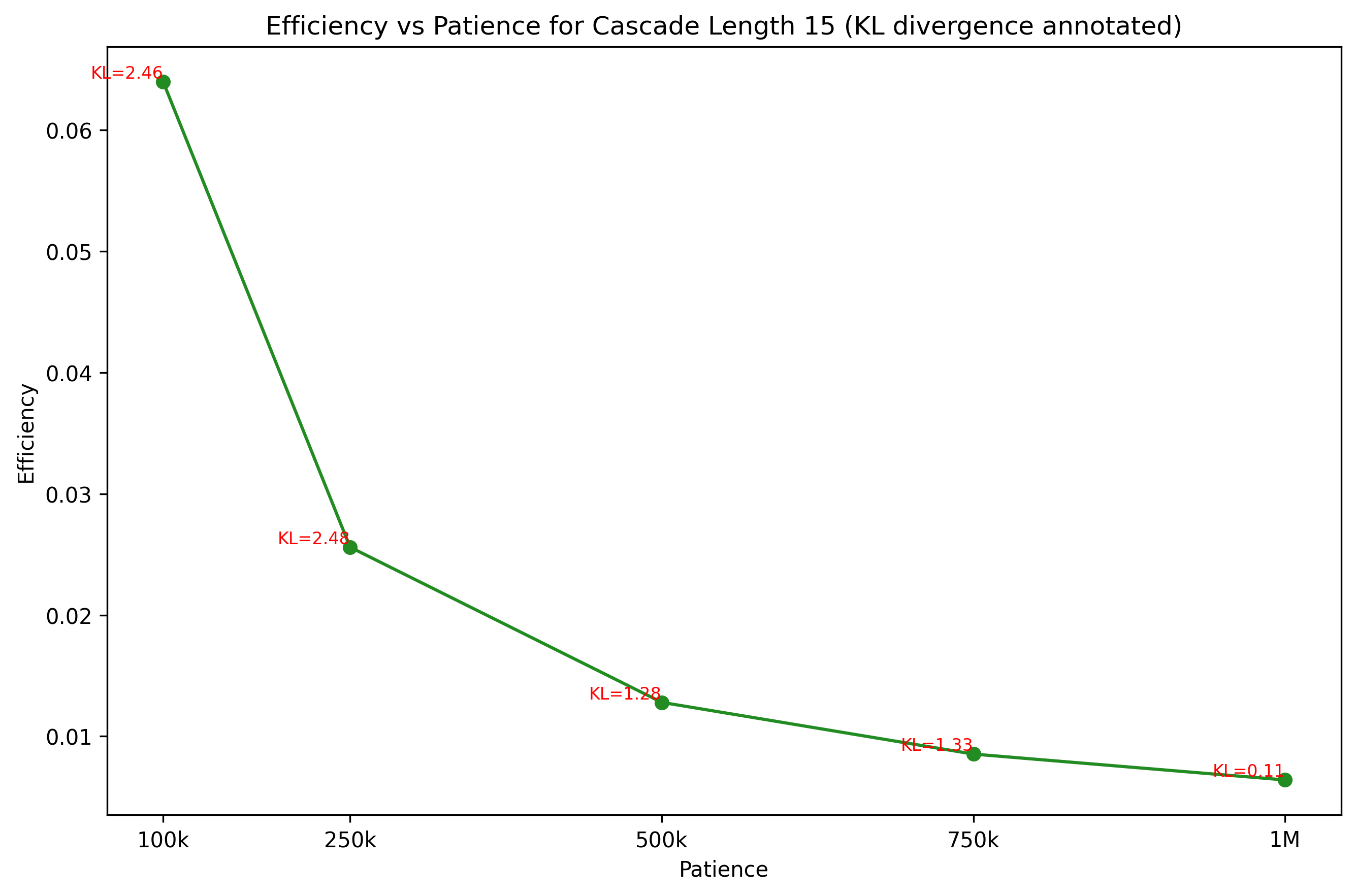}
    \caption{\textbf{Efficiency vs Patience:} Efficiency of the rejection sampling process for generating ground truth cascades of length 15 for various values of the patience parameter. The KL divergence from the ideal balanced distrbution is annotated per point with zero corresponding to achieving a perfectly balanced distribution.}
\label{fig:efficiency_patience_tradeoff_cascade_len_15}
\end{figure}

\begin{figure}
    \centering
    \includegraphics[width=0.5\textwidth]{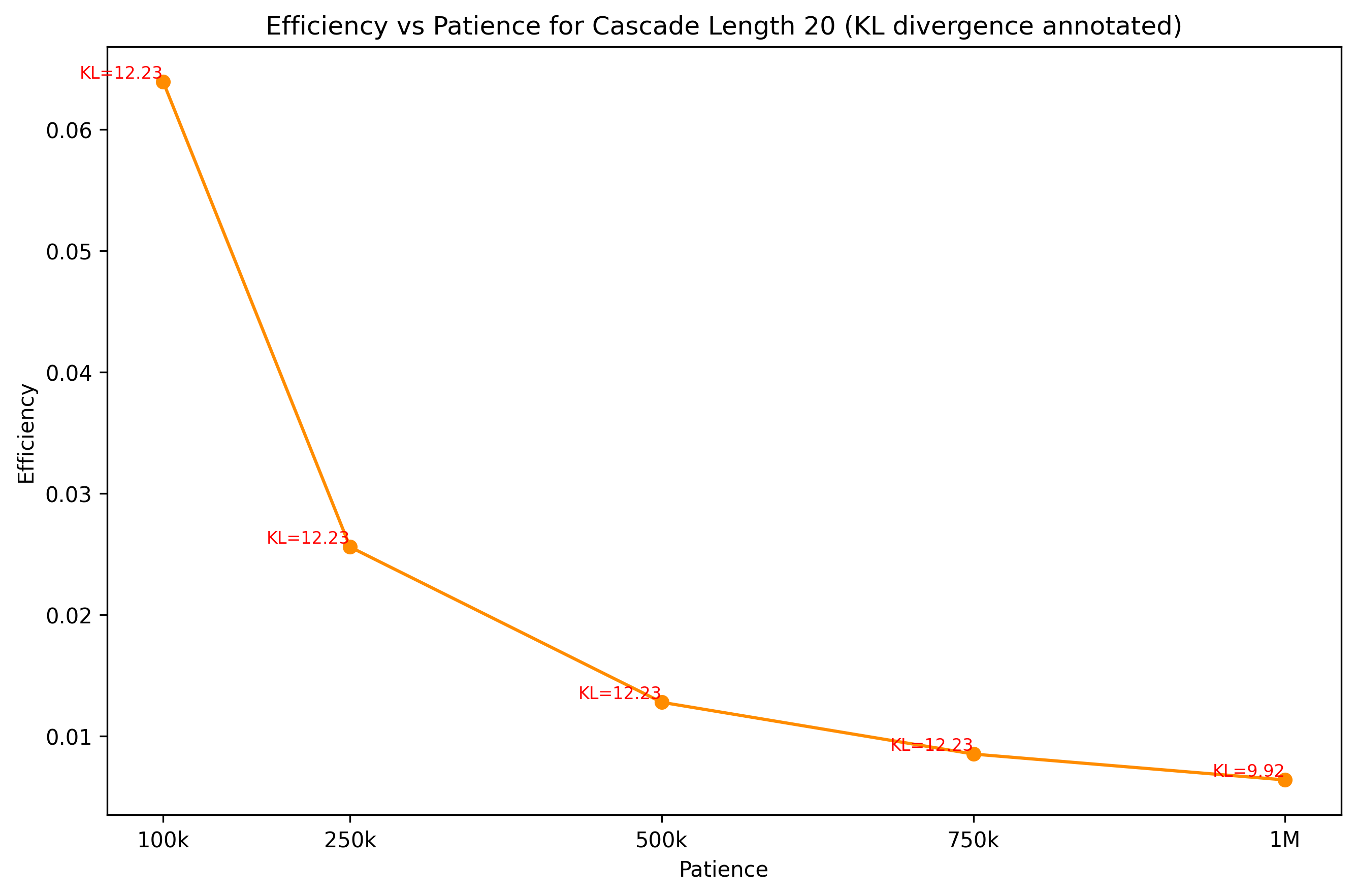}
    \caption{\textbf{Efficiency vs Patience:} Efficiency of the rejection sampling process for generating ground truth cascades of length 20 for various values of the patience parameter. The KL divergence from the ideal balanced distrbution is annotated per point with zero corresponding to achieving a perfectly balanced distribution.}
\label{fig:efficiency_patience_tradeoff_cascade_len_20}
\end{figure}

\begin{figure}
    \centering
    \includegraphics[width=0.5\textwidth]{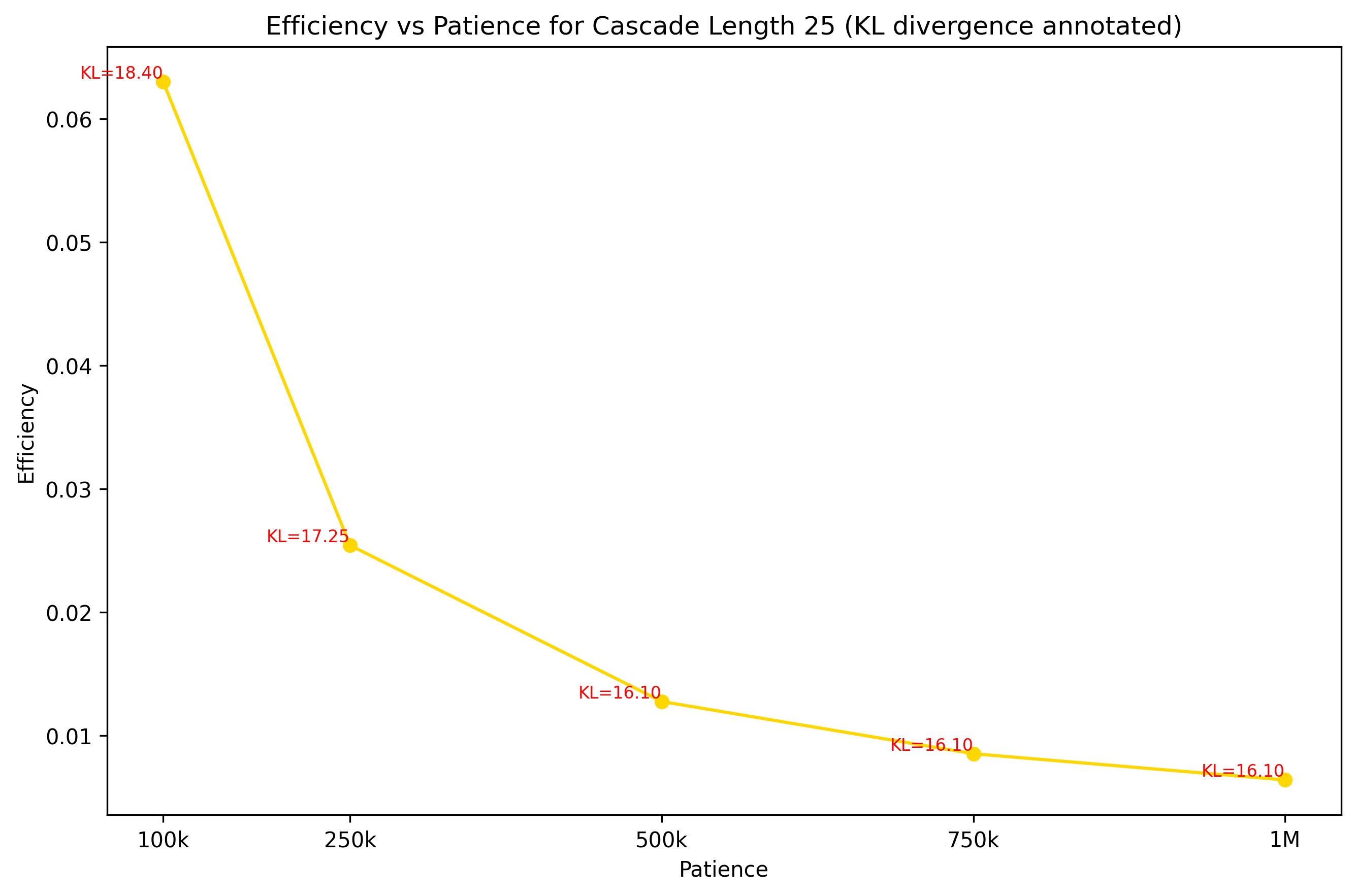}
    \caption{\textbf{Efficiency vs Patience:} Efficiency of the rejection sampling process for generating ground truth cascades of length 25 for various values of the patience parameter. The KL divergence from the ideal balanced distrbution is annotated per point with zero corresponding to achieving a perfectly balanced distribution.}
\label{fig:efficiency_patience_tradeoff_cascade_len_25}
\end{figure}

\begin{figure*}[t]
\centering
\begin{tcolorbox}[
    width=\textwidth,
    colback=green!5,
    colframe=green!40!black,
    title={\textbf{Program Ordering Prompt}},
    coltitle=white,
    colbacktitle=green!75!black,
]
You are solving a \textbf{program ordering puzzle}. Given input-output string pairs and a scrambled list of string replacement programs, your goal is to determine the correct execution order.

\medskip
\begin{verbatim}
## Background
\end{verbatim}

Each program performs a Python string replacement:
\verb|replace("A", "B")| replaces all occurrences of \texttt{"A"} with \texttt{"B"}.

\medskip
\textbf{Why order matters:}
\begin{itemize}
\item \textbf{Feeding:} One program creates substrings that another program can match.  \\
Example: \verb|replace("a","bc")| followed by \verb|replace("bc","x")|.
\item \textbf{Bleeding:} One program removes substrings that another program would have matched.  \\
Example: \verb|replace("ab","x")| followed by \verb|replace("a","y")|.
\end{itemize}

\medskip
\begin{verbatim}
## Your Task
\end{verbatim}

\begin{verbatim}
**Inputs:** {inputs}

**Outputs:** {outputs}

**Scrambled Programs** (indices 0 to {n_minus_1}):
{programs_formatted}
\end{verbatim}

Find the ordering \verb|[i0, i1, ..., i_{n_minus_1}]| such that applying programs in that order transforms each input to its corresponding output.

\medskip
\begin{verbatim}
## Approach
\end{verbatim}

\begin{enumerate}
\item Trace through what each program does
\item Identify potential feeding/bleeding interactions
\item Reason about which programs must come before others
\item Verify your ordering produces the expected outputs
\end{enumerate}

\medskip
\begin{verbatim}
## Output Format
\end{verbatim}

Provide your final answer as a JSON array of indices:

\begin{verbatim}
```json
[i0, i1, i2, ...]
```
\end{verbatim}

Your ordering must be a permutation of \verb|[0, 1, ..., {n_minus_1}]|.
\end{tcolorbox}
\end{figure*}

\clearpage
\subsection{Factorial Analysis}
\label{sec:appendix:factorial_analysis}
We conduct a factorial analysis to answer the following questions:

\subsubsection{Effect of Long Chain of Thought Reasoning}
We analyzed the effect of long chain-of-thought (LCoT) reasoning by selecting three models where it can be toggled on or off (Qwen3-32B, Claude-3.7-Sonnet, and Claude-4-Sonnet) on PBEBench-Lite (1008 instances) to evaluate its benefits. Independent variables included \texttt{model\_id} (Qwen, Claude-3.7, or Claude-4), \texttt{reasoning} (LCoT enabled or not), \texttt{cascade\_len} (ground-truth cascade length), and the presence of BFCC relations: \texttt{feeding}, \texttt{bleeding}, \texttt{counter-feeding}, and \texttt{counter-bleeding}. The dependent variable was binary Pass@1 (\texttt{passing}) per instance.

We fit a binary logistic regression model with \texttt{model\_id}, \texttt{reasoning}, \texttt{feeding}, \texttt{bleeding}, \texttt{counter-feeding}, and \texttt{counter-bleeding} as nominal predictors, and \texttt{cascade\_len} as a numeric covariate. All pairwise interaction terms were included. The Deviance goodness-of-fit test was non-significant, $\chi^{2}(6019) = 5398.5$, $p = 1.0$, indicating adequate model fit. The model explained $24.53\%$ of the variance in passing ($R^{2}_{\text{adj}}$). Wald tests for main effects and significant interactions are summarized in Table~\ref{tab:factorial_analysis:effect_of_reasoning}.

\begin{table*}[!tbh]
    \centering
    \begin{tabular}{@{}llrrl@{}}
    \toprule
    Term & df & $\chi^{2}$ & $p$ & Interpretation \\ \midrule
    Total model & 28 & 1027.68 & .000 & Deviance R\textasciicircum{}2 (adj) = 24.53\% \\
    cascade\_len & 1 & 415.35 & .000 & \begin{tabular}[c]{@{}l@{}}Passing is less likely as \\ cascade\_len increases\end{tabular} \\
    model\_id & 2 & 128.02 & .000 & \begin{tabular}[c]{@{}l@{}}The two sonnet models have \\ a higher pass rate than Qwen\end{tabular} \\
    reasoning & 1 & 18.70 & .000 & \begin{tabular}[c]{@{}l@{}}Reasoning models are more \\ likely to pass\end{tabular} \\
    feeding & 1 & 15.69 & .000 & \begin{tabular}[c]{@{}l@{}}Feeding reduces probability \\ of passing\end{tabular} \\
    bleeding & 1 & 36.10 & .000 & \begin{tabular}[c]{@{}l@{}}Bleeding reduces probability \\ of passing\end{tabular} \\
    counter-feeding & 1 & 21.37 & .000 & \begin{tabular}[c]{@{}l@{}}Counter-feeding reduces \\ probability of passing\end{tabular} \\
    counter-bleeding & 1 & 4.67 & .031 & \begin{tabular}[c]{@{}l@{}}Counter-bleeding increases \\ probability of passing\end{tabular} \\
    model\_id $\times$ reasoning & 2 & 189.59 & .000 & \begin{tabular}[c]{@{}l@{}}For reasoning models, Qwen \\ has a higher probability of \\ passing than the two Sonnet \\ models, which are not different \\ from each other.  For \\ non-reasoning models, Claude 4 \\ is better than Clause 3.7, which \\ is better than Qwen.\end{tabular} \\
    feeding $\times$ bleeding & 1 & 33.17 & .000 & \begin{tabular}[c]{@{}l@{}}Adding either feeding or bleeding \\ reduces the probability of a pass, \\ but both together is better than \\ just feeding.  Just bleeding is not \\ worse than having both and not \\ better than just feeding.\end{tabular} \\
    bleeding $\times$ counter-bleeding & 1 & 30.86 & .000 & \begin{tabular}[c]{@{}l@{}}With counter-bleeding, bleeding \\ loses its effect on probability \\ of a pass.\end{tabular} \\ \bottomrule
    \end{tabular}
    \caption{\textbf{Reasoning Manipulation Analysis on PBEBench-Lite:} Wald test results for main effects and significant interactions in the logistic regression predicting passing.}
    \label{tab:factorial_analysis:effect_of_reasoning}
\end{table*}

\subsubsection{Analysis of Models of Varying Capabilities}
We analyzed what makes PBEBench-Lite instances difficult for three representative models of different capabilities. Independent variables included \texttt{cascade\_len} (ground-truth cascade length) and the presence of BFCC relations: \texttt{feeding}, \texttt{bleeding}, \texttt{counter-feeding}, and \texttt{counter-bleeding}. The dependent variable was binary Pass@1 (\texttt{passing}). All models were fit with binary logistic regression, including all pairwise interaction terms. Results are summarized below.

\textbf{Codestral-22B (weakest model).} For Codestral-22B, only the effect of \texttt{cascade\_len} was significant. Passing was less likely as the cascade length increased. Other predictors showed no detectable effects.

\textbf{QwQ-32B (moderately good model).} For QwQ-32B, the model fit the data adequately, $\chi^{2}(996) = 1946.72$, $p = .129$. The model explained $19.71\%$ of variance ($R^{2}_{\text{adj}}$). Wald tests are summarized in Table~\ref{tab:factorial_analysis:qwq_analysis}. Cascade length, feeding, and bleeding all significantly reduced the probability of passing, while counter-feeding and counter-bleeding showed no effects. An interaction revealed that the joint presence of feeding and bleeding was less detrimental than either alone.

\begin{table*}[!tbh]
    \centering
    \begin{tabular}{@{}llrrl@{}}
    \toprule
    Term & df & $\chi^{2}$ & $p$ & Interpretation \\ \midrule
    Total model & 11 & 197.43 & .000 & Deviance R\textasciicircum{}2 (adj) = 19.71\% \\
    cascade\_len & 1 & 114.54 & .000 & \begin{tabular}[c]{@{}l@{}}Passing is less likely as \\ cascade\_len increases\end{tabular} \\
    feeding & 1 & 5.74 & .017 & \begin{tabular}[c]{@{}l@{}}Feeding reduces the \\ probability of a pass\end{tabular} \\
    bleeding & 1 & 7.37 & .007 & \begin{tabular}[c]{@{}l@{}}Bleeding reduces the \\ probability of a pass\end{tabular} \\
    counter-feeding & 1 & .22 & .637 & No effect \\
    counter-bleeding & 1 & .47 & .495 & No effect \\
    feeding $\times$ bleeding & 1 & 10.03 & .002 & \begin{tabular}[c]{@{}l@{}}Though bleeding or feeding \\ alone reduce the probability \\ of a pass, if they are both \\ present, they do not.\end{tabular} \\ \bottomrule
    \end{tabular}
    \caption{\textbf{QwQ-32B Analysis on PBEBench-Lite:} Wald test results for main effects and significant interactions in the logistic regression predicting passing.}
    \label{tab:factorial_analysis:qwq_analysis}
\end{table*}

\begin{table}[!tbh]
\centering
\begin{tabular}{@{}lrl@{}}
\toprule
Term & Coefficient & Interpretation \\ \midrule
cascade\_length & -0.521 & \begin{tabular}[c]{@{}l@{}}Strongest predictor. \\ Longer = harder.\end{tabular} \\
counterfeeding & -0.291 & Hardest relation. \\
counterbleeding & -0.068 & Moderate difficulty. \\
bleeding & -0.034 & \begin{tabular}[c]{@{}l@{}}Small negative \\ effect.\end{tabular} \\
feeding & -0.004 & Negligible effect. \\ \bottomrule
\end{tabular}
\caption{\textbf{Difficulty Analysis of PBEBench-Lite-Perm:} Logistic regression coefficients for factors affecting the difficulty of program reordering and interpretation of the values.}
\label{tab:pbebench_lite_perm_difficuly_analysis}
\end{table}

\textbf{GPT-5 (strongest model).} For GPT-5, the model fit the data adequately, $\chi^{2}(996) = 982.16$, $p = .617$. The model explained $17.28\%$ of variance ($R^{2}_{\text{adj}}$). Wald tests are summarized in Table~\ref{tab:gpt5-analysis}. Cascade length, feeding, and bleeding significantly reduced passing. Counter-feeding and counter-bleeding had no main effects, but an interaction showed that bleeding reduced passing only when counter-bleeding was absent.

\begin{table*}[!tbh]
    \centering
    \begin{tabular}{@{}llrrl@{}}
    \toprule
    Term & df & $\chi^{2}$ & $p$ & Interpretation \\ \midrule
    Total model & 11 & 158.16 & .000 & Deviance R\textasciicircum{}2 (adj) = 17.28\% \\
    cascade\_len & 1 & 63.76 & .000 & \begin{tabular}[c]{@{}l@{}}Passing is less likely as \\ cascade\_len increases\end{tabular} \\
    feeding & 1 & 9.48 & .002 & \begin{tabular}[c]{@{}l@{}}Feeding reduces the \\ probability of a pass\end{tabular} \\
    bleeding & 1 & 4.95 & .026 & \begin{tabular}[c]{@{}l@{}}Bleeding reduces the \\ probability of a pass\end{tabular} \\
    counter-feeding & 1 & 1.16 & .282 & No effect \\
    counter-bleeding & 1 & .54 & .464 & No effect \\
    bleeding $\times$ counter-bleeding & 1 & 8.50 & .004 & \begin{tabular}[c]{@{}l@{}}Bleeding only reduces \\ probability of a pass if \\ counter-bleeding is not present\end{tabular} \\ \bottomrule
    \end{tabular}
    \caption{\textbf{GPT-5 Analysis on PBEBench-Lite:} Wald test results for main effects and significant interactions in the logistic regression predicting passing.}
    \label{tab:gpt5-analysis}
\end{table*}

\subsection{Logistic Regression Analysis}
\label{sec:appendix:logistic_regression}
We conduct a logistic regression analysis on gpt-oss-120b predictions on PBEBench for cascades of length 2 to 20, on the following factors influencing difficulty: ground-truth cascade length (\texttt{cascade\_len}), the presence of BFCC relations: \texttt{feeding}, \texttt{bleeding}, \texttt{counter\_feeding}, and \texttt{counter\_bleeding}. 
The dependent variable was binary Pass@1 (passing) per instance.
The logistic regression analysis reveals that all the analyzed factors have a statistically significant impact on the success or passing, with feeding, counter feeding, and cascade length having strong negative effects, with cascade length having the strongest impact. However, bleeding has a slight positive effect, and counter-bleeding has a very weak negative effect. This indicates that the presence of bleeding can make the problems easier to solve, while counter bleeding only has a small effect on increasing hardness.

\begin{table}[!tbh]
    \centering
    \begin{tabular}{@{}lrr@{}}
    \toprule
    Term & Coefficient & $p$ \\ \midrule
    intercept & 2.152 & 0 \\
    feeding & -0.254 & 4.77e-19 \\
    bleeding & 0.134 & 2.92e-6 \\
    counter\_feeding & -0.176 & 6.27e-10 \\
    counter\_bleeding & -0.08 & 0.005 \\
    cascade\_len & -0.347 & 0 \\ \bottomrule
    \end{tabular}
    \caption{\textbf{Logistic Regression Difficulty Analysis of gpt-oss-120b on PBEBench:} Model predictions were analyzed with a sampling budget of 32 and maximum sequence length of 16,384. Each attempt across the 1,216 instances was treated as a datapoint, yielding 38,912 datapoints in total.}
    \label{tab:logistic_regression_gpt_oss_120B}
\end{table}

\begin{table}[!tbh]
\centering
\begin{tabular}{@{}rrr@{}}
\toprule
\multicolumn{1}{l}{\multirow{2}{*}{\begin{tabular}[c]{@{}l@{}}Cascade\\ Length\end{tabular}}} & \multicolumn{2}{r}{Pass@1} \\ \cmidrule(l){2-3} 
\multicolumn{1}{l}{} & First Code Block & Last Code Block \\ \midrule
2 & 1 & 1 \\
3 & 1 & 1 \\
4 & 0.9688 & 0.9688 \\
5 & 1 & 1 \\
6 & 0.9688 & 0.9688 \\
7 & 0.9688 & 0.9688 \\
8 & 0.9688 & 0.9688 \\
9 & 0.8438 & 0.8438 \\
10 & 0.8438 & 0.8438 \\
11 & 0.875 & 0.875 \\
12 & 0.7656 & 0.7656 \\
13 & 0.6719 & 0.6719 \\
14 & 0.4531 & 0.4531 \\
15 & 0.5 & 0.5 \\
16 & 0.3906 & 0.3906 \\
17 & 0.2812 & 0.2812 \\
18 & 0.1719 & 0.1719 \\
19 & 0.0781 & 0.0781 \\
20 & 0.0469 & 0.0469 \\ \bottomrule
\end{tabular}
\caption{\textbf{Performance across Cascade Lengths (gpt-oss-120b):} Variation of the performance of gpt-oss-120b across cascades of length 2-20 with sampling budget of 32 and max sequence length of 16384.}
\label{tab:gpt_oss_120b_cascade_len_perf}
\end{table}

\begin{table*}[!tbh]
\centering
\begin{tabular}{@{}lrrrrrr@{}}
\toprule
\multirow{2}{*}{\textbf{Model}} & \multicolumn{3}{r}{\textbf{First Code Block}} & \multicolumn{3}{r}{\textbf{Last Code Block}} \\
\cmidrule(l){2-7}
 & \textbf{Pass@1} & \textbf{Edit Sim} & \textbf{Valid Rate} & \textbf{Pass@1} & \textbf{Edit Sim} & \textbf{Valid Rate} \\
\midrule
gpt-oss-20b \reasoningemoji \moeemoji & 0.3958 & 0.4662 & 0.9748 & 0.3958 & 0.4662 & 0.9748 \\
gpt-oss-120b \reasoningemoji \moeemoji & 0.5542 & 0.7018 & 0.9211 & 0.5542 & 0.7018 & 0.9211 \\
\bottomrule
\end{tabular}
\caption{\textbf{PBEBench-Lite-MoreEg Performance:} We compute the Pass@1 and edit similarity as the coarse and fine-grained evaluation, respectively, for each model. \moeemoji - indicates mixture-of-experts (or MoE) model and \reasoningemoji - indicates reasoning on for model.}
\label{tab:pbebench_lite_more_eg_perf}
\end{table*}

\begin{table}[!tbh]
\centering
\begin{tabular}{@{}lcr@{}}
\toprule
Cascade Length & Sampling Budget & Pass@1 \\ \midrule
\multirow{4}{*}{20} & 1 & 0.1543 \\
 & 2 & 0.2416 \\
 & 3 & 0.2974 \\
 & 4 & 0.3377 \\
\multirow{4}{*}{25} & 1 & 0.0547 \\
 & 2 & 0.0911 \\
 & 3 & 0.1172 \\
 & 4 & 0.1406 \\
\multirow{4}{*}{30} & 1 & 0.0117 \\
 & 2 & 0.0234 \\
 & 3 & 0.0351 \\
 & 4 & 0.0469 \\ \bottomrule
\end{tabular}
\caption{\textbf{Performance across Cascade Lengths (GPT-5):} Variation of the performance of GPT-5 across cascades with sampling budget ranging from 1 to 4, medium reasoning effort and 65536 max completion tokens (which includes both reasoning and output tokens).}
\label{tab:gpt5_cascade_len_perf_and_sampling_budget_scaling}
\end{table}

\begin{table*}[!tbh]
\centering
\begin{tabular}{@{}lrr@{}}
\toprule
Sample Budget & First Code Block & Last Code Block \\ \midrule
1 & 0.1543 & 0.1543 \\
2 & 0.2416 & 0.2416 \\
3 & 0.2974 & 0.2974 \\
4 & 0.3377 & 0.3377 \\
5 & 0.3694 & 0.3694 \\
6 & 0.3956 & 0.3956 \\
7 & 0.418 & 0.418 \\
8 & 0.4375 & 0.4375 \\ \bottomrule
\end{tabular}
\caption{\textbf{Performance across sampling budget (GPT-5):} Variation of the performance of GPT-5 with sampling budget for cascade length of 20 for medium reasoning effort and 65536 max completion tokens.}
\label{tab:gpt5_sampling_budget}
\end{table*}

\begin{table*}[!tbh]
\centering
\begin{tabular}{@{}lccrrrr@{}}
\toprule
\begin{tabular}[c]{@{}l@{}}Max\\ Completion\\ Tokens\end{tabular} & \begin{tabular}[c]{@{}c@{}}Reasoning\\ Effort\end{tabular} & \begin{tabular}[c]{@{}c@{}}Avg\\ Tokens\\ Used\end{tabular} & Pass@1 & \begin{tabular}[c]{@{}r@{}}Null \\ Responses (\%)\end{tabular} & \begin{tabular}[c]{@{}r@{}}Non null\\ Correct\\ Responses (\%)\end{tabular} & \begin{tabular}[c]{@{}r@{}}Non null\\ Incorrect\\ Responses (\%)\end{tabular} \\ \midrule
512 & minimal & 194 & {\color[HTML]{000000} 0} & 0 & 0 & 100 \\
1024 & low & 1024 & {\color[HTML]{000000} 0} & 100 & 0 & 0 \\
2048 & low & 2046 & {\color[HTML]{000000} 0.0156} & 98.4 & 1.56 & 0 \\
4096 & low & 3909 & {\color[HTML]{000000} 0.1719} & 72 & 17.19 & 10.81 \\
4096 & medium & 4096 & {\color[HTML]{000000} 0} & 100 & 0 & 0 \\
8192 & low & 5041 & {\color[HTML]{000000} 0.6094} & 3.1 & 60.94 & 35.96 \\
8192 & medium & 7726 & {\color[HTML]{000000} 0.3906} & 61 & 39 & 0 \\
16384 & medium & 9738 & {\color[HTML]{000000} 0.8125} & 4.7 & 81.25 & 14.05 \\
16384 & high & 13358 & {\color[HTML]{000000} 0.6875} & 23 & 68.75 & 8.2 \\
32768 & high & 14045 & {\color[HTML]{000000} 0.9219} & 0 & 92.19 & 7.81 \\
65536 & high & 14549 & {\color[HTML]{000000} 0.9375} & 0 & 93.75 & 6.25 \\ \bottomrule
\end{tabular}
\caption{\textbf{Performance across reasoning efforts and max completion tokens (GPT-5):} Variation of the performance of GPT-5 with reasoning effort and max completion tokens for cascade length of 10, and sampling budget of 1.}
\label{tab:gpt5_max_len_scaling}
\end{table*}

\begin{table*}[!tbh]
\centering
\begin{tabular}{@{}lrrrrrrr@{}}
\toprule
\multirow{2}{*}{\begin{tabular}[c]{@{}l@{}}Sampling\\ Budget\end{tabular}} & \multicolumn{3}{r}{First Code Block} & \multicolumn{3}{r}{Last Code Block} & \multirow{2}{*}{Nulls} \\ \cmidrule(lr){2-7}
 & Pass@1 & Edit Sim & Valid Rate & Pass@1 & Edit Sim & Valid Rate &  \\ \midrule
1 & 0.25 & 0.7084 & 0.9776 & 0.25 & 0.7084 & 0.9776 & 1 \\
2 & 0.5 & 0.7931 & 0.9752 & 0.5 & 0.7931 & 0.9752 & 1 \\
4 & 0.6094 & 0.9081 & 0.983 & 0.6094 & 0.9081 & 0.983 & 5 \\
8 & 0.7344 & 0.9477 & 0.9588 & 0.7344 & 0.9477 & 0.9588 & 4 \\
12 & 0.7812 & 0.9619 & 0.9684 & 0.7812 & 0.9619 & 0.9684 & 10 \\
16 & 0.8281 & 0.9679 & 0.9706 & 0.8281 & 0.9679 & 0.9706 & 11 \\
20 & 0.8594 & 0.9691 & 0.9788 & 0.8594 & 0.9691 & 0.9788 & 26 \\
24 & 0.7969 & 0.9725 & 0.9638 & 0.7969 & 0.9725 & 0.9638 & 3 \\
28 & 0.8438 & 0.9644 & 0.9797 & 0.8438 & 0.9644 & 0.9797 & 12 \\
32 & 0.8438 & 0.9739 & 0.9806 & 0.8438 & 0.9739 & 0.9806 & 20 \\
64 & 0.9062 & 0.9808 & 0.9855 & 0.9062 & 0.9808 & 0.9855 & 27 \\ \bottomrule
\end{tabular}
\caption{\textbf{Performance across sampling budget (gpt-oss-120b):} gpt-oss-120b performance with sampling budget for max sequence length of 16384 on the 64 PBEBench instances with ground-truth cascade length 10. The nulls column counts cases (out of $Sampling\ Budget \times 32$) where the chain of thought fails to terminate within the max sequence length.}
\label{tab:gpt_oss_120b_sampling_budget}
\end{table*}

\begin{table*}[!tbh]
\centering
\begin{tabular}{@{}lrrrrrrr@{}}
\toprule
\multirow{2}{*}{\begin{tabular}[c]{@{}l@{}}Sampling\\ Budget\end{tabular}} & \multicolumn{3}{r}{First Code Block} & \multicolumn{3}{r}{Last Code Block} & \multirow{2}{*}{Nulls} \\ \cmidrule(lr){2-7}
 & Pass@1 & Edit Sim & Valid Rate & Pass@1 & Edit Sim & Valid Rate &  \\ \midrule
32 & 0.8438 & 0.9739 & 0.9806 & 0.8438 & 0.9739 & 0.9806 & 20 \\
32 & 0.8594 & 0.9716 & 0.9814 & 0.8594 & 0.9716 & 0.9814 & 16 \\
32 & 0.8281 & 0.9772 & 0.9745 & 0.8281 & 0.9772 & 0.9745 & 32 \\
\bottomrule
\end{tabular}
\caption{\textbf{Performance variance per run (gpt-oss-120b):} We analyze the variance exhibited by gpt-oss-120b across all the metrics for a given run.
We conduct this experiment on the 64 instances corresponding to cascade length 10 which is also used for all the scaling experiments and use a sampling budget of 32 and max sequence length of 16384, the same parameters used for evaluation of gpt-oss-120b on PBEBench. The nulls column counts cases (out of $Sampling\ Budget \times 32$) where the chain of thought fails to terminate within the max sequence length.}
\label{tab:gpt_oss_120b_run_variance}
\end{table*}

\begin{table*}[!tbh]
\centering
\begin{tabular}{@{}lrrrrrrr@{}}
\toprule
\multirow{2}{*}{\begin{tabular}[c]{@{}l@{}}Max Seq\\ Length\end{tabular}} & \multicolumn{3}{r}{First Code Block} & \multicolumn{3}{r}{Last Code Block} & \multirow{2}{*}{Nulls} \\ \cmidrule(lr){2-7}
 & Pass@1 & Edit Sim & Valid Rate & Pass@1 & Edit Sim & Valid Rate &  \\ \midrule
2048 & 0 & 0.0141 & 1 & 0 & 0.0141 & 1 & 2043 \\
3840 & 0.2812 & 0.5479 & 0.9717 & 0.2727 & 0.4099 & 1 & 1799 \\
5632 & 0.625 & 0.8447 & 0.9724 & 0.625 & 0.8447 & 0.9724 & 1212 \\
7424 & 0.7656 & 0.9473 & 0.9705 & 0.7656 & 0.9473 & 0.9705 & 607 \\
9216 & 0.8594 & 0.9631 & 0.9875 & 0.8594 & 0.9631 & 0.9875 & 305 \\
12800 & 0.875 & 0.9706 & 0.9788 & 0.875 & 0.9706 & 0.9788 & 80 \\
14592 & 0.875 & 0.9748 & 0.9825 & 0.875 & 0.9748 & 0.9825 & 28 \\
16384 & 0.8438 & 0.9739 & 0.9806 & 0.8438 & 0.9739 & 0.9806 & 20 \\
32768 & 0.9219 & 0.9819 & 0.9838 & 0.9219 & 0.9819 & 0.9838 & 0 \\
65536 & 0.9062 & 0.9755 & 0.9838 & 0.9062 & 0.9755 & 0.9838 & 0 \\ \bottomrule
\end{tabular}
\caption{\textbf{Performance across max sequence length (gpt-oss-120b):} gpt-oss-120b performance with max sequence length at sampling budget 32 on the 64 PBEBench instances with ground-truth cascade length 10. The null column counts cases (out of $64 \times 32 = 2048$) where the chain of thought fails to terminate within the sequence limit.}
\label{tab:gpt_oss_120b_max_seq_len}
\end{table*}

\begin{figure}[!tbh]
    \centering
    \includegraphics[width=0.5\textwidth]{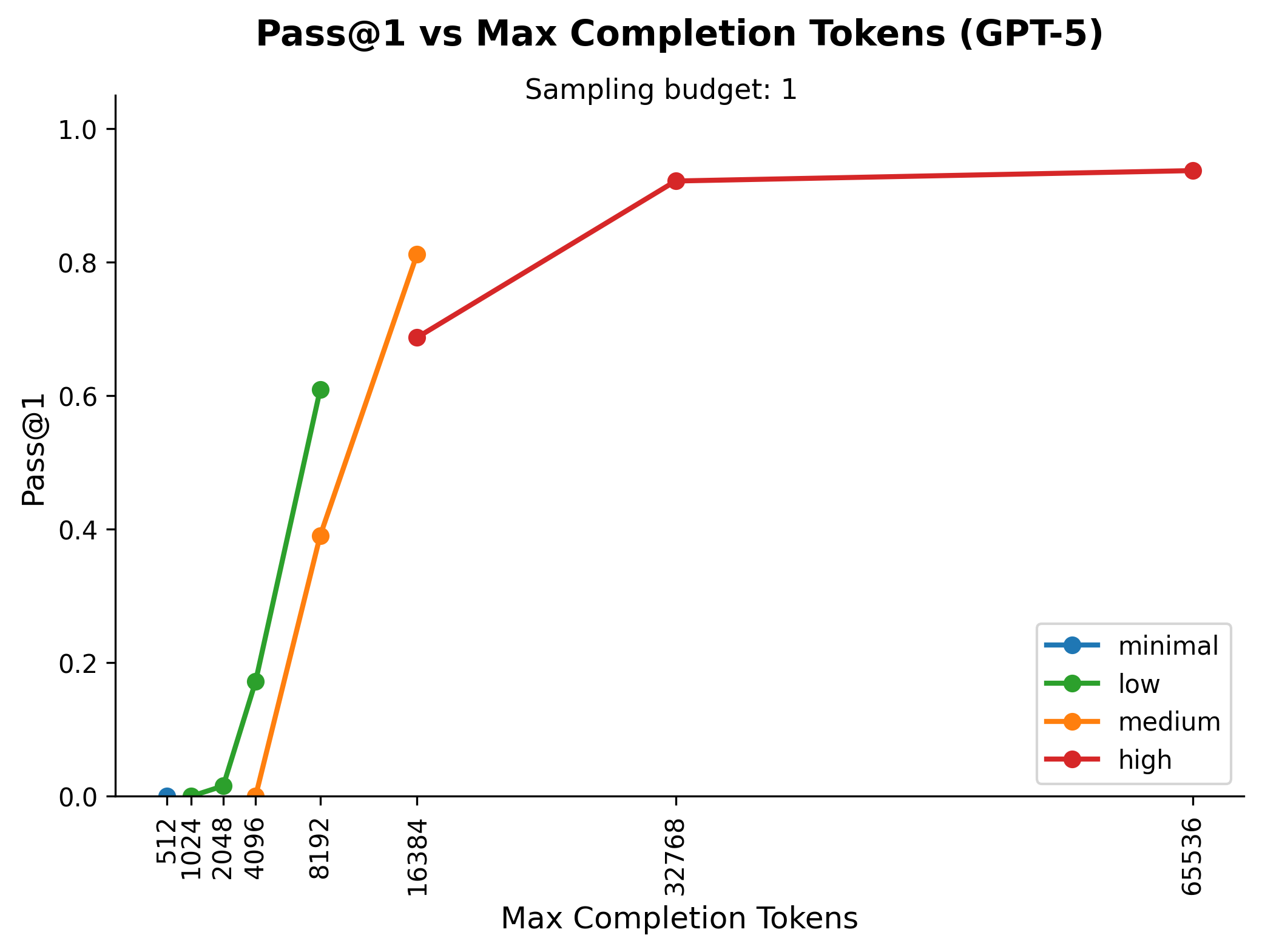}
    \caption{\textbf{Pass@1 vs Max Completion Tokens (sampling budget: 1):} Variation of Pass@1 vs Max Completion Tokens for various reasoning modes ranging from minimal, low, medium, to high.}
\label{fig:max_seq_len_scaling_gpt5_reasoning_mode}
\end{figure}

\subsection{Qualitative Analysis of Program Reordering Performance}
\label{sec:appendix:qual_analysis_program_reordering}
We present a qualitative analysis of reasoning strategies employed by large language models on a program reordering task grounded in phonological rule ordering. 
Through examination of model responses across systems with visible chain-of-thought, we identify two distinct reasoning paradigms: (1) template-based constraint reasoning that explicitly models feeding/bleeding interactions between rules, and (2) exhaustive enumeration via brute-force permutation testing with forward simulation. 
Models employing constraint-based reasoning (e.g., Codestral-22B) correctly identify feeding relationships but often derive incorrect ordering constraints—stating ``Program A feeds Program B'' but concluding A must precede B, the opposite of correct inference. 
Meanwhile, enumeration-based approaches (e.g., QwQ-32B, DeepSeek-R1) exhibit combinatorial collapse as problem size increases: response length grows from roughly 7,000 to about 28,000 characters while accuracy drops from nearly 90\% to around 40\% as the search space expands from 2 to 120 permutations. 
Notably, models rarely reference counterfeeding or counterbleeding despite these interactions being critical to the hardest instances. 
These findings reveal systematic gaps in how LLMs translate domain knowledge into valid constraint satisfaction and highlight the limitations of brute-force search for combinatorial reasoning tasks.

\subsection{Confusion Matrices for Cascade Lengths}
\label{sec:appendix:conf_mat_clen}
We analyze the length of model-predicted cascades against the ground truth cascades to analyze if the models tend to find more or fewer rules than the ground truth cascade for both successful/passing (Pass@1 = 1) and failure cases.
The results on the PBEBench-Lite dataset across all models for successful cases are shown in Fig~\ref{fig:conf_mat_success_clen} and Fig.~\ref{fig:conf_mat_failure_clen}.
The plots reveal that for successful cases, the models tend to largely find solutions of the correct length for shorter cascades, but for longer ground truth cascades, they tend to find fewer solutions in general, but can surprisingly find shorter solutions as well.
For failure scenarios, we note a bigger spread and almost see more cases of the LLMs generating longer programs than the ground truth. This might indicate that for the more complex cases, the LLMs tend to overthink and end up generating more complex cascades that don't work. We also see a large fraction of invalid programs, with this fraction growing more and more for longer ground truth cascades (more complex cases).

\begin{figure}[!tbh]
    \centering
    \includegraphics[width=0.5\textwidth]{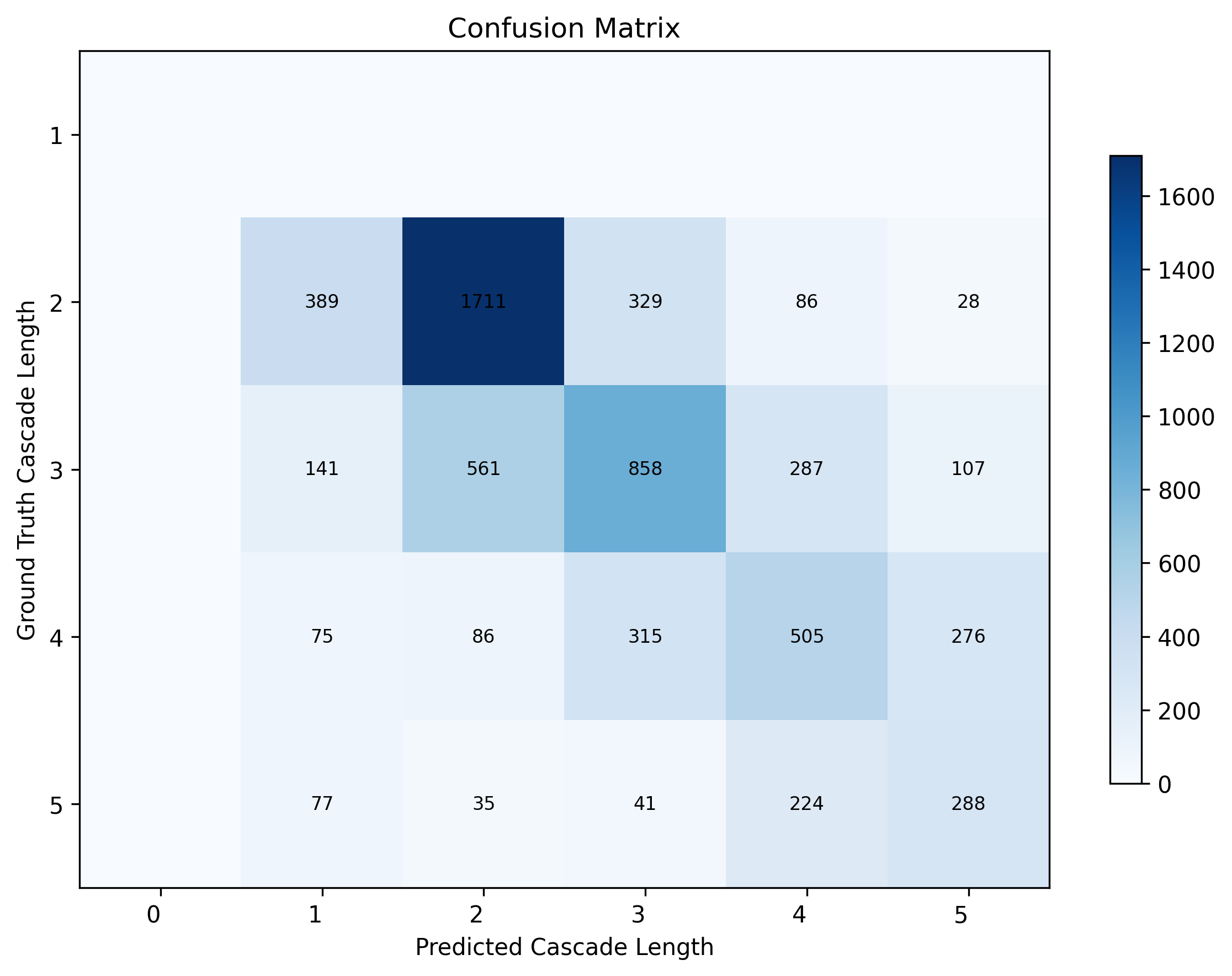}
    \caption{\textbf{Cascade Length Confusion Matrix for Success (All Models PBEBench-Lite)}: Confusion matrix showing the distribution of cascade lengths in the model prediction vs ground truth. 0 length on the predicted side corresponds to cases where the model fails to generate a valid cascade at all. The results are averaged across all models on PBEBench-Lite, and failure denotes Pass@1 = 0 with a sampling budget of 1.}
    \label{fig:conf_mat_success_clen}
\end{figure}

\begin{figure}[!tbh]
    \centering
    \includegraphics[width=0.5\textwidth]{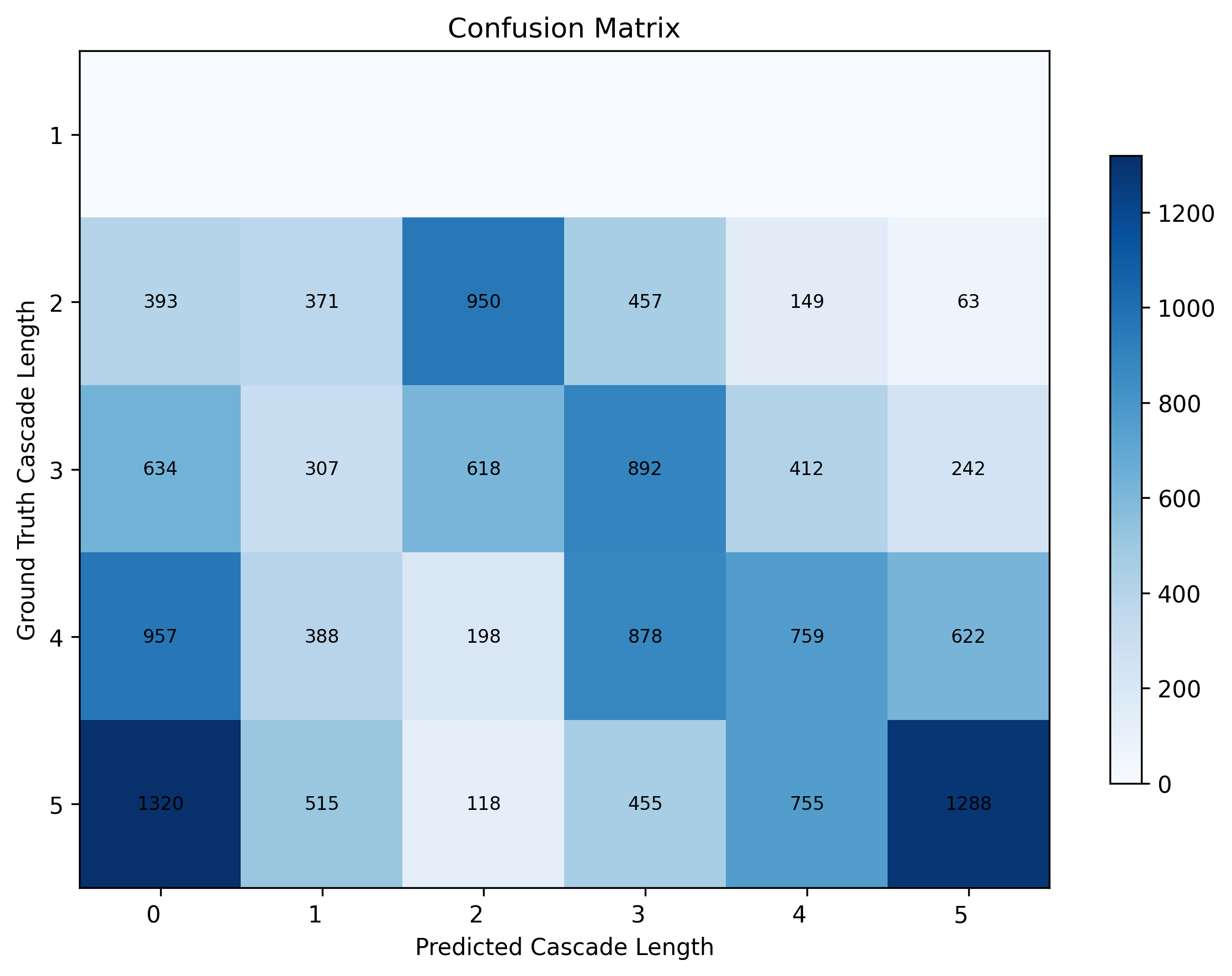}
    \caption{\textbf{Cascade Length Confusion Matrix for Failure (All Models PBEBench-Lite)}: Confusion matrix showing the distribution of cascade lengths in the model prediction vs ground truth. 0 length on the predicted side corresponds to cases where the model fails to generate a valid cascade at all. The results are averaged across all models on PBEBench-Lite, and failure denotes Pass@1 = 0 with a sampling budget of 1.}
    \label{fig:conf_mat_failure_clen}
\end{figure}

\subsection{Confusion Matrices for BFCC Relations}
\label{sec:appendix:conf_mat_reln}
We analyze the types of relations present in the model-generated program cascades and compare them against the ground truth relations and visualize the results via confusion matrices.
The results are on the PBEBench-Lite dataset across all models and separated by whether the model succeeds or fails (based on Pass@1). 
We normalize each row (fraction of predicted cases for each possible relation type for a given ground truth type) and show the overall results for successful cases in Fig.~\ref{fig:conf_mat_all_models_success} and failure cases in Fig.~\ref{fig:conf_mat_all_models_failure}.
While the per-model results for success cases and failure cases span from Fig~\ref{fig:conf_mat_gpt5_success} to Fig~\ref{fig:conf_mat_gpt_oss_20b_success}.
We also plot a simplified version of the confusion matrix that looks at each relation at a time and analyze true positives, false positives, false negatives and true negatives spearately for each relation type for both passing (Fig.~\ref{fig:conf_mat_all_models_success_simplified}) and non-passing cases (Fig.~\ref{fig:conf_mat_all_models_failure_simplified}) aggregated across all the models.
These show an interesting pattern where for successful cases for almost all relation types have relatively high false negative rates but it is especially bad for feeding (72\% false negatives) and counter-feeding (62\%) showing that even when the models can solve cases with ground truth cascades having these relation types they are highly biased against generating them.
Interetingly we also see low false positives for the success/passing cases which is consistent with the tendencies of these models to try and not predict cascades that incorporate BFCC relations and the false positive rate is highest for feeding the hardest relation type.
Finally for failure cases we note that there is a high false negative rate for every relation type ranging between 75\% to 80\%, consistent with the fact that the evaluated models try to avoid predicting BFCC relations and perhaps for cases where they are needed to find a correct solution, they fail to find one.

\begin{figure}[!tbh]
    \centering
    \includegraphics[width=0.5\textwidth]{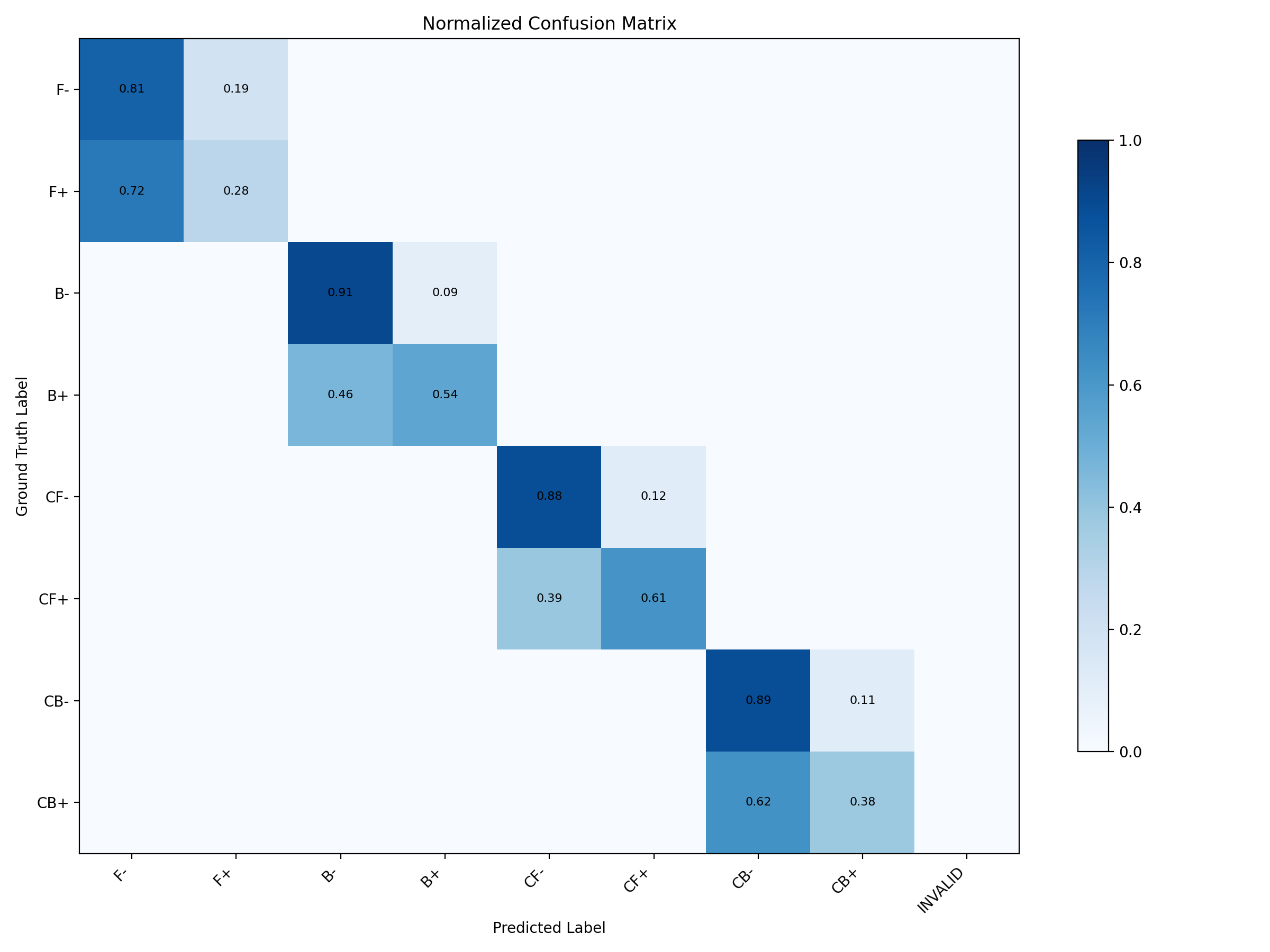}
    \caption{\textbf{Simplified Relation Type Confusion Matrix for Success (All Models PBEBench-Lite)}: Confusion matrix showing the distribution of relation types in the model prediction vs ground truth. INVALID category indicates cases where the model fails to generate a valid cascade at all. The results are averaged across all models on PBEBench-Lite, and success denotes Pass@1 = 1 with a sampling budget of 1.}
    \label{fig:conf_mat_all_models_success_simplified}
\end{figure}

\begin{figure}[!tbh]
    \centering
    \includegraphics[width=0.5\textwidth]{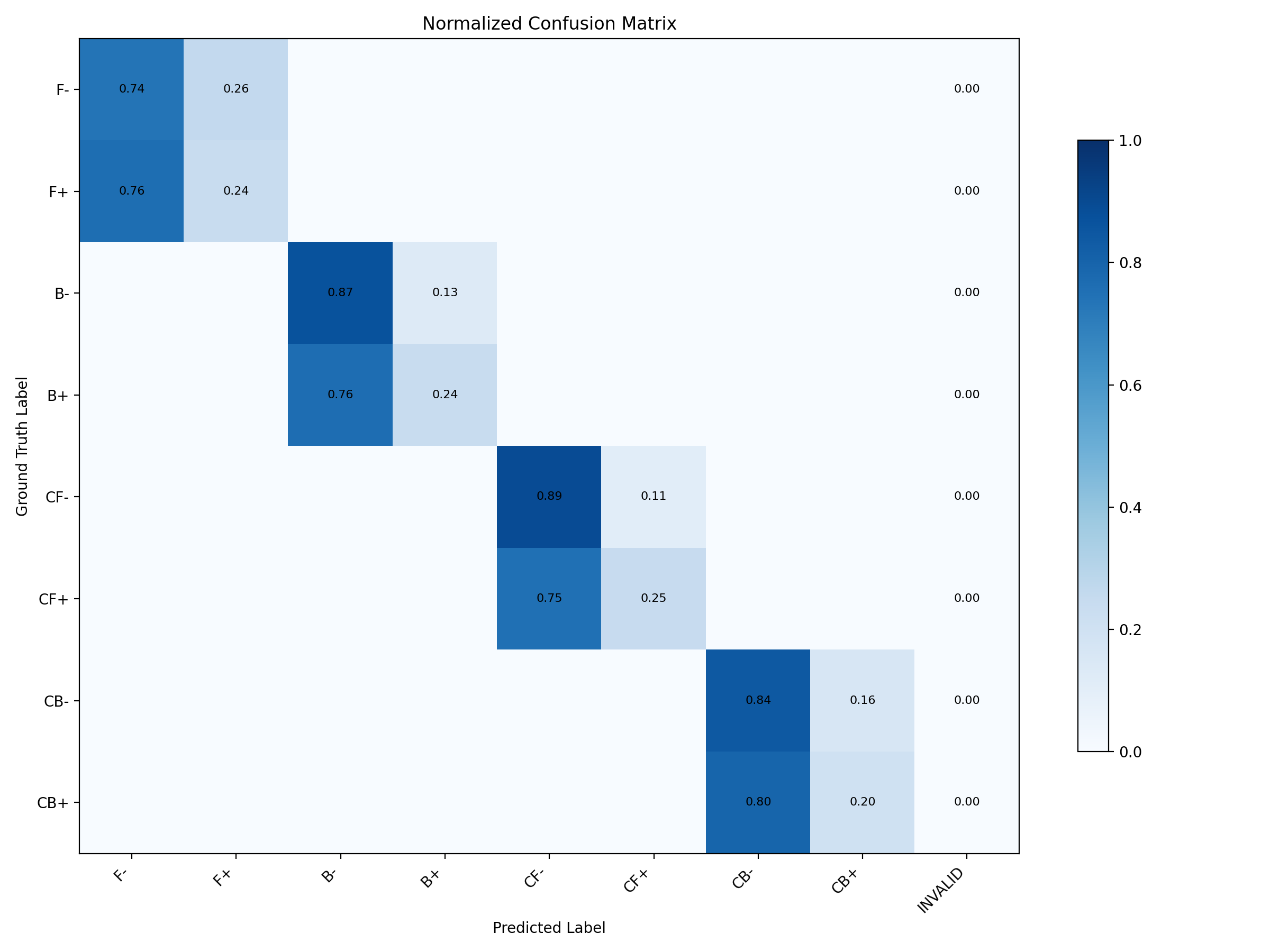}
    \caption{\textbf{Simplified Relation Type Confusion Matrix for Failure (All Models PBEBench-Lite)}: Confusion matrix showing the distribution of relation types in the model prediction vs ground truth. INVALID category indicates cases where the model fails to generate a valid cascade at all. The results are averaged across all models on PBEBench-Lite, and failure denotes Pass@1 = 0 with a sampling budget of 1.}
    \label{fig:conf_mat_all_models_failure_simplified}
\end{figure}

\begin{figure}[!tbh]
    \centering
    \includegraphics[width=0.5\textwidth]{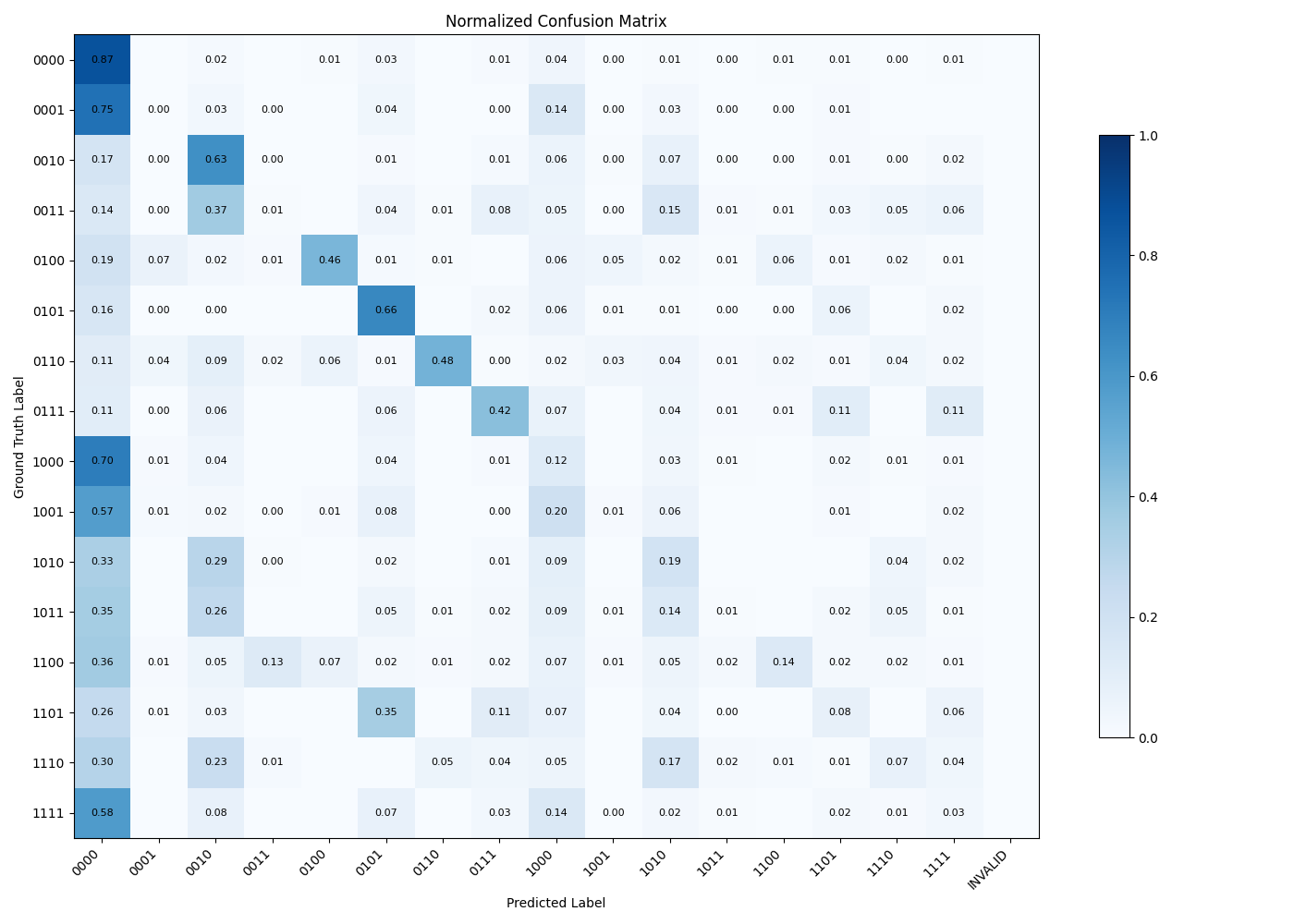}
    \caption{\textbf{Relation Type Confusion Matrix for Success (All Models PBEBench-Lite)}: Confusion matrix showing the distribution of relation types in the model prediction vs ground truth. INVALID category indicates cases where the model fails to generate a valid cascade at all. The results are averaged across all models on PBEBench-Lite, and success denotes Pass@1 = 1 with a sampling budget of 1.}
    \label{fig:conf_mat_all_models_success}
\end{figure}

\begin{figure}[!tbh]
    \centering
    \includegraphics[width=0.5\textwidth]{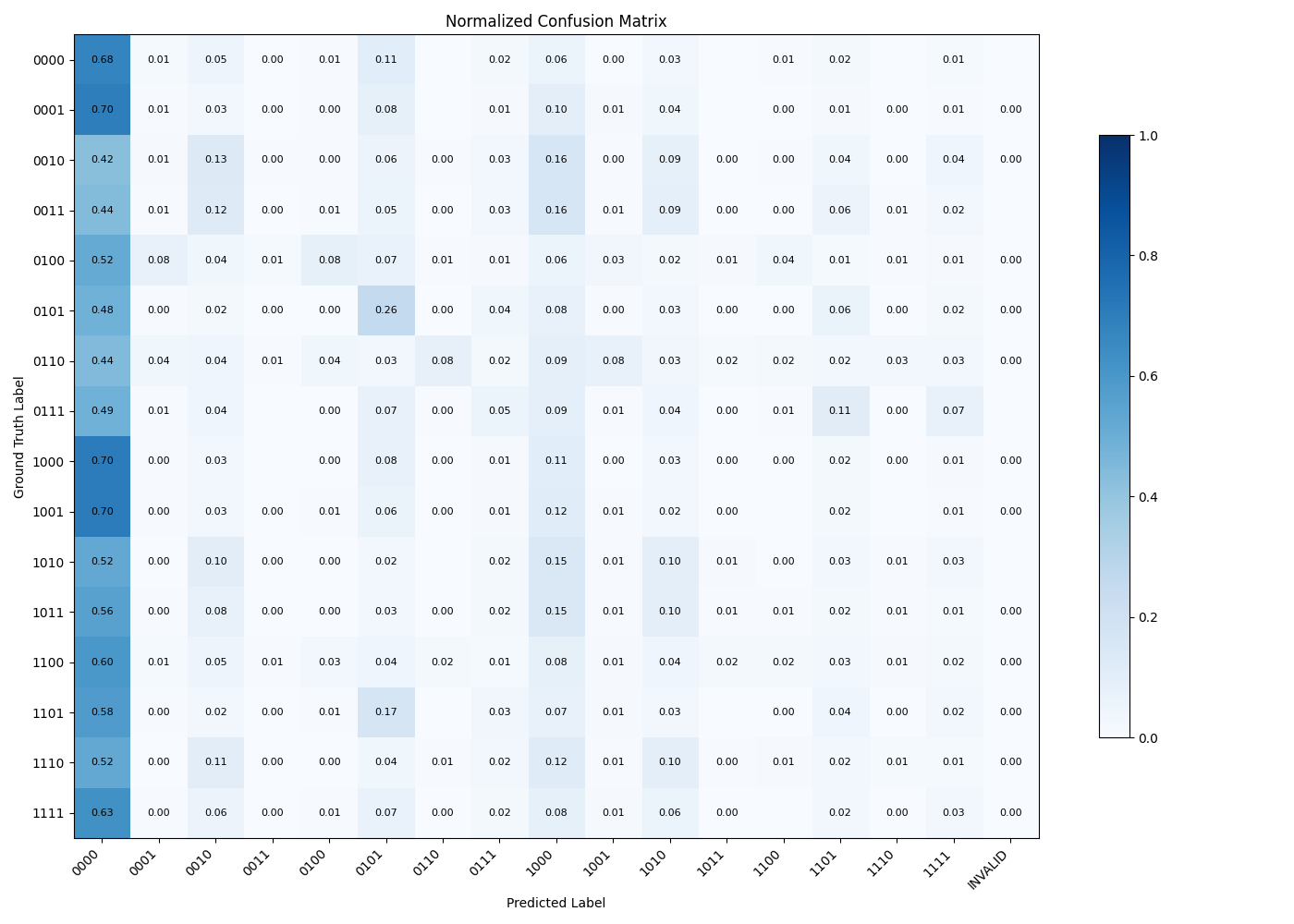}
    \caption{\textbf{Relation Type Confusion Matrix for Failure (All Models PBEBench-Lite)}: Confusion matrix showing the distribution of relation types in the model prediction vs ground truth. INVALID category indicates cases where the model fails to generate a valid cascade at all. The results are averaged across all models on PBEBench-Lite, and failure denotes Pass@1 = 0 with a sampling budget of 1.}
    \label{fig:conf_mat_all_models_failure}
\end{figure}

\begin{figure}[!tbh]
    \centering
    \includegraphics[width=0.5\textwidth]{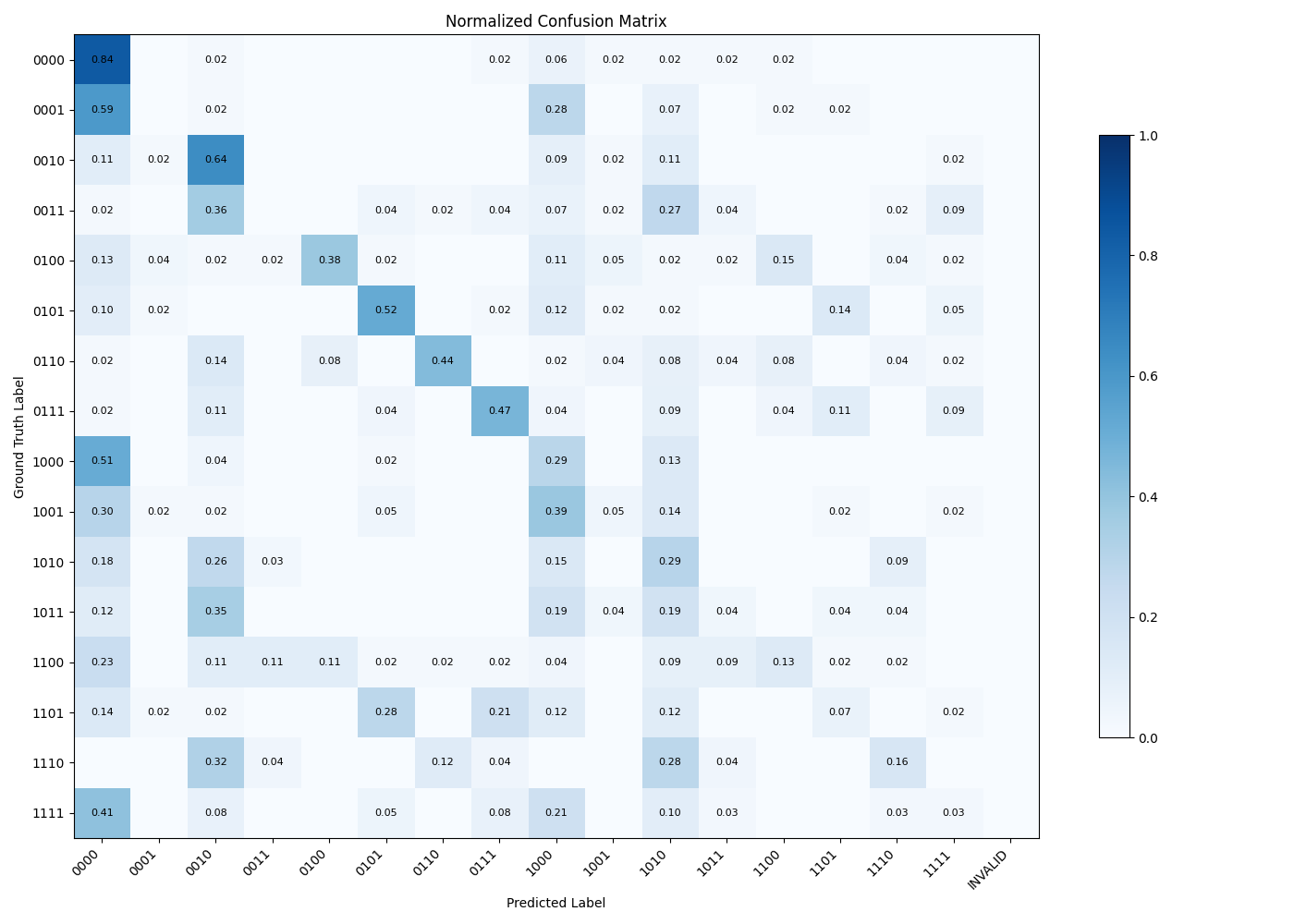}
    \caption{\textbf{Relation Type Confusion Matrix for Success (GPT-5 PBEBench-Lite)}: Confusion matrix showing the distribution of relation types in the model prediction vs ground truth. INVALID category indicates cases where the model fails to generate a valid cascade at all. The results are for GPT-5 on PBEBench-Lite, and success denotes Pass@1 = 1 with a sampling budget of 1.}
    \label{fig:conf_mat_gpt5_success}
\end{figure}

\begin{figure}[!tbh]
    \centering
    \includegraphics[width=0.5\textwidth]{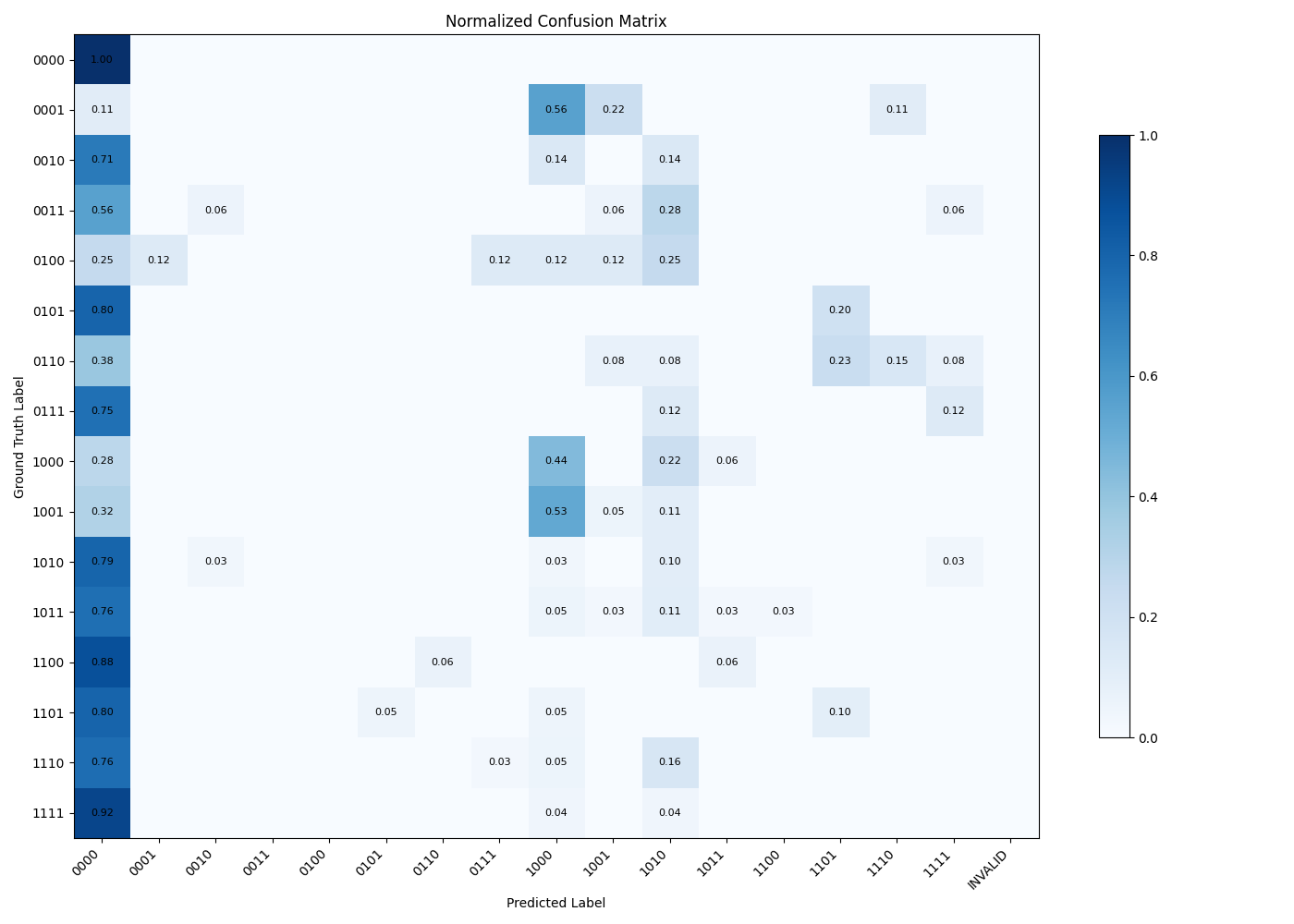}
    \caption{\textbf{Relation Type Confusion Matrix for Failure (GPT-5 PBEBench-Lite)}: Confusion matrix showing the distribution of relation types in the model prediction vs ground truth. INVALID category indicates cases where the model fails to generate a valid cascade at all. The results are for GPT-5 on PBEBench-Lite, and failure denotes Pass@1 = 0 with a sampling budget of 1.}
    \label{fig:conf_mat_gpt5_failure}
\end{figure}

\begin{figure}[!tbh]
    \centering
    \includegraphics[width=0.5\textwidth]{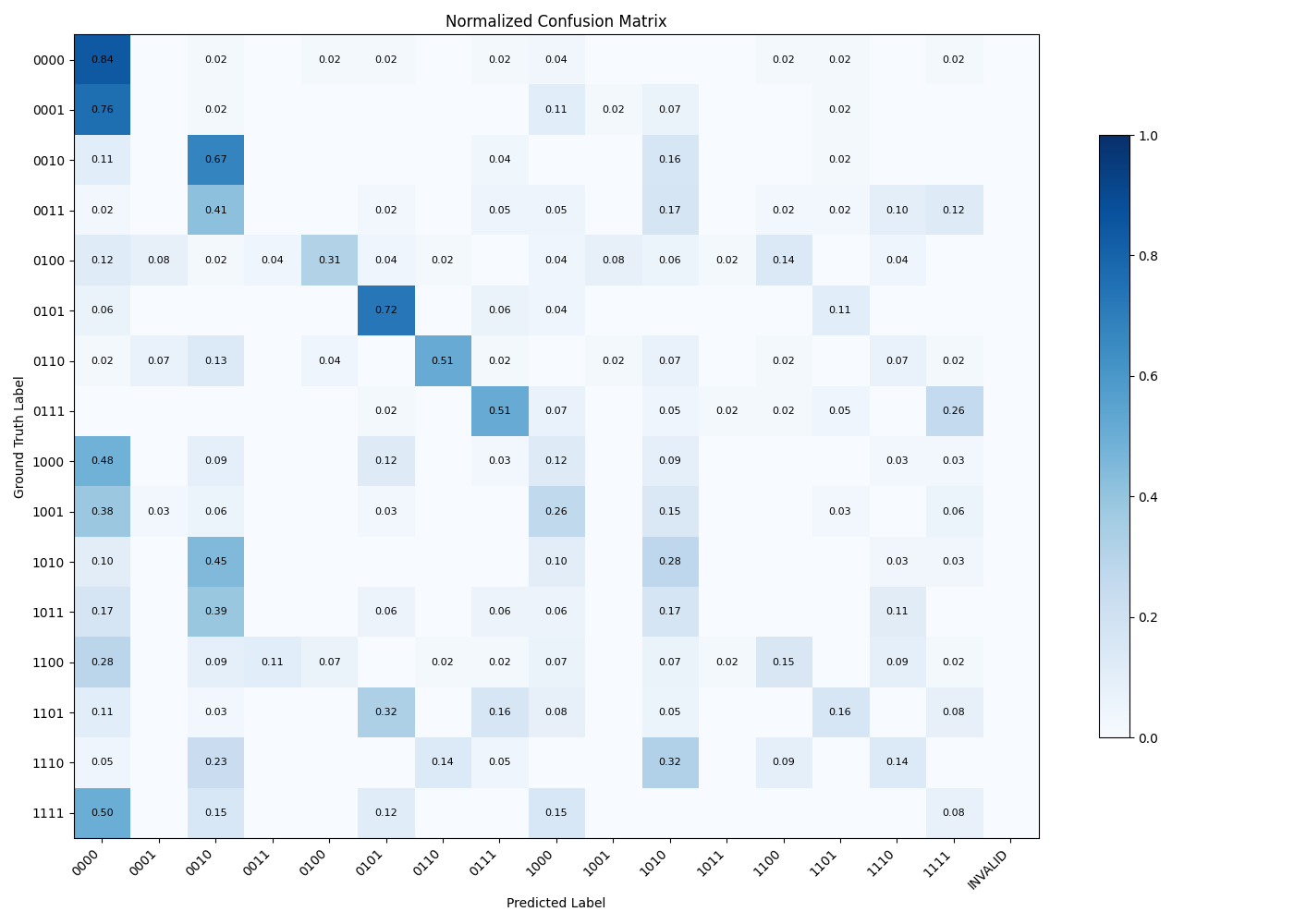}
    \caption{\textbf{Relation Type Confusion Matrix for Success (gpt-oss-120b PBEBench-Lite)}: Confusion matrix showing the distribution of relation types in the model prediction vs ground truth. INVALID category indicates cases where the model fails to generate a valid cascade at all. The results are for gpt-oss-120b on PBEBench-Lite, and success denotes Pass@1 = 1 with a sampling budget of 1.}
    \label{fig:conf_mat_gpt_oss_120b_success}
\end{figure}

\begin{figure}[!tbh]
    \centering
    \includegraphics[width=0.5\textwidth]{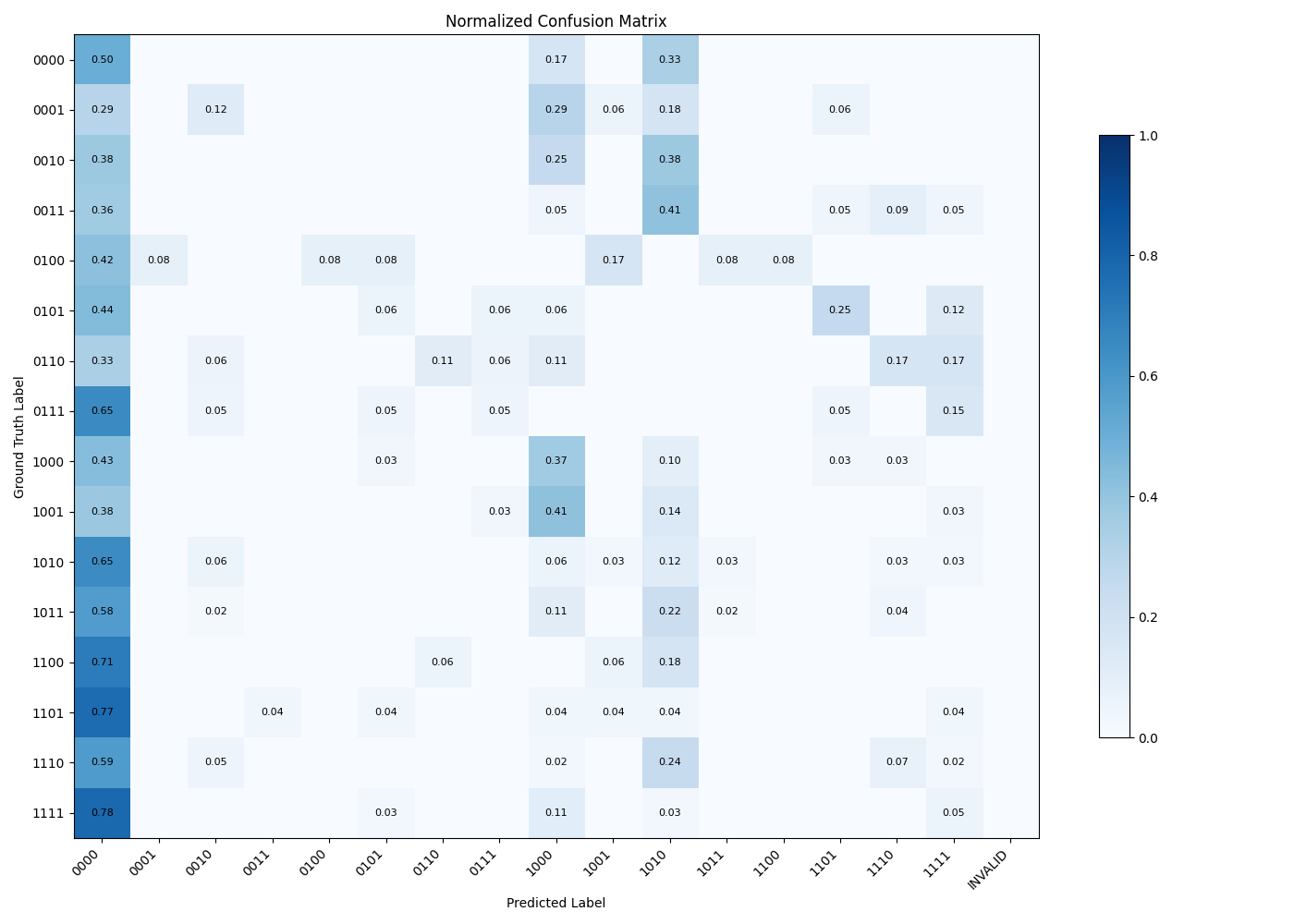}
    \caption{\textbf{Relation Type Confusion Matrix for Failure (gpt-oss-120b PBEBench-Lite)}: Confusion matrix showing the distribution of relation types in the model prediction vs ground truth. INVALID category indicates cases where the model fails to generate a valid cascade at all. The results are for gpt-oss-120b on PBEBench-Lite, and failure denotes Pass@1 = 0 with a sampling budget of 1.}
    \label{fig:conf_mat_gpt_oss_120b_failure}
\end{figure}

\begin{figure}[!tbh]
    \centering
    \includegraphics[width=0.5\textwidth]{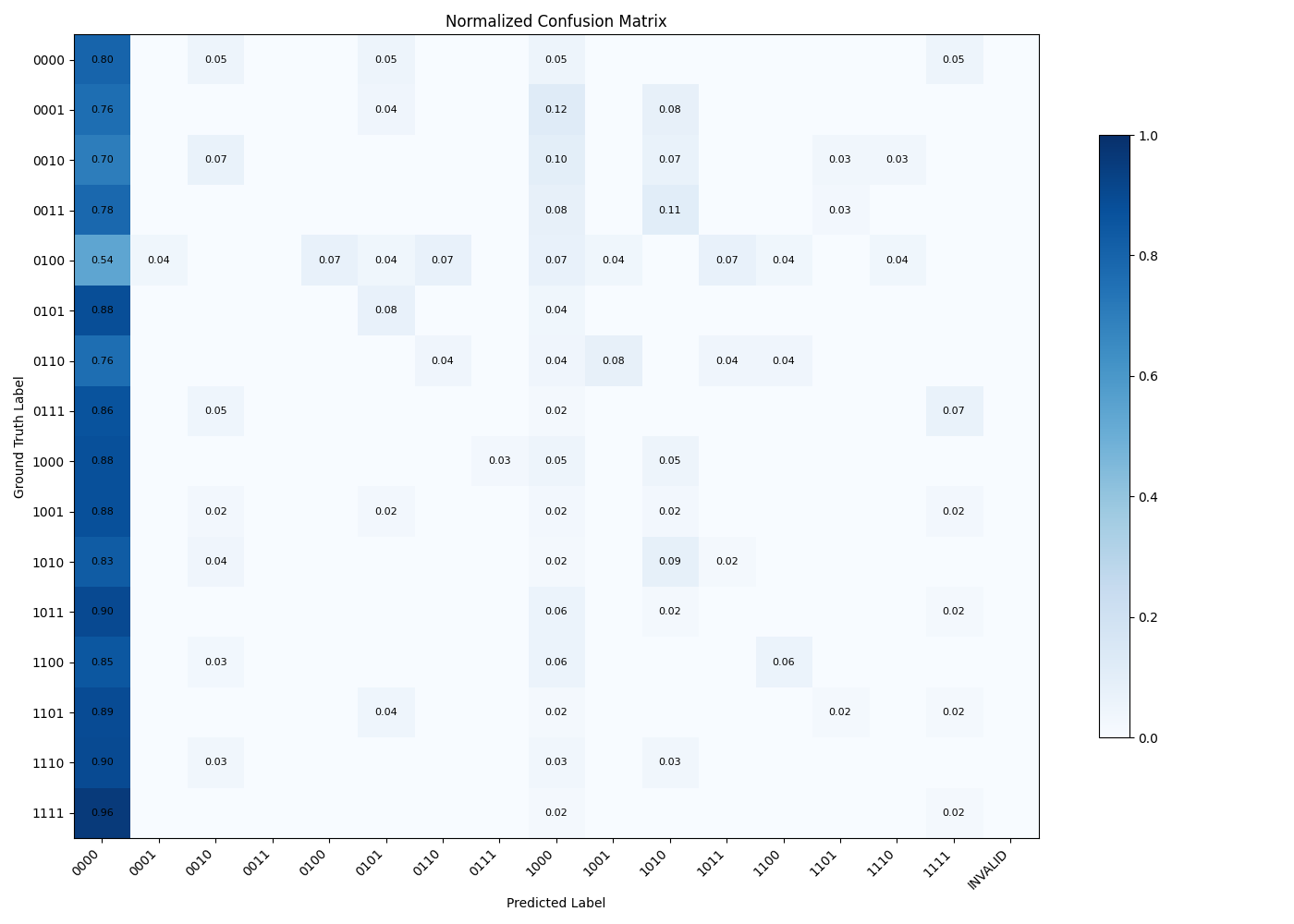}
    \caption{\textbf{Relation Type Confusion Matrix for Failure (gpt-oss-120b PBEBench-Lite)}: Confusion matrix showing the distribution of relation types in the model prediction vs ground truth. INVALID category indicates cases where the model fails to generate a valid cascade at all. The results are for gpt-oss-120b on PBEBench-Lite, and failure denotes Pass@1 = 0 with a sampling budget of 1.}
    \label{fig:conf_mat_gpt_oss_20b_failure}
\end{figure}

\begin{figure}[!tbh]
    \centering
    \includegraphics[width=0.5\textwidth]{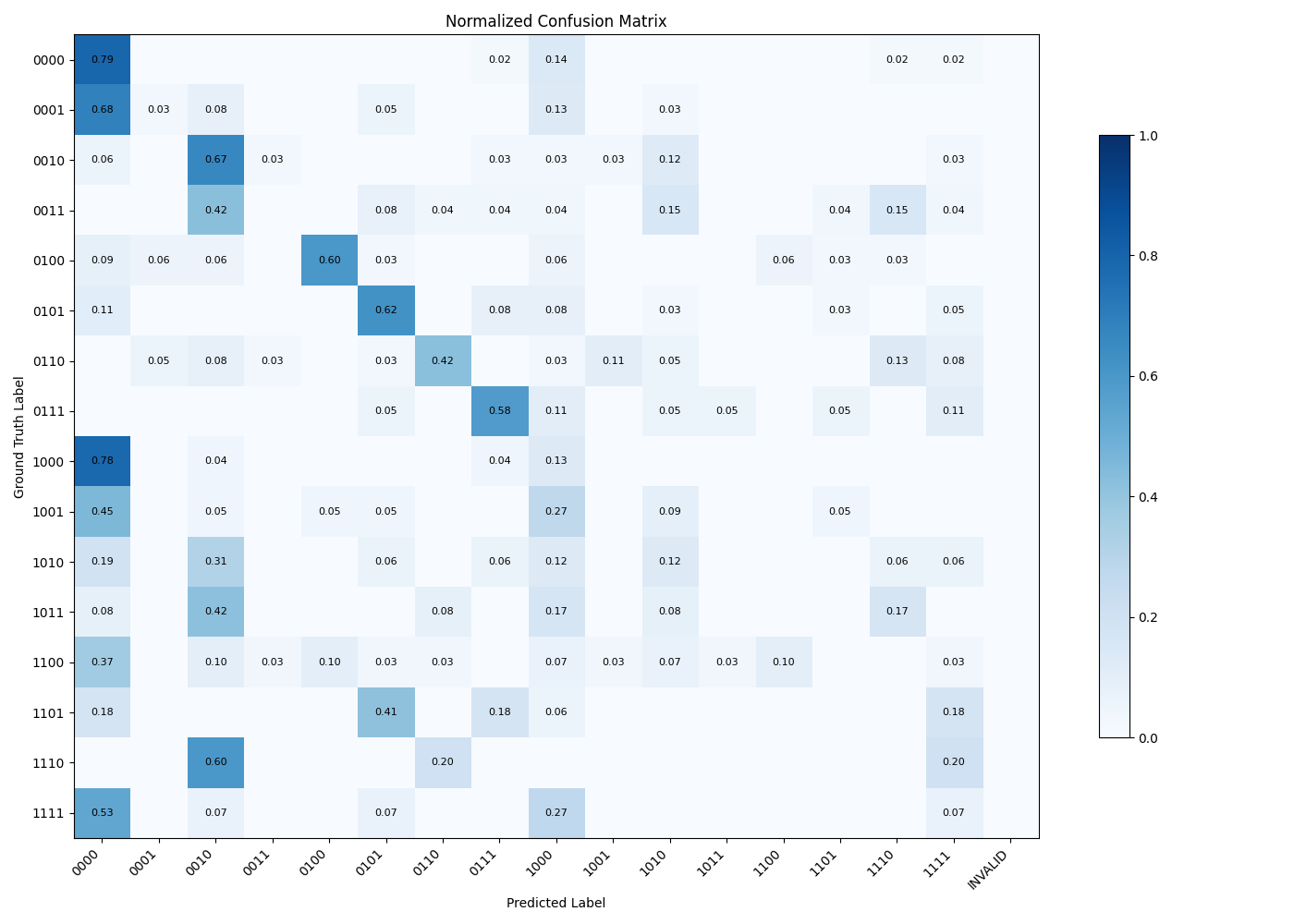}
    \caption{\textbf{Relation Type Confusion Matrix for Success (gpt-oss-120b PBEBench-Lite)}: Confusion matrix showing the distribution of relation types in the model prediction vs ground truth. INVALID category indicates cases where the model fails to generate a valid cascade at all. The results are for gpt-oss-120b on PBEBench-Lite, and success denotes Pass@1 = 1 with a sampling budget of 1.}
    \label{fig:conf_mat_gpt_oss_20b_success}
\end{figure}
\section{Use of AI Assistants}
We made limited use of AI assistants during the preparation of this manuscript, primarily for rephrasing, paraphrasing, and performing grammatical and stylistic checks. 
All technical content, analysis, and conclusions are solely those of the authors.

\end{document}